\pdfoutput=1
\PassOptionsToPackage{backref,pagebackref}{hyperref}
\documentclass[11pt]{article}

\usepackage[top=30truemm,bottom=30truemm,left=25truemm,right=25truemm]{geometry}
\usepackage[utf8]{inputenc}
\usepackage[T1]{fontenc}
\usepackage[authoryear, round]{natbib}
\usepackage[svgnames]{xcolor}
\usepackage{hyperref}
\usepackage{url}
\usepackage{booktabs}
\usepackage{amsfonts}
\usepackage{nicefrac}
\usepackage{microtype}

\usepackage{textcomp}
\usepackage{bm}
\usepackage{bbm}

\usepackage{mathrsfs}

\usepackage{multirow}

\usepackage{colortbl}
\hypersetup{
  colorlinks=true,
  linkcolor=blue,
  citecolor=blue,
}

\usepackage{stmaryrd}
\usepackage{tgtermes}
\usepackage{comment}

\usepackage[ruled,vlined,linesnumbered,algo2e]{algorithm2e}

\usepackage{graphicx}
\usepackage[labelfont=bf]{caption}
\usepackage{wrapfig}
\usepackage{tikz}

\usepackage{mathtools}
\usepackage{amssymb}
\usepackage{amsthm}

\usepackage{enumitem}

\usepackage[normalem]{ulem}
\usepackage{tcolorbox}
\usepackage{mdframed}
\usepackage{framed}
\definecolor{shadecolor}{gray}{0.93}

\usepackage{threeparttable}

\usepackage[nameinlink,capitalise]{cleveref}
\crefname{equation}{}{}
\Crefname{equation}{}{}
\Crefname{algocf}{Algorithm}{Algorithms}
\crefname{figure}{Figure}{Figures}
\Crefname{figure}{Figure}{Figures}

\crefalias{AlgoLine}{line}
\crefalias{algocfline}{line}

\theoremstyle{plain}
\newtheorem{theorem}{Theorem}
\newtheorem{proposition}[theorem]{Proposition}
\newtheorem{lemma}[theorem]{Lemma}

\theoremstyle{definition}

\renewcommand{\hat}{\widehat}
\renewcommand{\tilde}{\widetilde}
\renewcommand{\epsilon}{\varepsilon}

\def\e{\mathrm{e}}
\def\E{\mathbb{E}}
\def\P{\mathbb{P}}
\def\I{\mathbb{I}}
\def\R{\mathbb{R}}

\def\Rnn{\mathbb{R}_{\geq 0}}
\def\N{\mathbb{N}}

\def\calA{\mathcal{A}}
\def\calB{\mathcal{B}}
\def\calC{\mathcal{C}}
\def\calD{\mathcal{D}}
\def\calE{\mathcal{E}}
\def\calF{\mathcal{F}}

\def\calK{\mathcal{K}}

\def\calM{\mathcal{M}}

\def\calP{\mathcal{P}}
\def\calQ{\mathcal{Q}}
\def\calS{\mathcal{S}}

\def\calX{\mathcal{X}}

\DeclareMathOperator*{\argmax}{arg\,max}
\DeclareMathOperator*{\argmin}{arg\,min}

\DeclarePairedDelimiter{\brk}{[}{]}

\let\set\undefined
\DeclarePairedDelimiter{\set}{\{}{\}}
\DeclarePairedDelimiter{\prn}{(}{)}
\DeclarePairedDelimiter{\nrm}{\|}{\|}

\DeclarePairedDelimiter{\inpr}{\langle}{\rangle}

\DeclarePairedDelimiter{\floor}{\lfloor}{\rfloor}

\newcommand{\ind}[1]{\mathbbm{1}{\left[#1\right]}}
\newcommand{\relmiddle}[1]{\mathrel{}\middle#1\mathrel{}}

\newcommand{\comred}[1]{{\color{red}{}}}

\definecolor{mycitecolor}{HTML}{0049B7}

\newcommand{\proofparagraph}[1]{\par\medskip\noindent{\normalfont\bfseries #1}\quad}

\newcommand{\ones}{\mathbf{1}}

\newcommand{\Reg}{\mathsf{Reg}}

\newcommand{\unif}{\mathsf{Unif}}
\newcommand{\ber}{\mathsf{Ber}}

\newcommand{\Pmat}{\mathsf{P}}

\newcommand{\Ktil}{\tilde{K}}

\renewcommand{\H}{\operatorname{H}}

\newcommand{\nn}{\nonumber\\}

\newcommand{\per}{\,.}
\newcommand{\com}{\,,}

\newcommand{\sumT}{\sum_{t=1}^T}

\title{From Dueling Bandits to Episodic MDPs: \\ Learning with Adversarial Preference Feedback}

\author{
  Taira Tsuchiya\footnote{
    The University of Tokyo and RIKEN; 
    \texttt{tsuchiya@mist.i.u-tokyo.ac.jp}.
  }
  \and
  Shinji Ito\footnote{
    The University of Tokyo and RIKEN; \texttt{shinji@mist.i.u-tokyo.ac.jp}.
  }
  \and
  Haipeng Luo\footnote{
    University of Southern California; \texttt{haipengl@usc.edu}.
  }
}

\begin{document}
\maketitle

\begin{abstract}
In dueling bandits, the learner compares two arms in each round and observes only which one is preferred.
This paper asks what changes when the compared pair also affects the next state, as in a recommendation interface where the pair shown to a user affects the user's next context.
We formulate this problem as a state-wise preference-based MDP, an episodic tabular MDP in which the learner selects a pair of arms at each visited state, observes only a noisy comparison between them, and moves to a next state that may depend on the pair.
The preferences may change adversarially across episodes, and the loss of a pair is defined through Borda scores.
We show that the minimax regret of this problem under known transitions is $\tilde{\Theta}(\min\{HT,(H^3SK)^{1/3}T^{2/3}\})$, where $H$, $S$, $K$, and $T$ are the numbers of steps per episode, states, arms, and episodes, and that this rate is achieved by an algorithm based on online linear optimization over the set of occupancy measures.
If the pair did not affect the next state, each state would be a separate dueling bandit, and the $T^{2/3}$ term of the minimax regret would be smaller by a factor of $H^{1/3}$, demonstrating the cost of having action-dependent state transitions.
We also give a policy optimization algorithm with closed-form updates, whose regret is $\tilde{O}((H^4SK^3)^{1/3}T^{2/3})$ up to lower-order terms, and we extend both algorithms to unknown transitions.
\end{abstract}

\section{Introduction}\label{sec:introduction}
Reinforcement learning from preference feedback instead of numerical rewards has attracted increasing attention \citep{wirth17survey,christiano17deep,novoseller20dueling,joey23inverse,saha23dueling,wu24making}.
A basic model in this line of research is the dueling bandit problem~\citep{yue09interactively}, where the learner compares two arms and observes only a binary outcome indicating which one is preferred.
Using only this feedback, the learner aims to minimize the regret, that is, the cumulative difference between the losses of the compared arms and the loss of the best arm.
Most dueling bandit algorithms assume the existence of a Condorcet winner~\citep{urvoy13generic,zoghi14relative,komiyama15regret}, that is, an arm that wins in comparisons against every other arm, or of an underlying utility value for each arm~\citep{ailon14reducing,gajane15relative}.

Both assumptions fail when preferences are inconsistent or cyclic, since a Condorcet winner or an underlying utility need not exist.
Recent work has thus turned to an alternative notion of optimality, the \emph{Borda winner}~\citep{urvoy13generic,jamieson15sparse,ramamohan16dueling,falahatgar17maxing,heckel18approximate,saha21adversarial,wu24borda,suk24nonstationary}.
The Borda winner is the arm with the highest win probability against a uniformly random opponent and always exists.
Dueling bandits with a Borda winner can therefore handle arbitrary preferences, but they do not model how the compared arms affect future states.
For example, a recommendation interface shows a user a pair of items and observes which one is clicked, and the pair shown affects the user's next context.
This paper asks how to learn from comparisons when the compared pair also affects the future states.

\begin{table*}[t]
  \caption{Regret bounds established in this paper for state-wise PbMDPs with Borda scores.
  The notation $\tilde{O}$ hides logarithmic factors.
  Each bound is of the form $\min\{HT,\cdot\}$, and the trivial term $HT$ is omitted.
  OLO stands for online linear optimization.
  $^\dagger$Up to additive lower-order terms.
  }
  \label{table:regret}
  \centering
  \small
  \begin{tabular}{llll}
    \toprule
    Setting
    & Result & Regret & Algorithm
    \\
    \midrule
    Lower bound
    &
    \Cref{thm:lower_bound_adv_pref_borda} & $\Omega\prn[\big]{(H^3 S K)^{1/3} T^{2/3}}$ & --
    \\
    \midrule
    \multirow{2}{6em}{Known transitions}
    &
    \Cref{thm:reg_occ_meas}
    & $\tilde{O}\prn[\big]{(H^3 S K)^{1/3} T^{2/3}}$ & OLO over polytope
    \\
    &
    \Cref{thm:regret_po_23}
    & $\tilde{O}\prn[\big]{(H^4 S K^3)^{1/3} T^{2/3}}^\dagger$ & policy optimization
    \\
    \midrule
    \multirow{2}{6em}{Unknown transitions}
    &
    \Cref{thm:reg_occ_unknown}
    & $\tilde{O}\prn[\big]{(H^3 S K)^{1/3} T^{2/3} + H S K \sqrt{T}}^\dagger$ & OLO over polytope
    \\
    &
    \Cref{thm:regret_po_23_unk}
    & $\tilde{O}\prn[\big]{(H^4 S K^3)^{1/3} T^{2/3} + H^2 S K \sqrt{T}}^\dagger$ & policy optimization
    \\
    \bottomrule
  \end{tabular}
\end{table*}

\subsection{Contributions of this paper}
To answer this question, we propose a new online learning model of episodic MDPs with adversarial state-wise pairwise preference feedback, which we call \emph{state-wise preference-based MDPs (state-wise PbMDPs for short)}.
State-wise PbMDPs extend episodic MDPs with adversarial losses~\citep{neu10ssp,zimin13online,jin20learning,jin2021best} in the same way as dueling bandits extend multi-armed bandits.
At each visited state, the learner selects a pair of arms as its action and observes only the outcome of a comparison between them, rather than a numerical loss.
The preference probabilities may vary arbitrarily across episodes.
We measure the loss of a pair by the Borda scores of its arms, which are well defined for arbitrary preferences.
The model then contains adversarial dueling bandits with a Borda winner as a special case.
To our knowledge, this is the first regret analysis for episodic MDPs with adversarial state-wise preference feedback.

Our first algorithm runs follow-the-regularized-leader (FTRL) with the negative Shannon entropy over the set of occupancy measures, following the online linear optimization approach to adversarial MDPs~\citep{zimin13online,jin20learning}, and achieves a regret bound of $O((H^3SK\log(SK))^{1/3}T^{2/3})$ under known transitions, where $H$, $S$, $K$, and $T$ denote the numbers of steps per episode, states, arms, and episodes, respectively (\Cref{thm:reg_occ_meas}).
The main difficulty is that the Borda scores at a state are estimated from a uniformly random comparison at that state, whereas the probability of visiting the state can be arbitrarily small, so that an estimator that divides by this probability has a large second moment.
We resolve this by restricting the decision set of FTRL to the occupancy measures of policies that make a uniformly random comparison with a fixed probability at every state.
The executed policy then makes these comparisons itself, and the resulting regret bound does not depend on the state-visit probabilities.

The first algorithm has to solve a convex program in every episode, which is computationally demanding, and this approach is difficult to extend to structured MDPs.
These drawbacks have motivated policy optimization methods for adversarial MDPs~\citep{shani20optimistic,luo21policy,dann23best,li26policy}.
Our second algorithm follows this policy optimization framework, which maintains a distribution over pairs at each state and updates it by a multi-armed bandit algorithm whose losses correspond to an underlying Q-function.
The main new ingredient is an estimator of the Q-function from preference feedback.
It combines the losses of all pairs at a visited state, estimated from a single uniform comparison there, with the losses at later steps, estimated from the executed trajectory.
Under known transitions, the algorithm achieves a regret bound of $\tilde{O}((H^4SK^3)^{1/3}T^{2/3})$ up to lower-order terms (\Cref{thm:regret_po_23}), which has the same dependence on $T$ and $S$ as our first bound and is worse by a factor of $H^{1/3}K^{2/3}$.

We also extend both algorithms to unknown transitions using confidence sets for the transition kernel.
In this setting, we obtain a regret bound of $\tilde{O}((H^3SK)^{1/3}T^{2/3}+HSK\sqrt{T})$ for global optimization and a bound of $\tilde{O}((H^4SK^3)^{1/3}T^{2/3}+H^2SK\sqrt{T})$ for policy optimization, in both cases up to an additive term polylogarithmic in $T$ (\Cref{thm:reg_occ_unknown,thm:regret_po_23_unk}), and estimating the transitions thus adds only lower-order terms.
The main body presents the algorithms and analyses for known transitions, and the unknown-transition settings are deferred to \Cref{app:global_unknown,app:po_unknown}.

Finally, we prove that the regret bound of our first algorithm is optimal up to the logarithmic factor.
We show that, for any learning algorithm, there exists a state-wise PbMDP with Borda scores on which the expected regret is $\Omega(\min\{HT,(H^3SK)^{1/3}T^{2/3}\})$ (\Cref{thm:lower_bound_adv_pref_borda}), that is, $\Omega((H^3SK)^{1/3}T^{2/3})$ for $T\geq SK$.
Hence, under the assumptions of \Cref{thm:lower_bound_adv_pref_borda}, the minimax regret of state-wise PbMDPs with Borda scores under known transitions is $\tilde{\Theta}(\min\{HT,(H^3SK)^{1/3}T^{2/3}\})$, and the leading term of the policy optimization bound is optimal in $T$ and $S$.
The proof requires a new instance in which the comparisons that reveal the Borda winner also move the learner to states with larger losses for the rest of the episode.
When the pair does not affect the next state, each state is a separate dueling bandit, and the minimax regret is $\tilde{\Theta}(\min\{HT,(H^2SK)^{1/3}T^{2/3}\})$ (\cref{app:pair_independent}), which is smaller by a factor of $H^{1/3}$ in the $T^{2/3}$ term, demonstrating the cost of having action-dependent state transitions.
\Cref{table:regret} summarizes our results.
Experiments on small instances (\cref{app:experiments}) show that the regret grows at a rate consistent with $T^{2/3}$ for global optimization and at a rate approaching $T^{2/3}$ for policy optimization as $T$ increases, and that policy optimization has a shorter running time.

\subsection{Related work}\label{subsec:additional_discussion}\label{subsec:comparison}
Our work relates to preference-based reinforcement learning (PbRL).
We refer to the survey by \citet{wirth17survey} and, for recent developments, to \citet{christiano17deep,novoseller20dueling,xu20preference,joey23inverse,saha23dueling,wu24making}.

Many PbRL methods learn from comparisons of trajectories or trajectory segments~\citep{christiano17deep,xu20preference,saha23dueling,wu24making}.
In contrast, state-wise PbMDPs adopt \emph{state-wise} comparisons.
At each visited state, the learner selects a \emph{pair} of arms and observes a preference outcome between them, and the dynamics may depend on the displayed pair of arms.
The two formulations use different feedback and are not directly reducible to each other.
\Cref{tab:formulation_comparison} in \Cref{app:related_work} compares our interaction protocol with the trajectory-comparison formulation.

A natural example is an interactive ranking or recommendation interface~\citep{joachims02optimizing,radlinski08how,hofmann16online}, where the state is the observed user context, the displayed pair is compared through the user's choice, and the pair shown can affect the subsequent context.
This example motivates the state-wise formulation, which models the comparison of the displayed pair at the current state and its effect on the next state, and allows the preferences to vary across time and states in an inconsistent or nonstationary manner.
Comparing actions at the same state has also been considered in reinforcement and imitation learning~\citep{furnkranz12preference,myers23active,yang23boosting,sekhari23contextual}.

\section{Preliminaries}\label{sec:preliminaries}
\paragraph{Notation.}
For $n \in \N$, let $[n] = \set{1, \dots, n}$.
For a vector $x$, $x_i$ denotes its $i$-th element and $\nrm{x}_p$ its $\ell_p$-norm for $p \in [1, \infty]$.
For a finite set $\calK$, $\Delta(\calK)$ denotes the set of all probability distributions over $\calK$.
We write $\ind{\cdot}$ for the indicator function, $\calA^\calB$ for the set of all functions from $\calB$ to $\calA$, $\ber(p)$ for the Bernoulli distribution with mean $p$, and $\unif(\calA)$ for the uniform distribution over $\calA$.

\paragraph{Tabular MDPs.}
We study a tabular MDP $\calM = (\calS, \calA, P, H, s_0)$,
where $\calS$ is a finite state space with $S = |\calS|$, $\calA$ is a finite action space with $A = |\calA|$, $P \colon \calS \times \calA \to \Delta(\calS)$ is the transition kernel, $H$ is the horizon, that is, the number of steps per episode, and $s_0 \in \calS$ is the initial state.
Here $P(s' \mid s,a)$ denotes the probability of transitioning to state $s'$ when taking action $a$ in state $s$.

Following existing studies~\citep{zimin13online,du19provably,jin20learning,jin2021best}, we assume that $\calS$ is partitioned into $(H+1)$ disjoint layers, $\calS = \bigcup_{h=0}^{H} \calS_h$ with $S_h = |\calS_h|$, where $\calS_0 = \set{s_0}$ and $\calS_H = \set{s_H}$ are singletons and all intermediate layers $\calS_1, \dots, \calS_{H-1}$ are nonempty.
Transitions proceed forward by one layer: for $h \in \set{0, \dots, H-1}$ and $(s,a) \in \calS_h \times \calA$, $\sum_{s' \in \calS_{h+1}} P(s' \mid s, a) = 1$ and $P(s' \mid s, a) = 0$ for $s' \notin \calS_{h+1}$.
A general finite-horizon MDP can be unfolded into a layered MDP by augmenting the state with the time index, at the cost of increasing the state space by a factor of $H$.
All bounds in this paper are stated in terms of the layered state space.
For simplicity, we sometimes write $\calS$ for $\calS \setminus \set{s_H}$.

\paragraph{Policies and value functions.}
An episode starts at $s_0$ and lasts $H$ steps.
At step $h$, the learner observes the state $s_h$, selects an action $a_h$, and moves to $s_{h+1} \sim P(\cdot \mid s_h, a_h)$.
A stochastic Markov policy $\pi \in \Pi = \Delta(\calA)^{\calS}$ selects $a_h \sim \pi(\cdot \mid s_h)$.
Let $h(s)$ denote the layer of $s$, that is, $h(s) \in \set{0, \ldots, H}$ with $s \in \calS_{h(s)}$.
Given $P$, a function $\ell \colon \calS \times \calA \to [0,1]$, and $\pi \in \Pi$, the value function $V^\pi(\cdot; \ell)$ and the Q-function $Q^\pi(\cdot, \cdot; \ell)$ are defined by
$
V^\pi(s; \ell)
=
\E\brk[\big]{\sum_{k=h(s)}^{H-1} \ell(s_k, a_k) \bigm| s_{h(s)} = s, \pi, P}
$
and
$
Q^\pi(s, a; \ell)
=
\ell(s,a)
+
\E\brk[\big]{
  \sum_{k=h(s)+1}^{H-1} \ell(s_k, a_k)
  \bigm|
  s_{h(s)}=s,\,
  a_{h(s)}=a,\,
  \pi,\,
  P
},
$
where the expectation is over $a_k \sim \pi(\cdot \mid s_k)$ and $s_{k+1} \sim P(\cdot \mid s_k, a_k)$ for $k \geq h(s)$.
We omit the dependence on $P$ when it is clear from the context.
The occupancy measure $q^{\pi} \colon \calS \times \calA \to [0,1]$ of $\pi$ is the probability of visiting $(s,a)$ under $\pi$ and $P$:
$q^{\pi}(s,a) = \Pr[\exists h\in\{0,\ldots,H-1\}, s_h = s, a_h = a \mid \pi, P]$, with $q^{\pi}(s) = \sum_{a \in \calA} q^{\pi}(s,a)$.
Let $\Omega = \set{q^\pi \colon \pi \in \Pi}$ denote the set of all occupancy measures, which forms a polytope with $\Theta(S)$ linear constraints~\citep{puterman14markov}, and let $\pi^q \in \Pi$ be the Markov policy induced by $q \in \Omega$, given by $\pi^q(a \mid s) = q(s,a) / q(s)$.

\section{State-wise preference-based MDPs and the Borda score}\label{sec:mdp_adv_pref}
This section introduces state-wise preference-based MDPs (state-wise PbMDPs) and then focuses on the case where the loss functions are determined by Borda scores.

\paragraph{State-wise preference-based MDPs.}
Formally, a state-wise PbMDP is specified by a tabular MDP $\calM$ with action space $\calA = [K] \times [K]$ for the set of arms $[K] = \set{1,\dots,K}$, where $K\geq2$, and an adversarial sequence $\set{\Pmat_t}_{t=1}^T$ of preference functions $\Pmat_t \colon \calS \times \calA \to [0,1]$ satisfying $\Pmat_t(s,(a^{\mathrm{L}}, a^{\mathrm{R}})) = 1 - \Pmat_t(s,(a^{\mathrm{R}}, a^{\mathrm{L}}))$ for every $s \in \calS$ and every pair $a = (a^{\mathrm{L}}, a^{\mathrm{R}}) \in \calA$.
Here $\Pmat_t(s,(i,j))$ is the probability that arm $i$ is preferred to arm $j$ at state $s$ in episode $t$.
The incurred loss $\ell_t \colon \calS \times \calA \to [0,1]$ is determined by $\Pmat_t$, and in all main results below it is the Borda-induced loss defined in~\cref{eq:def_borda_ell}.

The interaction protocol proceeds as follows:
\begin{shaded}
  \noindent\textbf{Protocol.}\quad For each episode $t = 1, \dots, T$:
  \begin{enumerate}[topsep=1pt, itemsep=0pt, partopsep=0pt, leftmargin=25pt]
    \item Environment chooses a preference function $\Pmat_t \colon \calS \times \calA \to [0,1]$;
    \item
    Learner selects a policy $\pi_t$ based on past observations without knowing $\Pmat_t$;
    \item
    Learner executes policy $\pi_t$ to obtain a trajectory $\set[\big]{ (s_{t,h}, a_{t,h}, o_t(s_{t,h}, a_{t,h}^{\mathrm{L}}, a_{t,h}^{\mathrm{R}})) }_{h=0}^{H-1} \in (\calS \times \calA \times \set{0,1})^{H}$.
    Specifically,
    for each step $h = 0, \dots, H-1$,
    the learner selects an arm pair $a_{t,h} = (a_{t,h}^{\mathrm{L}}, a_{t,h}^{\mathrm{R}}) \sim \pi_t( \cdot \mid s_{t,h})$,
    moves to $s_{t,h+1} \sim P( \cdot \mid s_{t,h}, a_{t,h})$,
    incurs the loss $\ell_t(s_{t,h}, a_{t,h})$,
    and observes the preference feedback $o_t(s_{t,h}, a_{t,h}^{\mathrm{L}}, a_{t,h}^{\mathrm{R}}) \sim \ber(\Pmat_t(s_{t,h}, a_{t,h}))$.
  \end{enumerate}
\end{shaded}
We assume an oblivious adversary, that is, the entire sequence $\set{\Pmat_t}_{t=1}^T$ is fixed before the first episode and does not depend on the learner's randomization, the transition randomness, or the preference-feedback noise.
The learner aims to minimize the \emph{regret} $\Reg_T=\sumT V^{\pi_t}(s_0; \ell_t)-\sumT V^{\pi^*}(s_0; \ell_t)$, where $\pi^*\in\argmin_{\pi \in \Pi} \sumT V^{\pi}(s_0; \ell_t)$.
We study both the known-transition setting, in which the learner has full access to $P$, and the unknown-transition setting, in which the learner infers $P$ from interactions.

\paragraph{Borda score.}
Given $\Pmat_t$, the (shifted) Borda score $b_t \colon \calS \times [K] \to [-1,0]$ and the loss function $\ell_t \colon \calS \times \calA \to [0,1]$ are defined as
\begin{equation}
    \label{eq:def_borda_ell}
  \textstyle
  b_t(s,i) = \frac{1}{K} \sum_{j=1}^K \prn*{\Pmat_t(s, (i,j)) - 1}
  \com\quad
  \ell_t(s, a)
  =
  -
  \frac12\prn*{
    b_t(s, a^{\mathrm{L}}) + b_t(s, a^{\mathrm{R}})
  }
  \per
\end{equation}
The shifted Borda score is the probability that arm $i$ is preferred to a uniformly random opponent at state $s$, shifted by $-1$ for analytical convenience.
Note that state-wise PbMDPs with Borda scores reduce to dueling bandits with a Borda winner when $H=1$, that is, when the MDP has a single nonterminal state and a single step.
Even in this much simpler multi-armed bandit setting, different notions of a winner lead to different problems that require different algorithms~\citep{zoghi14relative,ramamohan16dueling,dudik15contextual,saha21adversarial}.
We therefore focus on the Borda winner, which is always defined, and leave other preference-based objectives, such as those based on the Condorcet winner, for future work.

\paragraph{Additional notation.}
Let
$
\I_t(s,a)
\coloneqq
\ind{ \exists h\in\{0,\ldots,H-1\}, (s_{t,h}, a_{t,h}) = (s, a)}
$
indicate whether $(s,a)$ is visited in episode $t$, and $\I_t(s) = \sum_{a \in \calA} \I_t(s,a)$.
We use $\E_t\brk{\cdot}$ for the expectation conditioned on all observations before episode $t$.

\section{Global optimization}\label{sec:global_optimization}
This section presents \Cref{alg:pf_mdp_occ}, a global optimization algorithm for state-wise PbMDPs with Borda scores.
We first describe the algorithm and its regret bound for known transitions, and then state the regret bound for unknown transitions.

\subsection{Algorithm design under known transitions}\label{subsec:global_known_design}

\paragraph{Reduction to online linear optimization.}
The regret can be written as $\E\brk{\Reg_T}=\E\brk[\big]{\sumT\inpr{q^{\pi_t}-q^{\pi^*},\ell_t}}$, where $\ell_t$ is the Borda-induced loss in~\cref{eq:def_borda_ell}.
Hence, it suffices to run FTRL over the set of occupancy measures $\Omega$, following the online linear optimization approach to adversarial MDPs~\citep{zimin13online,jin20learning}.
Unlike in adversarial MDPs with loss feedback, the learner observes no entry of $\ell_t$, and $\ell_t$ has to be estimated from the preference feedback.

\begin{algorithm2e}[t]
\small
\KwIn{Learning rate $\eta>0$ and exploration rate $\rho\in(0,1/2]$}
\For{$t=1,\ldots,T$}{
  Compute $\widetilde q_t$ by FTRL over $\Omega_\rho$ and set $\pi_t=\pi^{\widetilde q_t}$:\label{line:global_ftrl}
  \begin{equation}
    \widetilde q_t\in\argmin_{q\in\Omega_\rho}
    \set*{\inpr*{q,\sum_{\tau=1}^{t-1}\hat\ell_\tau}-\frac1\eta\H(q)},\quad
    \H(q)=\sum_{s,a}q(s,a)\log\prn*{\frac{1}{q(s,a)}}
    .
    \label{eq:ftrl_occ}
\end{equation}\par
  Compute the residual policy:\label{line:global_residual}
  \refstepcounter{equation}\label{eq:global_residual_policy}%
  $\pi_t^{\rm r}(a\mid s)=(\pi_t(a\mid s)-\rho/A)/(1-\rho)$.\hfill\textup{(\theequation)}\par
  \For{$h=0,\ldots,H-1$\label{line:global_execute}}{
    At $s=s_{t,h}$, sample $Z_t(s)\sim\ber(\rho)$, and select $a_{t,h}$ uniformly from $\calA$ if $Z_t(s)=1$ and from $\pi_t^{\rm r}(\cdot\mid s)$ otherwise\label{line:global_coin}\;
    Observe $o_t(s,a_{t,h}^{\rm L},a_{t,h}^{\rm R})$ and $s_{t,h+1}$\label{line:global_observe}\;
  }
  Set $Z_t(s)=0$ at unvisited states, compute $\hat b_t$ by~\cref{eq:def_b_hat}, and set\label{line:global_borda_estimate}
  \begin{equation}
    \widetilde b_t(s,i)=\frac{\I_t(s)Z_t(s)\hat b_t(s,i)}{\rho\widetilde q_t(s)},
    \qquad
    \hat\ell_t(s,a)=-\frac{\widetilde b_t(s,a^{\rm L})+\widetilde b_t(s,a^{\rm R})}{2}.
    \label{eq:global_borda_estimator}
\end{equation}\par
}
\caption{Global optimization under known transitions}
\label{alg:pf_mdp_occ}
\end{algorithm2e}

\paragraph{Borda score estimation.}
We first describe how to estimate the Borda scores from the preference feedback.
For a visited state $s$, let $a_{t,s}^{\rm L}$ and $a_{t,s}^{\rm R}$ be the arms compared at $s$ in episode $t$, and define
\begin{equation}
  \textstyle
  \hat b_t(s,i)
  =K\ind{a_{t,s}^{\rm L}=i}
  \bigl(o_t(s,a_{t,s}^{\rm L},a_{t,s}^{\rm R})-1\bigr)
  \qquad \mbox{for each} \ i\in[K].
  \label{eq:def_b_hat}
\end{equation}
We set $\hat b_t(s,\cdot)=0$ when $s$ is not visited.
If the compared pair is drawn uniformly from $[K]\times[K]$, then $\hat b_t(s,i)$ is an unbiased estimator of $b_t(s,i)$ (\cref{lem:bhat_unbiased}).
Since the estimates in \cref{eq:def_b_hat} for all $i\in[K]$ are computed from the same comparison, a single uniform comparison at $s$ yields unbiased estimates of the Borda scores of all $K$ arms, and hence of the losses $\ell_t(s,a)$ of all $A$ pairs at $s$.

The question is how to make such uniform comparisons without losing control of the state-visit probabilities.
In adversarial dueling bandits, \citet{saha21adversarial} draw the two arms independently from a distribution over arms and estimate the Borda scores by inverse weighting with respect to that distribution.
In a state-wise PbMDP, the transition depends on the ordered pair, and drawing the two arms independently from the marginals of a policy would in general change the next-state distribution.
For example, a policy that selects $(1,1)$ and $(2,2)$ with probability $1/2$ each never selects $(1,2)$, whereas independent draws from its marginals select it with probability $1/4$.
We instead keep the decision variable as an occupancy measure over ordered pairs and build the uniform comparisons into the policy returned by the optimization problem.

\paragraph{FTRL with uniform comparisons.}
For $\rho\in(0,1/2]$, we define the restricted set of occupancy measures $\Omega_\rho=\{q\in\Omega:q(s,a)\geq\rho q(s)/A$ for all $(s,a)\in\calS\times\calA\}$.
The constraints are linear and require only that, conditional on a visit to $s$, every pair is selected with probability at least $\rho/A$.
The set $\Omega_\rho$ is a convex polytope, and it is nonempty since it contains the occupancy measure of the mixture $(1-\rho)\pi+(\rho/A)\ones$ of any policy $\pi$ with the uniform policy.

In each episode $t$, the algorithm computes $\widetilde q_t$ by FTRL over $\Omega_\rho$ with the negative Shannon entropy regularizer, as in \cref{eq:ftrl_occ}, and executes the induced policy $\pi_t=\pi^{\widetilde q_t}$ (\cref{line:global_ftrl}).
To make the uniform comparisons explicit, it also computes the residual policy $\pi_t^{\rm r}$ in \cref{eq:global_residual_policy}, which satisfies $\pi_t(a\mid s)=(1-\rho)\pi_t^{\rm r}(a\mid s)+\rho/A$ (\cref{line:global_residual}).
The residual probabilities are nonnegative and sum to one since $\pi_t(a\mid s)\geq\rho/A$, and at a state with $\widetilde q_t(s)=0$ both policies are defined to be uniform.
At each visited state $s$, the algorithm samples $Z_t(s)\sim\ber(\rho)$ and selects the pair uniformly from $\calA$ if $Z_t(s)=1$ and from $\pi_t^{\rm r}(\cdot\mid s)$ otherwise (\cref{line:global_coin}).
This implements $\pi_t$ exactly, and $Z_t(s)=1$ indicates that a uniform comparison is made at $s$.

After the episode, the algorithm sets $Z_t(s)=0$ at unvisited states and computes the estimates $\widetilde b_t$ and $\hat\ell_t$ in \cref{eq:global_borda_estimator}, where the first estimate is defined as zero if its denominator is zero (\cref{line:global_borda_estimate}).
Since $\rho\widetilde q_t(s)$ is exactly the probability that a uniform comparison is made at $s$ in episode $t$, both estimates are unbiased at every reachable state (\cref{lem:global_known_estimator_moments}).

\subsection{Regret upper bounds}\label{subsec:global_known_bound}
The above algorithm achieves the following bound, whose proof is given in \cref{app:proof_global_opt}.
\begin{theorem}\label{thm:reg_occ_meas}
  Suppose that the transition kernel is known.
  With $\eta=(\log(SA)/(T\sqrt{SK}))^{2/3}$ and $\rho=\min\{1/2,\sqrt{\eta SK}/H\}$, \Cref{alg:pf_mdp_occ} satisfies $\E\brk{\Reg_T}=O\prn{\min\set{HT,(H^3SK\log(SK))^{1/3}T^{2/3}}}$.
\end{theorem}

For $H=1$, state-wise PbMDPs with Borda scores reduce to adversarial dueling bandits with a Borda winner, and the bound in \cref{thm:reg_occ_meas} becomes $O((K\log K)^{1/3}T^{2/3})$, which recovers the minimax regret of \citet[Theorems~2 and~16]{saha21adversarial} up to the logarithmic factor.
\Cref{sec:lower_bounds} shows that the bound is minimax optimal up to the logarithmic factor for general $H$ as well.

\paragraph{Technical and proof overview.}
With loss feedback, the losses of the visited pairs are observed, and FTRL over the set of occupancy measures needs no explicit exploration.
With preference feedback, the Borda scores at $s$ are estimated from a uniform comparison at $s$, which the executed policy makes with probability $\rho\widetilde q_t(s)$, and the estimator in \cref{eq:global_borda_estimator} divides by this probability.
Since the estimate in \cref{eq:def_b_hat} has second moment at most $K$ under a uniform comparison (\cref{lem:bhat_unbiased}), we have $\E_t\brk{\widetilde b_t(s,i)^2}\leq K/(\rho\widetilde q_t(s))$ (\cref{lem:global_known_estimator_moments}), which is large at rarely visited states.
The restriction to $\Omega_\rho$ lets the algorithm execute $\pi^{\widetilde q_t}$ without further mixing it with the uniform policy, so that the occupancy measure of the executed policy is the FTRL iterate $\widetilde q_t$ and the factor $1/\widetilde q_t(s)$ is cancelled by the weight $\widetilde q_t(s,a)$ in the FTRL bound, as in the loss-feedback case.
This can be confirmed in the following proof sketch.

Since $q^{\pi^*}$ need not lie in $\Omega_\rho$, we compare with the mixture $(1-\rho)\pi^*+(\rho/A) \ones$, whose occupancy measure $q_\rho^*$ lies in $\Omega_\rho$, at a cost of $\rho H^2T$ (\cref{lem:global_comparator_smoothing}).
Since $\hat\ell_t$ is unbiased, the regret against $q_\rho^*$ equals $\E\brk{\sum_t\inpr{\widetilde q_t-q_\rho^*,\hat\ell_t}}$, and the standard analysis of FTRL over $\Omega_\rho$ gives $\sum_t\inpr{\widetilde q_t-q_\rho^*,\hat\ell_t}\leq H\log(SA)/\eta+(\eta/2)\sum_t\sum_{s,a}\widetilde q_t(s,a)\hat\ell_t(s,a)^2$ (\cref{lem:global_nested_ftrl}).
The weight $\widetilde q_t(s,a)$ in the second term is the probability of visiting $(s,a)$ under the executed policy, and hence
\[
 \textstyle
  \E_t\brk[\big]{\sum_{s,a}\widetilde q_t(s,a)\hat\ell_t(s,a)^2}
  \leq\sum_{s,a}\widetilde q_t(s,a) \, ({K}/{\rho\widetilde q_t(s)})
  \leq {SK}/{\rho},
\]
where the first inequality follows from $\E_t\brk{\widetilde b_t(s,i)^2}\leq K/(\rho\widetilde q_t(s))$ and the second inequality from $\sum_a\widetilde q_t(s,a)=\widetilde q_t(s)$.
The regret is thus at most $H\log(SA)/\eta+\eta SKT/(2\rho)+\rho H^2T$, and choosing $\eta$ and $\rho$ to balance the three terms gives the bound in \cref{thm:reg_occ_meas}.%

\paragraph{Unknown transitions.}
When the transition kernel is unknown, we combine \Cref{alg:pf_mdp_occ} with confidence sets for $P$ as in \citet{jin20learning}.
The resulting algorithm runs the FTRL update over the set of occupancy measures that are consistent with the current confidence set, and divides the Borda estimates by an upper confidence bound on the state-visit probability instead of by $\widetilde q_t(s)$.
The algorithm and its proof are given in \cref{app:global_unknown}, and it achieves $\E\brk{\Reg_T}=\tilde{O}(\min\{HT,(H^3SK)^{1/3}T^{2/3}+HSK\sqrt T+HS^4K^2\})$, where the last two terms are the cost of estimating the transitions (\cref{thm:reg_occ_unknown}).

\section{Policy optimization}\label{sec:policy_optimization}
The convex program over $\Omega_\rho$ in the last section can be solved in polynomial time by standard convex optimization methods, but it is a global optimization problem over all states and has to be solved in every episode.
This section presents \Cref{alg:pf_mdp_policy_optimization}, a policy optimization algorithm that avoids it, updating a distribution over pairs at each state in closed form and using only dynamic programming.

\subsection{Algorithm design}

\paragraph{Reduction to online linear optimization at each state.}
By the performance difference lemma~\citep{kakade02approximately}, we can decompose the regret as
\begin{equation}
  \textstyle
    \Reg_T=\sum_{s\in\calS}q^{\pi^*}(s)
    \sumT\inpr{\pi_t(\cdot\mid s)-\pi^*(\cdot\mid s),Q^{\pi_t}(s,\cdot;\ell_t)}.
  \label{eq:state_wise_reg}
\end{equation}
Hence it suffices to solve an online linear optimization problem over $\Delta(\calA)$ with loss vectors $Q^{\pi_t}(s,\cdot;\ell_t)$ for each $s\in\calS$.
As in the previous section, these loss vectors must be estimated from a single trajectory per episode.
We next describe the exploration scheme and the Q-function estimate.

\paragraph{Exploration and loss estimation.}
At each state $s$, the algorithm maintains a distribution $\tilde\pi_t(\cdot\mid s)$ over pairs, whose update is described at the end of this section.
As in \Cref{sec:global_optimization}, uniform comparisons are used for estimation.
At each visited state $s$, the algorithm samples $Z_t(s)\sim\ber(\gamma)$ and selects the pair uniformly from $\calA$ if $Z_t(s)=1$ and from $\tilde\pi_t(\cdot\mid s)$ otherwise (\cref{line:po_coin}), so that the executed policy is the mixture $\pi_t$ in \cref{eq:po_executed_policy}.

The loss estimate $\hat\ell_t$ in \cref{eq:def_l_hat} is built from $\hat b_t$ in \cref{eq:def_b_hat}.
Since $Z_t(s)$ is drawn after reaching $s$, this estimate satisfies $\E_t\brk{\hat\ell_t(s,a)\mid\I_t(s)=1}=\ell_t(s,a)$ for every pair $a$ (\cref{lem:po_local_moments}).
In contrast to the estimator in \cref{eq:global_borda_estimator} for global optimization, $\hat\ell_t$ is divided only by the probability $\gamma$ and not by a state-visit probability, since the visit probabilities are accounted for in the Q-function estimate below.

\begin{algorithm2e}[t]
\small
\KwIn{Learning rate $\eta>0$, exploration rate $\gamma\in(0,1/2]$, and $\delta=2\eta(K/\gamma+H)$}
Set $\tilde\pi_1(a\mid s)=1/A$ for every nonterminal $(s,a)$\label{line:po_initialize}\;
\For{$t=1,\ldots,T$}{
  Set $\pi_t$ and compute $q_t$ under $P$ as follows:\label{line:po_occupancy}
  \begin{equation}
    \pi_t(a\mid s)=(1-\gamma)\tilde\pi_t(a\mid s)+\gamma/A,
    \qquad q_t=q^{\pi_t}.
    \label{eq:po_executed_policy}
\end{equation}\par
  Compute $m_t$ and $M_t$, interpreting $m_t(s,a)=m_t(s)$:\label{line:po_bonus}
  \begin{equation}
    m_t(s)=2\delta H\sum_{a\in\calA}\frac{\tilde\pi_t(a\mid s)}{q_t(s,a)+\delta},
    \qquad M_t(s,a)=Q^{\pi_t}(s,a;m_t).
    \label{eq:def_bonus_M_po}
\end{equation}\par
  \For{$h=0,\ldots,H-1$\label{line:po_execute}}{
    At $s=s_{t,h}$, draw $Z_t(s)\sim\ber(\gamma)$, and select $a_{t,h}$ uniformly from $\calA$ if $Z_t(s)=1$ and from $\tilde\pi_t(\cdot\mid s)$ otherwise\label{line:po_coin}\;
    Observe $o_t(s,a_{t,h}^{\rm L},a_{t,h}^{\rm R})$ and $s_{t,h+1}$\label{line:po_observe}\;
  }
  Set $Z_t(s)=0$ at unvisited states, compute $\hat b_t$ by~\cref{eq:def_b_hat}, and set\label{line:po_loss_estimates}
  \begin{equation}
    \hat\ell_t(s,a)=-\frac{Z_t(s)}{2\gamma}
       \bigl(\hat b_t(s,a^{\rm L})+\hat b_t(s,a^{\rm R})\bigr).
    \label{eq:def_l_hat}
\end{equation}\par
  Compute $\hat L_{t,h}$ and $\hat Q_t(s,a)$ for all $h$ and nonterminal $(s,a)$:\label{line:po_local_estimates}\label{line:po_Q_estimate}
  \begin{equation}
    \begin{aligned}
      \hat L_{t,h}&=\sum_{j=h}^{H-1}\sum_{a'\in\calA}
           \pi_t(a'\mid s_{t,j})\hat\ell_t(s_{t,j},a'),\\
      \hat Q_t(s,a)&=\frac{\I_t(s)\pi_t(a\mid s)\hat\ell_t(s,a)
           +\I_t(s,a)\hat L_{t,h(s)+1}}{q_t(s,a)+\delta}.
    \end{aligned}
    \label{eq:def_Q_hat}
\end{equation}\par
  Update each distribution over $\calA$ by FTRL:\label{line:po_ftrl}
  \begin{equation}
    \tilde\pi_{t+1}(a\mid s)\propto
       \tilde\pi_t(a\mid s)
       \exp\bigl(-\eta(\hat Q_t(s,a)-M_t(s,a))\bigr).
    \label{eq:po_ftrl_update}
\end{equation}\par
}
\caption{Policy optimization under known transitions}
\label{alg:pf_mdp_policy_optimization}
\end{algorithm2e}

\paragraph{Q-function estimation.}
We next describe the Q-function estimate $\hat Q_t$ in \cref{eq:def_Q_hat}.
With loss feedback, the standard estimate of $Q^{\pi_t}(s,a;\ell_t)$ is the sum of the losses observed at and after $(s,a)$ in the episode, divided by $q_t(s,a)$ when $(s,a)$ is visited~\citep{luo21policy}.
With preference feedback, the losses are not observed, and the two parts of $Q^{\pi_t}(s,a;\ell_t)$ have to be estimated differently.
The immediate loss $\ell_t(s,a)$ can be estimated from any visit to $s$, since a uniform comparison at $s$ yields the estimates $\hat\ell_t(s,\cdot)$ of the losses of all pairs at $s$.
The losses at later steps can be estimated only at the visited states, and hence only from the trajectory after $(s,a)$, which is available when $a$ is executed at $s$.
Accordingly, the first term of $\hat Q_t(s,a)$ in \cref{eq:def_Q_hat} is weighted by $\I_t(s)$ and the second by $\I_t(s,a)$.

In the second term, the loss at a later visited state $s_{t,j}$ is estimated by the average $\sum_{a'}\pi_t(a'\mid s_{t,j})\hat\ell_t(s_{t,j},a')$ over pairs, since $\hat\ell_t(s_{t,j},a')$ is unbiased for every fixed pair $a'$ but not for the executed pair $a_{t,j}$, from which $\hat b_t$ is computed.
This average is unbiased for the expected loss of $\pi_t$ at $s_{t,j}$, as required by the Bellman recursion for $Q^{\pi_t}$ (\cref{lem:po_suffix_moments}).
The parameter $\delta>0$ keeps the denominators in \cref{eq:def_Q_hat} positive.

\paragraph{Bonus term.}
We next describe the bonus term $M_t$ in \cref{eq:def_bonus_M_po}.
It is needed because the decomposition~\cref{eq:state_wise_reg} weights the estimation error of $\hat Q_t(s,\cdot)$ by the comparator's visit probability $q^{\pi^*}(s)$, whereas the error grows as the learner's visit probability $q_t(s)$ decreases (\cref{lem:Qhat_est_exp}), and $q^{\pi^*}(s)$ can be much larger than $q_t(s)$.
To correct for this discrepancy, following \citet{luo21policy}, the algorithm subtracts from $\hat Q_t$ the bonus $M_t(s,a)=Q^{\pi_t}(s,a;m_t)$, where $m_t(s)$ accounts for the bias and the second moment of $\hat Q_t(s,\cdot)$.
The bonus makes FTRL favor pairs leading to rarely visited states.

\paragraph{Policy update.}
Finally, at each state $s$, the algorithm updates $\tilde\pi_t(\cdot\mid s)$ by FTRL with the negative Shannon entropy applied to the loss vectors $\hat Q_\tau(s,\cdot)-M_\tau(s,\cdot)$, which has the closed form $\tilde\pi_t(a\mid s)\propto\exp(-\eta\sum_{\tau<t}(\hat Q_\tau(s,a)-M_\tau(s,a)))$ (\cref{line:po_initialize,line:po_ftrl}).

\subsection{Regret upper bounds}\label{subsec:po_regret_bounds}
The above algorithm achieves the following bound, whose proof is given in \cref{app:proof_policy_opt_23}.
\begin{theorem}\label{thm:regret_po_23}
  Under known transitions, with the parameter choices specified in the proof, \cref{alg:pf_mdp_policy_optimization} satisfies $\E\brk{\Reg_T}=O\prn{\min\set{HT,(H^4SK^3\log K)^{1/3}T^{2/3}+(H^3+H^{3/2}K\sqrt S)\sqrt{T\log K}}}$.
\end{theorem}
Compared with \Cref{thm:reg_occ_meas}, the leading term of the bound in \cref{thm:regret_po_23} has the same dependence on $T$ and $S$ and is larger by the factor $H^{1/3}K^{2/3}$, and the second term is of lower order in $T$.
\paragraph{Technical and proof overview.}
As in the standard analysis for policy optimization~\citep{luo21policy}, subtracting the bonus $M_t$ from $\hat Q_t$ lets the analysis weight the estimation error at each state by the learner's visit probability $q_t(s)$ instead of $q^{\pi^*}(s)$ in \cref{eq:state_wise_reg}.
\Cref{lem:po_common_regret} in \cref{app:proof_policy_opt_23} states this in a form that applies to any nonnegative estimate $\hat Q_t$ with the bonus in \cref{eq:def_bonus_M_po}.
For $\eta\leq1/(9H^2)$, suppose that $\E_t\brk{\hat Q_t(s,a)}\leq Q^{\pi_t}(s,a;\ell_t)$ for every $(s,a)$ and $\sum_a\tilde\pi_t(a\mid s)\bigl(Q^{\pi_t}(s,a;\ell_t)-\E_t\brk{\hat Q_t(s,a)}+\eta\E_t\brk{\hat Q_t(s,a)^2}\bigr)\leq m_t(s)$ for every $s$.
Then $\E\brk{\Reg_T}\leq\gamma H^2T+H\log A/\eta+81\eta H^5T+\E\brk{\sum_t\sum_s q_t(s)m_t(s)}$.
We state \Cref{lem:po_common_regret} in this form since, with loss feedback, $Q^{\pi_t}(s,a;\ell_t)$ is estimated without bias from the observed losses, whereas the estimate in \cref{eq:def_Q_hat} from preference feedback, as well as its unknown-transition version in \cref{app:po_unknown}, has a different mean and a second moment that depends on $K$ and $\gamma$.

It remains to evaluate $\E_t\brk{\hat Q_t}$ and $\E_t\brk{\hat Q_t^2}$ in the conditions of \cref{lem:po_common_regret} for the estimate in \cref{eq:def_Q_hat}.
By \cref{eq:def_l_hat} and $\E_t\brk{\hat b_t(s,i)^2\mid\I_t(s)=1,Z_t(s)=1}\leq K$ (\cref{lem:bhat_unbiased}), we have $\E_t\brk{\hat\ell_t(s,a)^2\mid\I_t(s)=1}\leq K/\gamma$ (\cref{lem:po_local_moments}).
Since $\hat L_{t,h(s)+1}$ in \cref{eq:def_Q_hat} has at most $H$ terms, each of which is nonnegative and has conditional mean at most one, we can show that $\E_t\brk{\hat L_{t,h(s)+1}^2\mid\I_t(s,a)=1}\leq HK/\gamma+H^2$ (\cref{lem:po_suffix_moments}), where the first term comes from the squared terms and the second from the cross terms.
\Cref{lem:Qhat_est_exp} then gives
\[
  \E_t\brk*{\hat Q_t(s,a)}=\frac{q_t(s,a)}{q_t(s,a)+\delta}Q^{\pi_t}(s,a;\ell_t),
  \qquad
  \E_t\brk*{\hat Q_t(s,a)^2}\leq\frac{2H(K/\gamma+H)}{q_t(s,a)+\delta},
\]
so that the conditions of \cref{lem:po_common_regret} hold with $m_t(s)$ in \cref{eq:def_bonus_M_po} for $\delta=2\eta(K/\gamma+H)$.
The last term in the bound of \cref{lem:po_common_regret} is thus at most $4\delta HSAT=8\eta H(K/\gamma+H)SK^2T$, and choosing $\eta$ and $\gamma$ to balance the terms gives the bound in \cref{thm:regret_po_23}.%

\paragraph{Unknown transitions.}
When the transition kernel is unknown, we combine \Cref{alg:pf_mdp_policy_optimization} with the same confidence sets for $P$ as in \cref{sec:global_optimization}, which changes only the denominator of $\hat Q_t$ and the bonus.
The unknown $q_t(s,a)$ in \cref{eq:def_Q_hat} is replaced by its upper confidence bound, and the bonus $M_t$ takes the maximum over the confidence set at each step as in \citet{luo21policy}, with an additional term in $m_t$ for the bias caused by this replacement.
The algorithm and its proof are given in \cref{app:po_unknown}, and it achieves $\E\brk{\Reg_T}=\tilde{O}(\min\{HT,(H^4SK^3)^{1/3}T^{2/3}+H^2SK\sqrt T+H^2S^4K^2\})$, where the last two terms are again the cost of estimating the transitions (\cref{thm:regret_po_23_unk}).

\section{Regret lower bound}\label{sec:lower_bounds}
This section gives a regret lower bound for state-wise PbMDPs with Borda scores.

\begin{figure}[tbp]
  \centering
  \resizebox{0.9\linewidth}{!}{%
  \begin{tikzpicture}[
    x=0.90cm,y=1cm,>=stealth,
    font=\normalsize,
    nodebase/.style={rectangle,rounded corners=1.2pt,
      minimum height=0.70cm,minimum width=1.30cm,
      inner xsep=3pt,inner ysep=2pt,line width=0.75pt},
    xnode/.style={nodebase,draw=blue!55!black,fill=blue!5},
    tailnode/.style={nodebase,draw=blue!55!black,fill=white},
    znode/.style={nodebase,draw=orange!65!black,fill=orange!7},
    endnode/.style={circle,draw=black!80,fill=white,
      minimum size=0.67cm,inner sep=1pt,line width=0.8pt},
    chainnode/.style={nodebase,minimum width=1.28cm,
      draw=black!60,fill=white},
    flow/.style={->,line width=0.9pt,draw=blue!60!black},
    zflow/.style={->,line width=0.9pt,draw=orange!65!black},
    switch/.style={->,line width=1.0pt,draw=orange!65!black,
      dash pattern=on 6pt off 4pt},
    neutral/.style={->,line width=0.8pt,draw=black!65},
    note/.style={font=\small,align=center},
    dots/.style={inner xsep=3pt,inner ysep=0.5pt},
  ]
    \node[anchor=west] at (-0.05,2.20) {(a) $M$ chains};
    \node[endnode] (start) at (0.35,0.10) {$s_0$};
    \node[chainnode] (c1) at (2.20,1.17) {chain $1$};
    \node[chainnode,draw=blue!55!black,fill=blue!5] (cc)
      at (2.20,0.10) {chain $c$};
    \node[chainnode] (cm) at (2.20,-0.97) {chain $M$};
    \node[inner sep=0pt] at (2.20,0.635) {$\vdots$};
    \node[inner sep=0pt] at (2.20,-0.435) {$\vdots$};
    \node[endnode] (terminal) at (3.70,0.10) {$s_H$};
    \draw[neutral] (start.north east) --
      node[above left=-0.3pt,note] {$1/M$} (c1.west);
    \draw[neutral] (start.east) --
      node[above=2pt,note] {$1/M$} (cc.west);
    \draw[neutral] (start.south east) --
      node[below left=-0.3pt,note] {$1/M$} (cm.west);
    \draw[neutral] (c1.east) -- (terminal.north west);
    \draw[neutral] (cc.east) -- (terminal.west);
    \draw[neutral] (cm.east) -- (terminal.south west);
    \node[note] at (2.0,-1.75) {no transitions between chains};
    \draw[black!23,line width=0.5pt] (4.30,-1.98)--(4.30,2.37);

    \node[anchor=west] at (4.63,2.20) {(b) Chain $c$};
    \fill[blue!2,rounded corners=2pt]
      (4.60,-1.53) rectangle (13.35,1.75);
    \fill[black!2,rounded corners=2pt]
      (13.57,-1.53) rectangle (18.50,1.75);
    \draw[black!25,line width=0.5pt] (13.46,-1.53)--(13.46,1.75);
    \node[note] at (8.975,1.54) {first $L$ intermediate layers};
    \node[note] at (8.975,1.17) {preference matrix $\Pmat$ or $\Pmat^{(m)}$};
    \node[note] at (16.035,1.54) {final $R$ intermediate layers};
    \node[note] at (16.035,1.17) {$x_{c,h}$: preference matrix $\Pmat$};

    \node[xnode] (x1) at (5.35,0.12) {$x_{c,1}$};
    \node[xnode] (xh) at (7.68,0.12) {$x_{c,h}$};
    \node[xnode] (xh1) at (10.16,0.12) {$x_{c,h+1}$};
    \node[xnode] (xl) at (12.55,0.12) {$x_{c,L}$};
    \node[tailnode] (xl1) at (14.62,0.12) {$x_{c,L+1}$};
    \node[tailnode] (xend) at (17.61,0.12) {$x_{c,H-1}$};

    \node[znode] (zh) at (7.68,-1.00) {$z_{c,h}$};
    \node[znode] (zh1) at (10.16,-1.00) {$z_{c,h+1}$};
    \node[znode] (zl) at (12.55,-1.00) {$z_{c,L}$};
    \node[znode] (zl1) at (14.62,-1.00) {$z_{c,L+1}$};
    \node[znode] (zend) at (17.61,-1.00) {$z_{c,H-1}$};
    \node[endnode] (end) at (19.35,-0.44) {$s_H$};

    \draw[flow] (x1.east)--node[dots,fill=blue!2] {$\cdots$} (xh.west);
    \draw[flow] (xh.east)--
      node[above=12pt,note] {$a\notin\calA_b$} (xh1.west);
    \draw[flow] (xh1.east)--node[dots,fill=blue!2] {$\cdots$} (xl.west);
    \draw[flow] (xl.east)--(xl1.west);
    \draw[flow] (xl1.east)--node[dots,fill=black!2] {$\cdots$} (xend.west);
    \draw[flow] (xend.east)--(end.north west);

    \draw[zflow] (zh.east)--(zh1.west);
    \draw[zflow] (zh1.east)--node[dots,fill=blue!2] {$\cdots$} (zl.west);
    \draw[zflow] (zl.east)--(zl1.west);
    \draw[zflow] (zl1.east)--node[dots,fill=black!2] {$\cdots$} (zend.west);
    \draw[zflow] (zend.east)--(end.south west);

    \draw[switch] (xh.south east)
      .. controls (8.65,-0.30) and (9.04,-0.65) ..
      node[pos=0.46,above right=0pt,note,fill=blue!2,inner sep=0.8pt]
        {$a\in\calA_b$} (zh1.north west);
    \draw[switch] (xl.south east)
      .. controls (13.53,-0.32) and (13.65,-0.71) .. (zl1.north west);

    \node[note] at (9.65,-1.88)
      {penalty states: $\mathsf U=\tfrac12\ones_{K\times K}$};
    \node[note] at (16.03,-1.88)
      {loss gap per step $\geq1/5$};
  \end{tikzpicture}%
  }
  \caption{Lower-bound instance used to prove \Cref{thm:lower_bound_adv_pref_borda}.
  Each chain is selected with probability $1/M$.
  At $x_{c,h}$ with $h\leq L$, a pair in $\calA_b=\calA\setminus(\calK_g\times\calK_g)$, where $\calK_g=[\lfloor K/2\rfloor]$, leads to the penalty state $z_{c,h+1}$, and any other pair leads to $x_{c,h+1}$.
  All other transitions do not depend on the pair.}
  \label{fig:strong_pref_instance}
\end{figure}

\begin{theorem}\label{thm:lower_bound_adv_pref_borda}
  Suppose that $H\geq3$, $K\geq5$, and $S\geq2H$.
  For any learning algorithm and any $T\geq1$, there exists an episodic state-wise PbMDP with Borda scores and a known transition kernel such that $\E\brk{\Reg_T}=\Omega\prn{\min\set{HT,(H^3SK)^{1/3}T^{2/3}}}$.
\end{theorem}

\Cref{thm:lower_bound_adv_pref_borda}, whose proof is given in \Cref{app:proof_lower_bounds}, matches the upper bound of \Cref{thm:reg_occ_meas} up to the factor $(\log(SK))^{1/3}$.
Hence, under the assumptions of \Cref{thm:lower_bound_adv_pref_borda}, the minimax regret of state-wise PbMDPs with Borda scores under known transitions is $\tilde{\Theta}(\min\{HT,(H^3SK)^{1/3}T^{2/3}\})$, and the leading term of the policy optimization bound of \Cref{thm:regret_po_23} is optimal in $T$ and $S$.
If the transition kernel does not depend on the selected pair, the learner faces a separate dueling bandit problem at each state, and the minimax regret is $\tilde{\Theta}(\min\{HT,(H^2SK)^{1/3}T^{2/3}\})$ (\cref{app:pair_independent}).
The factor $H^{1/3}$ in the $T^{2/3}$ term is thus the cost of transitions that depend on the compared pair.

\paragraph{Proof overview.}
The instance used in the proof is shown in \Cref{fig:strong_pref_instance}, where $M=\Theta(S/H)$, $L=\Theta(H)$, and $R=\Theta(H)$.
It is built so that every comparison that helps to identify the Borda winner moves the learner to the penalty states.
At each state $x_{c,h}$ with $h\leq L$, the preference matrix is either $\Pmat$ in \cref{eq:strong_pref_matrices}, under which all arms in $\calK_g=[\floor{K/2}]$ have the largest Borda score, or $\Pmat^{(m)}$ with $m\in\calK_g$, which differs from $\Pmat$ only in that arm $m$ beats each arm in $\calK_b=[K]\setminus\calK_g$ with probability larger by $2\epsilon$, so that $(m,m)$ is the unique optimal pair with a loss gap of at least $\epsilon/2$ (see \cref{eq:strong_pref_normal_gap}).
The two matrices thus differ only on the pairs in $\calA_b=\calA\setminus(\calK_g\times\calK_g)$, which lead to the penalty states, where every pair has loss $1/2$, exceeding the loss of the optimal policy by at least $1/5$ at each of the final $R$ layers (see \cref{eq:strong_pref_penalty_gap}).

This instance gives the desired lower bound.
Identifying the optimal index $m$ among $\Theta(K)$ candidates requires $\Omega(K/\epsilon^2)$ comparisons in $\calA_b$, each costing $\Omega(H)$ instead of a constant, whereas a learner that does not identify $m$ incurs regret $\Omega(\epsilon T')$ at $x_{c,h}$ or $z_{c,h}$ in layer $h$, where $T'=T/M$ is the expected number of visits to the chain.
Hence, each of the $LM=\Theta(S)$ states $x_{c,h}$ with $h\leq L$ contributes $\Omega(\min\{\epsilon T',HK/\epsilon^2\})=\Omega((HK)^{1/3}(T')^{2/3})$ to the regret, where we choose $\epsilon\asymp(HK/T')^{1/3}$ as in \cref{eq:strong_pref_epsilon}, and summing over these states gives $\Omega((H^3SK)^{1/3}T^{2/3})$.

\section{Conclusion and future work}\label{sec:conclusion}
This paper studied online learning in episodic tabular MDPs with adversarial state-wise preference feedback.
For losses induced by Borda scores, we gave a global optimization algorithm and a policy optimization algorithm under known and unknown transitions, together with a lower bound showing that the first algorithm is minimax optimal up to a logarithmic factor.

Our results have several limitations.
First, the policy optimization bound is worse than the lower bound by a factor of $H^{1/3}K^{2/3}$ in the leading term, and under unknown transitions both of our bounds contain additional lower-order terms.
Whether these terms can be removed is left open.
Second, our guarantees are specific to Borda scores, since the estimators use that the Borda score is linear in the preference probabilities and that the loss of a pair is the negative average of the scores of its two arms.
Other objectives such as Condorcet winners~\citep{urvoy13generic,zoghi14relative,komiyama15regret} and von Neumann winners~\citep{dudik15contextual} deserve study.
Third, the stochastic setting and best-of-both-worlds guarantees remain to be studied, and it is unclear whether the techniques for online tabular MDPs~\citep{jin2020simultaneously,dann23best,li26data,li26policy} can be combined with learning rates tailored to minimax regret of order $T^{2/3}$~\citep{tsuchiya24simple}.
Finally, extensions to structured MDPs such as linear MDPs~\citep{jin20provably,cai20provably,neu21online} and to adversarial trajectory-comparison feedback are left for future work.

\subsection*{AI use statement}

The problem formulation, the algorithms, and the original analysis in this paper were developed without generative AI tools.
For the original manuscript, such tools were used only to proofread the English.
In revising the paper, we used generative AI tools (GPT-5.6-Sol, GPT-6-Astra, and Fable-5) to improve the regret upper bound for the global optimization algorithm and the regret lower bound.
We also used them to check the proofs in the appendix for errors, to implement the algorithms for the experiments, to create the figures, and to proofread and edit the text for readability.
We have reviewed all AI-assisted work and we take responsibility for the final content of this work.

\bibliographystyle{plainnat}
\bibliography{tex_source/ref}

\newpage

\appendix

\crefalias{section}{appendix}
\crefalias{subsection}{appendix}
\crefalias{subsubsection}{appendix}

\section{Additional related work}\label{app:related_work}
\begin{table}[t]
\centering
\small
\caption{Comparison of state-wise PbMDPs with the trajectory-comparison formulation.}
\label{tab:formulation_comparison}
\vspace{2pt}
\begin{tabular}{lll}
\toprule
 & Trajectory-comparison formulation & State-wise PbMDPs (this paper) \\
\midrule
Feedback & Preference over two trajectories & Preference over two arms per state \\
Decision unit & One arm per state (per policy) & One ordered pair of arms per state \\
Transition kernel & Depends on the chosen arm & May depend on the compared pair \\
\bottomrule
\end{tabular}
\end{table}

Several PbRL models assume that the preferences are generated from an underlying reward function or from the Q-values of an expert policy~\citep{zhu23principled,zhan24offline,zhan24rewardagnostic,wang23rlhf,sekhari23contextual}.
General preferences in sequential decision problems have also been studied through Nash equilibria of the preference game between policies~\citep{wu25multistep,colaco26reinforcement}, which generalize the von Neumann winner of dueling bandits~\citep{dudik15contextual}.
We instead study the regret for the losses induced by Borda scores, with preference probabilities that may vary arbitrarily across episodes.

Learning in state-wise PbMDPs extends adversarial dueling bandits to settings with state transitions.
For a survey of dueling bandits, we refer to \citet{bengs21preference}.
Adversarial dueling bandits with a Borda winner were studied by \citet{saha21adversarial}, who showed that the minimax regret is of order $K^{1/3}T^{2/3}$ up to logarithmic factors.
Extensions to generalized linear preference models~\citep{wu24borda} and to nonstationary preferences~\citep{suk24nonstationary} have also been studied.
Our estimator of the Borda scores differs from that of \citet{saha21adversarial}, who draw the two arms independently and reweight the outcome, in that the uniform comparisons are built into the policy chosen by FTRL (\Cref{sec:global_optimization}).
Regret of order $T^{2/3}$ also arises in other problems with limited feedback, for instance in hard partial monitoring games~\citep{bartok11minimax} and under weakly observable feedback graphs~\citep{alon2015online}.

Online learning in MDPs with adversarial losses has been studied extensively since the work of \citet{evendar09online} and \citet{neu10ssp,neu14online}.
For episodic tabular MDPs with bandit feedback, the two main approaches are online linear optimization over the set of occupancy measures~\citep{zimin13online,jin20learning,li26data} and policy optimization~\citep{shani20optimistic,luo21policy,dann23best,li26data,li26policy}, and best-of-both-worlds guarantees are also available~\citep{jin2020simultaneously,jin2021best,dann23best,li26data,li26policy}.
Our two algorithms follow these two approaches, respectively.
Existing regret lower bounds for tabular MDPs~\citep{jaksch10thomas,jin18is,domingues21episodic} concern loss feedback, whereas our lower bound concerns preference feedback and uses transitions that depend on the compared pair (\Cref{sec:lower_bounds}).

\section{Deferred proofs for the global optimization approach}\label{app:proof_global_opt}
Here we provide the proofs of the results in \Cref{sec:global_optimization}.
\Cref{lem:bhat_unbiased,lem:global_nested_ftrl,lem:global_comparator_smoothing} below are used in both the known- and unknown-transition proofs.
We use $\E_t$ as in \Cref{sec:mdp_adv_pref}, including the learner's past random choices in the conditioning, and write $\P_t(E)=\E_t\brk{\ind{E}}$ for the conditional probability of an event $E$.
Sums over $s\in\calS$ exclude the terminal state, and conditional expectations given a visit are used only when the conditioning event has positive conditional probability.

\subsection{Borda-score estimates}
The first lemma concerns the estimate in~\cref{eq:def_b_hat} when the pair compared at $s$ is drawn uniformly from $\calA$, that is, when $Z_t(s)=1$.
\begin{lemma}\label{lem:bhat_unbiased}
  For the estimate in~\cref{eq:def_b_hat},
  \begin{align*}
    \E_t\brk*{\hat b_t(s,i)\relmiddle|\I_t(s)=1,Z_t(s)=1}
    &=b_t(s,i),
    \\
    \E_t\brk*{\hat b_t(s,i)^2\relmiddle|\I_t(s)=1,Z_t(s)=1}
    &\leq K.
\end{align*}
\end{lemma}
\begin{proof}
  When $Z_t(s)=1$, the pair compared at $s$ is drawn uniformly from $\calA$, so that each ordered pair $(u,v)$ has conditional probability $1/K^2$.
  Moreover, by the feedback protocol in \Cref{sec:mdp_adv_pref},
  \begin{equation*}
    \E_t\brk*{o_t(s,u,v)\relmiddle|\I_t(s)=1,Z_t(s)=1,
      a_{t,s}^{\rm L}=u,a_{t,s}^{\rm R}=v}=\Pmat_t(s,(u,v)).
\end{equation*}
  Consequently,
  \begin{align*}
    \E_t\brk*{\hat b_t(s,i)\relmiddle|\I_t(s)=1,Z_t(s)=1}
    &=K\sum_{u,v\in[K]}\frac{\ind{i=u}}{K^2}
       \prn[\big]{\Pmat_t(s,(u,v))-1}\nn
    &=\frac1K\sum_{v\in[K]}\prn[\big]{\Pmat_t(s,(i,v))-1}=b_t(s,i).
\end{align*}
  Since $(o-1)^2=1-o$ for $o\in\{0,1\}$, we similarly obtain
  \begin{align*}
    \E_t\brk*{\hat b_t(s,i)^2\relmiddle|\I_t(s)=1,Z_t(s)=1}
    &=K^2\sum_{u,v\in[K]}\frac{\ind{i=u}}{K^2}
       \prn[\big]{1-\Pmat_t(s,(u,v))}\nn
    &=-Kb_t(s,i)\leq K,
\end{align*}
  which completes the proof.
\end{proof}

For known transitions, we omit states unreachable under every policy.
The remaining states have positive visit probability under every policy satisfying $\pi(a\mid s)\geq\rho/A$, since every path with positive transition probabilities also has positive probability under such a policy.
Thus the denominators in~\cref{eq:global_borda_estimator} are positive on these states.

\begin{lemma}\label{lem:global_known_estimator_moments}
  The estimates in~\cref{eq:global_borda_estimator} satisfy
  \begin{equation*}
    \E_t\brk*{\widetilde b_t(s,i)}=b_t(s,i),\qquad
    \E_t\brk*{\hat\ell_t(s,a)}=\ell_t(s,a),\qquad
    \E_t\brk*{\widetilde b_t(s,i)^2}\leq\frac{K}{\rho\widetilde q_t(s)}.
\end{equation*}
  Moreover,
  \begin{equation}
    \E_t\brk*{\sum_{s,a}\widetilde q_t(s,a)\hat\ell_t(s,a)^2}
    \leq\frac{SK}{\rho}.
    \label{eq:global_known_total_second_moment}
\end{equation}
\end{lemma}
\begin{proof}
  Since the sampling rule in \cref{line:global_coin} of \Cref{alg:pf_mdp_occ} selects each pair with probability $\pi_t(a\mid s)$, the occupancy measure of the executed policy is $\widetilde q_t$.
  Since $Z_t(s)\sim\ber(\rho)$ is drawn when $s$ is visited, independently of the trajectory up to $s$, we have
  \begin{equation*}
    \P_t(\I_t(s)=1,Z_t(s)=1)=\rho\widetilde q_t(s).
\end{equation*}
Hence, by the definition of $\widetilde b_t$ in~\cref{eq:global_borda_estimator} and \cref{lem:bhat_unbiased}, we have
  \begin{align*}
    \E_t\brk*{\widetilde b_t(s,i)}
    &=\frac{\rho\widetilde q_t(s)}{\rho\widetilde q_t(s)}
      \E_t\brk*{\hat b_t(s,i)\relmiddle|\I_t(s)=1,Z_t(s)=1}=b_t(s,i),
      \\
    \E_t\brk*{\widetilde b_t(s,i)^2}
    &=\frac{\rho\widetilde q_t(s)}{\rho^2\widetilde q_t(s)^2}
      \E_t\brk*{\hat b_t(s,i)^2\relmiddle|\I_t(s)=1,Z_t(s)=1}
      \leq\frac{K}{\rho\widetilde q_t(s)}.
\end{align*}
  By the definition of $\hat\ell_t$ in~\cref{eq:global_borda_estimator} and the first identity,
  \begin{equation*}
    \E_t\brk*{\hat\ell_t(s,a)}
    =-\frac{\E_t\brk{\widetilde b_t(s,a^{\rm L})}+\E_t\brk{\widetilde b_t(s,a^{\rm R})}}{2}
    =-\frac{b_t(s,a^{\rm L})+b_t(s,a^{\rm R})}{2}=\ell_t(s,a).
\end{equation*}

  For each state $s$, we have
  \begin{align*}
    \E_t\brk*{\sum_{i,j}\widetilde q_t(s,(i,j))\hat\ell_t(s,(i,j))^2}
    &\leq\frac12\sum_{i,j}\widetilde q_t(s,(i,j))
      \prn[\Big]{\E_t\brk[\big]{\widetilde b_t(s,i)^2}+\E_t\brk[\big]{\widetilde b_t(s,j)^2}}\nn
    &\leq\frac{K}{\rho\widetilde q_t(s)}
      \sum_{i,j}\widetilde q_t(s,(i,j))=\frac K\rho,
\end{align*}
  where the first inequality follows from the definition of $\hat\ell_t$ and $(x+y)^2\leq2x^2+2y^2$, and the second inequality from the bound on $\E_t\brk{\widetilde b_t(s,i)^2}$ shown above.
  Summing over the states completes the proof.
\end{proof}

\subsection{FTRL with decreasing feasible sets}
In the unknown-transition algorithm of \cref{app:global_unknown}, the feasible set is determined by transition confidence sets that can only shrink over time.
We state the FTRL bound for decreasing feasible sets, which includes the case of a fixed feasible set.
\begin{lemma}\label{lem:global_nested_ftrl}
  Fix $\eta>0$, and let $\calQ_1\supseteq\cdots\supseteq\calQ_n$ be nonempty compact convex subsets of $\Rnn^{\calS\times\calA}$ satisfying
    $\sum_{s\in\calS_h,a\in\calA}q(s,a)=1$
  for every $q\in\calQ_1$ and every $h\in\{0,\ldots,H-1\}$.
  For nonnegative loss vectors $\hat\ell_1,\ldots,\hat\ell_n$, let
  \begin{equation*}
    \widetilde q_t\in\argmin_{q\in\calQ_t}
    \set*{\inpr*{q,\sum_{\tau=1}^{t-1}\hat\ell_\tau}-\frac1\eta\H(q)}.
\end{equation*}
  Then, for every $q\in\calQ_n$,
  \begin{equation*}
    \sum_{t=1}^n\inpr{\widetilde q_t-q,\hat\ell_t}
    \leq\frac{H\log(SA)}\eta
      +\frac\eta2\sum_{t=1}^n\sum_{s,a}
        \widetilde q_t(s,a)\hat\ell_t(s,a)^2.
\end{equation*}
\end{lemma}
\begin{proof}
  Set $\calQ_{n+1}=\calQ_n$ and define the minimum objective value
  \begin{equation*}
    \Phi_t=\min_{u\in\calQ_t}
      \set*{\inpr*{u,\sum_{\tau=1}^{t-1}\hat\ell_\tau}-\frac1\eta\H(u)}.
\end{equation*}
  We use the unnormalized relative entropy
  \begin{equation*}
    D(u,v)=\sum_{s,a}
      \prn*{u(s,a)\log\frac{u(s,a)}{v(s,a)}-u(s,a)+v(s,a)}
      =-\H(u)+\H(v)+\inpr{\nabla\H(v),u-v}.
\end{equation*}
  We drop the coordinates that vanish on all of $\calQ_t$.
  On every other coordinate, $\widetilde q_t$ is positive, since the derivative of $y\log y$ tends to $-\infty$ as $y\to0$, and hence the first-order optimality condition gives
  \begin{equation*}
    \inpr*{\sum_{\tau<t}\hat\ell_\tau-\frac1\eta\nabla\H(\widetilde q_t),
      u-\widetilde q_t}\geq0
  \end{equation*}
  for every $u\in\calQ_t$.
  Substituting the identity for $D$ gives
  \begin{equation*}
    \inpr*{u,\sum_{\tau<t}\hat\ell_\tau}-\frac1\eta\H(u)
    \geq\Phi_t+\frac1\eta D(u,\widetilde q_t)
  \end{equation*}
  for every $u\in\calQ_t$.
  Since $\calQ_{t+1}\subseteq\calQ_t$, this lower bound holds for every $u\in\calQ_{t+1}$.
  Dropping all constraints except $u\geq0$ then yields
  \begin{align}
    \Phi_{t+1}-\Phi_t
    &\geq\min_{u\in\calQ_{t+1}}
      \set*{\inpr{u,\hat\ell_t}+\frac1\eta D(u,\widetilde q_t)}\nn
    &\geq\sum_{s,a}\min_{y\geq0}
      \left\{y\hat\ell_t(s,a)+\frac1\eta
        \prn*{y\log\frac{y}{\widetilde q_t(s,a)}-y+\widetilde q_t(s,a)}\right\}\nn
    &=\frac1\eta\sum_{s,a}\widetilde q_t(s,a)
       \prn[\big]{1-\e^{-\eta\hat\ell_t(s,a)}}\nn
    &\geq\inpr{\widetilde q_t,\hat\ell_t}
       -\frac\eta2\sum_{s,a}\widetilde q_t(s,a)\hat\ell_t(s,a)^2,
    \label{eq:global_nested_stability_bound}
\end{align}
  where the equality follows from the fact that the minimum over $y\geq0$ is attained at $y=\widetilde q_t(s,a)\e^{-\eta\hat\ell_t(s,a)}$, and the last inequality from $\e^{-z}\leq1-z+z^2/2$ for $z\geq0$.

  For $q\in\calQ_n=\calQ_{n+1}$, we have $\Phi_{n+1}\leq\inpr{q,\sum_{t=1}^n\hat\ell_t}-\H(q)/\eta$ and $\Phi_1=-\max_{u\in\calQ_1}\H(u)/\eta$.
  Summing~\cref{eq:global_nested_stability_bound} gives
  \begin{align*}
    \sum_{t=1}^n\inpr{\widetilde q_t-q,\hat\ell_t}
    &\leq\Phi_{n+1}-\Phi_1-\inpr*{q,\sum_{t=1}^n\hat\ell_t}
      +\frac\eta2\sum_{t=1}^n\sum_{s,a}\widetilde q_t(s,a)\hat\ell_t(s,a)^2\nn
    &\leq\frac{\max_{u\in\calQ_1}\H(u)-\H(q)}\eta
      +\frac\eta2\sum_{t=1}^n\sum_{s,a}\widetilde q_t(s,a)\hat\ell_t(s,a)^2.
\end{align*}
  By the assumption on $\calQ_1$, the restriction of $q$ to each layer $\calS_h\times\calA$ is a probability distribution, whose entropy lies in $[0,\log(S_hA)]$.
  Summing over the layers gives $0\leq\H(q)\leq\sum_{h=0}^{H-1}\log(S_hA)\leq H\log(SA)$, which completes the proof.
\end{proof}

\subsection{Cost of the uniform component}
The occupancy measure of the best policy need not belong to $\Omega_\rho$.
The following lemma bounds the cost of adding the required uniform component.
\begin{lemma}\label{lem:global_comparator_smoothing}
  For a Markov policy $\pi$ and $\rho\in[0,1]$, define
  \begin{equation}
    \pi_\rho(a\mid s)=(1-\rho)\pi(a\mid s)+\rho/A.
    \label{eq:global_smoothed_policy}
\end{equation}
  For every $\ell\colon\calS\times\calA\to[0,1]$,
  \begin{equation*}
    V^{\pi_\rho}(s_0;\ell)-V^\pi(s_0;\ell)\leq\rho H^2.
\end{equation*}
\end{lemma}
\begin{proof}
  By the performance difference lemma~\citep{kakade02approximately}, we have
  \begin{align*}
    V^{\pi_\rho}(s_0;\ell)-V^\pi(s_0;\ell)
    &=\sum_s q^{\pi_\rho}(s)
      \inpr{\pi_\rho(\cdot\mid s)-\pi(\cdot\mid s),Q^\pi(s,\cdot;\ell)}\nn
    &=\rho\sum_s q^{\pi_\rho}(s)
      \left(\frac1A\sum_a Q^\pi(s,a;\ell)
        -\sum_a\pi(a\mid s)Q^\pi(s,a;\ell)\right)\nn
    &\leq\rho H\sum_s q^{\pi_\rho}(s)=\rho H^2,
\end{align*}
  where the inequality follows from $Q^\pi(s,a;\ell)\in[0,H]$ for every $(s,a)$ and the last equality from $\sum_s q^{\pi_\rho}(s)=H$.
  This completes the proof.
\end{proof}

\subsection{Known transitions (Proof of \cref{thm:reg_occ_meas})}
\begin{proof}[Proof of \cref{thm:reg_occ_meas}]
  Let $\pi^*$ be the best fixed Markov policy in hindsight and set $q_\rho^*=q^{\pi_\rho^*}$, with $\pi_\rho^*$ defined by~\cref{eq:global_smoothed_policy} in~\cref{lem:global_comparator_smoothing}.
  Then $q_\rho^*\in\Omega_\rho$.
  By \cref{lem:global_comparator_smoothing},
  \begin{equation*}
    \E\brk*{\Reg_T}\leq\rho H^2T
      +\E\brk*{\sumT\inpr{\widetilde q_t-q_\rho^*,\ell_t}}.
\end{equation*}
  Since $\widetilde q_t$ and $q_\rho^*$ are determined before episode $t$, \cref{lem:global_known_estimator_moments} gives
  \begin{align*}
    \E\brk*{\sumT\inpr{\widetilde q_t-q_\rho^*,\ell_t}}
    =\E\brk*{\sumT\inpr{\widetilde q_t-q_\rho^*,\E_t\brk{\hat\ell_t}}}
    =\E\brk*{\sumT\inpr{\widetilde q_t-q_\rho^*,\hat\ell_t}}.
\end{align*}
  Since $\hat\ell_t\geq0$ by $o_t\leq1$, this is further upper bounded as
  \begin{align*}
    \E\brk*{\sumT\inpr{\widetilde q_t-q_\rho^*,\hat\ell_t}}
    &\leq\frac{H\log(SA)}\eta
      +\frac\eta2\sumT\E\brk*{\sum_{s,a}\widetilde q_t(s,a)\hat\ell_t(s,a)^2}\nn
    &\leq\frac{H\log(SA)}\eta+\frac{\eta SKT}{2\rho},
\end{align*}
  where the first inequality follows from \cref{lem:global_nested_ftrl} with $\calQ_t=\Omega_\rho$ for every $t$ and the second inequality from~\cref{eq:global_known_total_second_moment} in~\cref{lem:global_known_estimator_moments}.
  Combining these inequalities yields
  \begin{equation}
    \E\brk*{\Reg_T}\leq\frac{H\log(SA)}\eta+\frac{\eta SKT}{2\rho}+\rho H^2T.
    \label{eq:global_known_before_parameters}
\end{equation}
  We choose
  \begin{equation*}
    \eta=\prn*{\frac{\log(SA)}{T\sqrt{SK}}}^{2/3},
    \qquad
    \rho=\min\set*{\frac12,\frac{\sqrt{\eta SK}}H}.
\end{equation*}
  If $\sqrt{\eta SK}/H\leq1/2$, then
  \begin{align*}
    \frac{\eta SKT}{2\rho}+\rho H^2T
    &=\frac32HT\sqrt{\eta SK},
    \\
    \frac{H\log(SA)}\eta
    &=HT\sqrt{\eta SK}
      =\prn[\big]{H^3SK\log(SA)}^{1/3}T^{2/3},
\end{align*}
  and hence the right-hand side of~\cref{eq:global_known_before_parameters} is at most $\frac52(H^3SK\log(SA))^{1/3}T^{2/3}$.
  Otherwise,
  \begin{equation*}
    \frac{(H^3SK\log(SA))^{1/3}T^{2/3}}{HT}
      =\prn*{\frac{SK\log(SA)}T}^{1/3}
      =\sqrt{\eta SK}>\frac H2\geq\frac12,
\end{equation*}
  and hence $HT\leq2(H^3SK\log(SA))^{1/3}T^{2/3}$.
  Combining the two cases with the trivial bound $\Reg_T\leq HT$, we obtain $\E\brk{\Reg_T}\leq\frac52\min\set{HT,(H^3SK\log(SA))^{1/3}T^{2/3}}$ for every $T\geq1$.
  Using $A=K^2$ and $\log(SA)\leq2\log(SK)$ completes the proof.
\end{proof}

\subsection{Unknown transitions}\label{app:global_unknown}
We now turn to the case where the transition kernel $P$ is fixed but unknown.
The algorithm keeps the FTRL update and the uniform-comparison rule of \Cref{alg:pf_mdp_occ}, but the feasible set of FTRL is now determined by a confidence set for $P$, and the estimator divides by an upper confidence bound on the state-visit probability instead of by $\widetilde q_t(s)$.

\begin{theorem}\label{thm:reg_occ_unknown}
  Suppose that the transition kernel is fixed but unknown.
  With the parameter choices specified in the proof, \Cref{alg:pf_mdp_occ_unknown} satisfies, for every $T\geq1$,
  \begin{equation}
    \E\brk*{\Reg_T}
    =\tilde{O}\prn*{\min\set*{HT,\prn{H^3SK}^{1/3}T^{2/3}+HSK\sqrt T+HS^4K^2}}.
    \label{eq:global_unknown_bound}
\end{equation}
\end{theorem}

The last two terms in~\cref{eq:global_unknown_bound} are the cost of estimating the transitions.
The term $HSK\sqrt T$ comes from the width of the confidence sets, and $HS^4K^2$ comes from the lower-order term of the bound of \citet{dann23best} on the difference between the occupancy measures under the true transition kernel and the kernels in the confidence sets (\cref{lem:transition_occupancy_deviation}).

\subsubsection{Transition confidence sets}
We construct the confidence sets as in \citet{jin20learning}.
The episodes are divided into epochs, and the empirical transition kernel and its entrywise Bernstein confidence intervals are updated only at the beginning of each epoch.
Let $N_k(s,a)$ and $M_k(s,a,s')$ be the visit counts before epoch $k$, and let $N(s,a)$ and $M(s,a,s')$ denote the current counts.
An epoch ends when $N(s,a)\geq\max\{1,2N_k(s,a)\}$ for some visited $(s,a)$.
For $k\geq2$, define
\begin{align}
    \bar P_k(s'\mid s,a)&=\frac{M_k(s,a,s')}{\max\{1,N_k(s,a)\}},
  \label{eq:global_empirical_transition}\\
  \beta_k(s,a,s')
  &=2\sqrt{\frac{\bar P_k(s'\mid s,a)\iota}{\max\{1,N_k(s,a)-1\}}}
    +\frac{14\iota}{3\max\{1,N_k(s,a)-1\}},
  \qquad \iota=\log\frac{TSA}{\delta}.
  \label{eq:global_transition_radius}
\end{align}
Here $\delta\in(0,1)$ is a confidence parameter.
Let $\calP_1$ be the set of all layered transition kernels and set
\begin{equation}
    \calC_k=\set*{P'\in\calP_1:
    |P'(s'\mid s,a)-\bar P_k(s'\mid s,a)|\leq\beta_k(s,a,s')
    \text{ for all }(s,a,s')\text{ with }s'\in\calS_{h(s)+1}}.
  \label{eq:global_raw_confidence_set}
\end{equation}
The remaining entries of a layered kernel are zero.
With $\calC_1=\calP_1$, set
\begin{equation}
    \calP_k=
  \begin{cases}
    \calP_{k-1}\cap\calC_k,&\text{if }\calP_{k-1}\cap\calC_k\ne\varnothing,\\
    \calP_{k-1},&\text{otherwise}.
  \end{cases}
  \label{eq:global_nested_confidence_set}
\end{equation}
Under this rule, $\calP_k$ is nonempty and $\calP_k\subseteq\calP_{k-1}$ for every history.
If $P\in\calC_k$ for every epoch, the first case always applies and $P\in\calP_k\subseteq\calC_k$.

The constraints defining $\calP_k$ are intervals for the individual entries $P'(s'\mid s,a)$.
For every $(s,a,s')$ with $s'\in\calS_{h(s)+1}$, define
\begin{equation*}
    P_k^-(s'\mid s,a)=\min_{P'\in\calP_k}P'(s'\mid s,a),
  \qquad
  P_k^+(s'\mid s,a)=\max_{P'\in\calP_k}P'(s'\mid s,a).
\end{equation*}
Thus
\begin{equation*}
    \calP_k=\{P'\in\calP_1:P_k^-\leq P'\leq P_k^+\},
\end{equation*}
where both inequalities hold for every $(s,a,s')$ with $s'\in\calS_{h(s)+1}$.

\subsubsection{Extended occupancy measures}\label{app:global_extended_occupancy}
Since $P$ is unknown, the algorithm replaces $\Omega_\rho$ by the set of occupancy measures that are consistent with some kernel in $\calP_k$.
To describe this set by linear constraints, we use occupancy measures over triples $(s,a,s')$, as in \citet{jin20learning}.
For a kernel $P'$ and a policy $\pi$, write $q^{P',\pi}$ for the occupancy measure of $\pi$ under $P'$.
For $(s,a,s')$ with $s'\in\calS_{h(s)+1}$, define the extended occupancy measure
\begin{equation*}
    x(s,a,s')=q^{P',\pi}(s,a)P'(s'\mid s,a),
  \qquad
  q_x(s,a)=\sum_{s'}x(s,a,s'),\quad q_x(s)=\sum_a q_x(s,a).
\end{equation*}
Let $\calX_\rho(\calP_k)$ be the set of all $x(s,a,s')\geq0$ satisfying
\begin{align*}
    \sum_{a\in\calA,s'\in\calS_1}x(s_0,a,s')&=1,
    \\
  \sum_{s^-\in\calS_{h-1},a}x(s^-,a,s)&=q_x(s)
    &&\text{for every }s\in\calS_h\text{ with }1\leq h<H,
    \\
  P_k^-(s'\mid s,a)q_x(s,a)&\leq x(s,a,s')\leq P_k^+(s'\mid s,a)q_x(s,a),
    \\
  q_x(s,a)&\geq\rho q_x(s)/A.
\end{align*}
The last two constraints are imposed for all $(s,a,s')$ with $s'\in\calS_{h(s)+1}$ and for all nonterminal $(s,a)$, respectively.
All the constraints above are linear.
Since the extended occupancy measure of the uniform policy under any $P'\in\calP_k$ is feasible, $\calX_\rho(\calP_k)$ is a nonempty convex polytope.
Its projection
\begin{equation*}
    \Omega_\rho(\calP_k)=\{q_x:x\in\calX_\rho(\calP_k)\}
\end{equation*}
is also convex and satisfies $\Omega_\rho(\calP_{k+1})\subseteq\Omega_\rho(\calP_k)$.

We now show that $\Omega_\rho(\calP_k)$ consists exactly of the occupancy measures $q^{P',\pi}$ with $P'\in\calP_k$ and $\pi(a\mid s)\geq\rho/A$ for all $(s,a)$.
For $x\in\calX_\rho(\calP_k)$, define
\begin{equation*}
    P_x(s'\mid s,a)=\frac{x(s,a,s')}{q_x(s,a)},
  \qquad \pi_x(a\mid s)=\frac{q_x(s,a)}{q_x(s)}
\end{equation*}
whenever the denominators are positive.
If $q_x(s,a)=0$, let $P_x(\cdot\mid s,a)$ be $P'(\cdot\mid s,a)$ for an arbitrary fixed $P'\in\calP_k$, and if $q_x(s)=0$, let $\pi_x(\cdot\mid s)$ be uniform.
Since the constraints defining $\calP_k$ are imposed separately on each $P'(\cdot\mid s,a)$, we have $P_x\in\calP_k$, and the last constraint in the definition of $\calX_\rho(\calP_k)$ gives $\pi_x(a\mid s)\geq\rho/A$.
By induction over the layers starting from $s_0$, the second constraint gives $q_x=q^{P_x,\pi_x}$.
Conversely, for any $P'\in\calP_k$ and any policy $\pi$ with $\pi(a\mid s)\geq\rho/A$, the extended occupancy measure $x(s,a,s')=q^{P',\pi}(s,a)P'(s'\mid s,a)$ satisfies all constraints.
In particular, $\Omega_\rho\subseteq\Omega_\rho(\calP_k)$ whenever $P\in\calP_k$.

\subsubsection{Algorithm under unknown transitions}
Let $k_t$ be the epoch containing episode $t$.
The algorithm replaces the update~\cref{eq:ftrl_occ} of \Cref{alg:pf_mdp_occ} by
\begin{equation}
    \widetilde x_t\in\argmin_{x\in\calX_\rho(\calP_{k_t})}
    \set*{\inpr*{q_x,\sum_{\tau<t}\hat\ell_\tau}-\frac1\eta\H(q_x)},
  \qquad \widetilde q_t=q_{\widetilde x_t}.
  \label{eq:global_unknown_ftrl}
\end{equation}
Since the objective depends only on $q_x$, $\widetilde q_t$ is the FTRL minimizer over $\Omega_\rho(\calP_{k_t})$.
The algorithm sets $\pi_t=\pi_{\widetilde x_t}$, defines $\pi_t^{\rm r}$ by~\cref{eq:global_residual_policy}, and computes
\begin{equation}
    u_t(s)=\max_{P'\in\calP_{k_t}}q^{P',\pi_t}(s).
  \label{eq:global_upper_state_probability}
\end{equation}
The maximization is over the kernel only, and $u_t$ can be computed by dynamic programming over the intervals defining $\calP_{k_t}$~\citep[Section~3.3]{jin20learning}.
The algorithm executes $\pi_t$ with the uniform-comparison rule of \Cref{alg:pf_mdp_occ} and replaces the estimator in~\cref{eq:global_borda_estimator} by
\begin{equation}
    \widetilde b_t(s,i)=\frac{\I_t(s)Z_t(s)\hat b_t(s,i)}{\rho u_t(s)},
  \qquad
  \hat\ell_t(s,a)=-\frac{\widetilde b_t(s,a^{\rm L})+\widetilde b_t(s,a^{\rm R})}{2},
  \label{eq:global_unknown_borda_estimator}
\end{equation}
with the first estimate defined as zero when $u_t(s)=0$.

\begin{algorithm2e}[t]
\small
\KwIn{Learning rate $\eta>0$, exploration rate $\rho\in(0,1/2]$, and $\delta\in(0,1)$}
Set $k=1$, let $\calP_1=\calC_1$ be the set of all layered transition kernels, and initialize all counts $N,M,N_1,M_1$ to zero\;
\For{$t=1,\ldots,T$}{
  Set $k_t=k$ and compute $\widetilde x_t,\widetilde q_t$ by~\cref{eq:global_unknown_ftrl}\label{line:global_unknown_ftrl}\;
  Set $\pi_t=\pi_{\widetilde x_t}$, and compute $\pi_t^{\rm r}$ and $u_t$ by~\cref{eq:global_residual_policy,eq:global_upper_state_probability}\label{line:global_unknown_upper}\;
  Execute $\pi_t$ as in \crefrange{line:global_execute}{line:global_observe} of \Cref{alg:pf_mdp_occ}\label{line:global_unknown_execute}\;
  Set $Z_t(s)=0$ at unvisited states and compute the estimates by~\cref{eq:def_b_hat,eq:global_unknown_borda_estimator}\label{line:global_unknown_estimate}\;
  Increment $N(s,a)$ and $M(s,a,s')$ for every observed transition $(s,a,s')$\label{line:global_unknown_counts}\;
  \If{$t<T$ and $N(s,a)\geq\max\{1,2N_k(s,a)\}$ for some visited $(s,a)$}{
    Set $k\leftarrow k+1$, $N_k\leftarrow N$, and $M_k\leftarrow M$\label{line:global_unknown_epoch}\;
    Compute $\bar P_k,\beta_k,\calC_k$ by~\cref{eq:global_empirical_transition,eq:global_transition_radius,eq:global_raw_confidence_set}, and update $\calP_k$ by~\cref{eq:global_nested_confidence_set}\label{line:global_unknown_confidence}\;
  }
}
\caption{Global optimization under unknown transitions}
\label{alg:pf_mdp_occ_unknown}
\end{algorithm2e}

Three properties of \Cref{alg:pf_mdp_occ_unknown} are used in the analysis.
The quantities $\widetilde q_t$, $\pi_t$, and $u_t$ are determined by the observations before episode $t$ (\crefrange{line:global_unknown_ftrl}{line:global_unknown_upper}), the transition counts are generated by the executed policies $\pi_t$ (\cref{line:global_unknown_counts}), and the feasible set changes between epochs without restarting FTRL (\crefrange{line:global_unknown_epoch}{line:global_unknown_confidence}).
If $\calP_k=\{P\}$ for every $k$, then $\widetilde x_t(s,a,s')=\widetilde q_t(s,a)P(s'\mid s,a)$ and $u_t(s)=\widetilde q_t(s)$, and the algorithm reduces to \Cref{alg:pf_mdp_occ}.

\subsubsection{Borda-score estimates under unknown transitions}
In what follows, let $q_t=q^{P,\pi_t}$ be the occupancy measure of $\pi_t$ under the true transition kernel.
In general, $q_t$ differs from the FTRL iterate $\widetilde q_t$.
\begin{lemma}\label{lem:global_unknown_estimator_moments}
  Suppose that $P\in\calP_{k_t}$.
  Then, for every state $s$ with $u_t(s)>0$,
  \begin{align}
    \E_t\brk*{\widetilde b_t(s,i)}&=\frac{q_t(s)}{u_t(s)}b_t(s,i),
    \nonumber\\
    \E_t\brk*{\hat\ell_t(s,a)}&=\frac{q_t(s)}{u_t(s)}\ell_t(s,a)\leq\ell_t(s,a),
    \label{eq:global_unknown_loss_mean}\\
    \E_t\brk*{\widetilde b_t(s,i)^2}&\leq\frac{K}{\rho u_t(s)}.
    \nonumber
\end{align}
  For every state $s$ with $u_t(s)=0$, both $q_t(s)$ and $\widetilde q_t(s)$ are zero and the estimates at $s$ are zero.
  Moreover, under the same assumption,
  \begin{equation}
    \E_t\brk*{\sum_{s,a}\widetilde q_t(s,a)\hat\ell_t(s,a)^2}\leq\frac{SK}{\rho}.
    \label{eq:global_unknown_total_second_moment}
\end{equation}
\end{lemma}
\begin{proof}
  Since both $P$ and $P_{\widetilde x_t}$ belong to $\calP_{k_t}$, the definition of $u_t$ in~\cref{eq:global_upper_state_probability} gives
  \begin{equation*}
    q_t(s)\leq u_t(s),\qquad \widetilde q_t(s)\leq u_t(s).
\end{equation*}
  If $q_t(s)=0$, then $s$ is not visited, and the claims hold since the estimates at $s$ are zero.
  Suppose that $q_t(s)>0$.
  Then $\P_t(\I_t(s)=1,Z_t(s)=1)=\rho q_t(s)>0$, and \cref{lem:bhat_unbiased} gives
  \begin{align*}
    \E_t\brk*{\widetilde b_t(s,i)}
    &=\frac{\rho q_t(s)}{\rho u_t(s)}
      \E_t\brk*{\hat b_t(s,i)\relmiddle|\I_t(s)=1,Z_t(s)=1}
      =\frac{q_t(s)}{u_t(s)}b_t(s,i),
    \\
    \E_t\brk*{\widetilde b_t(s,i)^2}
    &=\frac{\rho q_t(s)}{\rho^2u_t(s)^2}
      \E_t\brk*{\hat b_t(s,i)^2\relmiddle|\I_t(s)=1,Z_t(s)=1}\nn
    &\leq\frac{Kq_t(s)}{\rho u_t(s)^2}\leq\frac{K}{\rho u_t(s)}.
\end{align*}
  By the definition of $\hat\ell_t$ in~\cref{eq:global_unknown_borda_estimator} and the first identity,
  \begin{equation*}
    \E_t\brk*{\hat\ell_t(s,a)}
    =-\frac{\E_t\brk{\widetilde b_t(s,a^{\rm L})}+\E_t\brk{\widetilde b_t(s,a^{\rm R})}}{2}
    =\frac{q_t(s)}{u_t(s)}\ell_t(s,a)
    \leq\ell_t(s,a),
\end{equation*}
  where the inequality follows from $q_t(s)\leq u_t(s)$ and $\ell_t(s,a)\geq0$.

  As in the proof of~\cref{lem:global_known_estimator_moments}, for every state $s$ with $u_t(s)>0$ we have
  \begin{align*}
    \E_t\brk*{\sum_{i,j}\widetilde q_t(s,(i,j))\hat\ell_t(s,(i,j))^2}
    &\leq\frac12\sum_{i,j}\widetilde q_t(s,(i,j))
      \prn[\Big]{\E_t\brk[\big]{\widetilde b_t(s,i)^2}+\E_t\brk[\big]{\widetilde b_t(s,j)^2}}\nn
    &\leq\frac{K\widetilde q_t(s)}{\rho u_t(s)}\leq\frac K\rho,
\end{align*}
  where the last inequality follows from $\widetilde q_t(s)\leq u_t(s)$.
  Since states with $u_t(s)=0$ contribute zero, summing over the states completes the proof.
\end{proof}

\subsubsection{Transition estimation error}\label{app:transition_estimation_error}
In this section, $\pi_t$ may be any Markov policy determined by the observations before episode $t$, and $q_t=q^{P,\pi_t}$.
The counts and the confidence sets $\calC_k$ are defined as in~\cref{eq:global_empirical_transition,eq:global_transition_radius,eq:global_raw_confidence_set}, computed from the trajectories of these policies.
By \citet[Lemma~2]{jin20learning},
\begin{equation}
    \P(P\in\calC_k\text{ for every epoch }k)\geq1-4\delta.
  \label{eq:global_raw_confidence_probability}
\end{equation}
The following bound on the occupancy measures under kernels in the confidence sets is used for both global and policy optimization.

\begin{lemma}\label{lem:transition_occupancy_deviation}
  Suppose that $T\geq A$, and let $\delta\in(0,1)$ be the confidence parameter in~\cref{eq:global_transition_radius}.
  There are universal constants $C,C_0>0$ such that, with probability at least $1-C_0H\delta$, $P\in\calC_k$ for every epoch $k$, and
  \begin{equation}
    \sumT\sum_s
      |q^{P_t^s,\pi_t}(s)-q_t(s)|
    \leq C\prn*{HS\sqrt{AT}\,\iota+HS^4A\iota^2}
    \label{eq:po_imported_occupancy_deviation}
\end{equation}
  holds simultaneously for every choice of kernels $P_t^s\in\calC_{k_t}$, one for each $(t,s)$, where $\iota=\log(TSA/\delta)$.
\end{lemma}

\Cref{lem:transition_occupancy_deviation} follows from \citet[Lemma~C.4]{dann23best}, whose proof uses the counts and $\delta$ only through two inequalities.
We state their lemma with these inequalities as conditions, writing $N^{<t}(s,a)=\sum_{u<t}\I_u(s,a)$ for the number of visits to $(s,a)$ before episode $t$ and $n_t(s,a)=\max\{1,N^{<t}(s,a)\}$.

\begin{lemma}[{adapted from \citealp[Lemma~C.4]{dann23best}}]\label{lem:dann_occupancy_deviation}
  Let $\iota\geq1$ and $C>0$.
  Suppose that
  \begin{equation}
    |P'(s'\mid s,a)-P(s'\mid s,a)|
    \leq C\min\set*{1,\sqrt{\frac{P(s'\mid s,a)\iota}{n_t(s,a)}}+\frac{\iota}{n_t(s,a)}}
    \label{eq:dann_input_radius}
  \end{equation}
  for every $t\in[T]$, every $(s,a,s')$ with $s'\in\calS_{h(s)+1}$, and every $P'\in\calC_{k_t}$, and that
  \begin{equation}
    \sumT\sum_{s\in\calS_h,a\in\calA}\frac{q_t(s,a)}{n_t(s,a)}\leq CS_hA\iota
    \label{eq:dann_input_counts}
  \end{equation}
  for every $h\in\{0,\ldots,H-1\}$.
  Then, for every choice of kernels $P_t^s\in\calC_{k_t}$, one for each $(t,s)$, and every $g_t\colon\calS\to[0,G]$,
  \begin{equation*}
    \sumT\sum_s|q^{P_t^s,\pi_t}(s)-q_t(s)|\,g_t(s)
    \leq C'\prn*{\sqrt{HS^2A\iota^2\sumT\sum_s q_t(s)g_t(s)^2}+HS^4AG\iota^2},
  \end{equation*}
  where $C'>0$ depends only on $C$.
\end{lemma}
\begin{proof}
  \citet{dann23best} prove this bound with $\iota\log T$ in place of $\iota^2$, since they set $\delta=T^{-3}$.
  Their proof uses only the identity of their Lemma~C.3, their Lemma~B.1, which is~\cref{eq:dann_input_radius}, and their Lemma~B.2, which is~\cref{eq:dann_input_counts} with $S_hA\log T+\log(1/\delta)$ in place of $S_hA\iota$.
  The same proof therefore applies under the conditions above, which completes the proof.
\end{proof}

\begin{proof}[Proof of \cref{lem:transition_occupancy_deviation}]
  We verify the conditions of \cref{lem:dann_occupancy_deviation} with $\iota=\log(TSA/\delta)$.
  Since an epoch ends as soon as a visited $(s,a)$ satisfies $N(s,a)\geq\max\{1,2N_k(s,a)\}$ and $N(s,a)$ changes only when $(s,a)$ is visited, every $(s,a)$ satisfies $N^{<t}(s,a)<\max\{1,2N_{k_t}(s,a)\}$ during epoch $k_t$, and hence
  \begin{equation}
    \max\{1,N_{k_t}(s,a)\}\leq n_t(s,a)\leq2\max\{1,N_{k_t}(s,a)\}\leq4\max\{1,N_{k_t}(s,a)-1\}.
    \label{eq:po_frozen_counts}
  \end{equation}
  Suppose that $P\in\calC_k$ for every epoch $k$, which holds with probability at least $1-4\delta$ by~\cref{eq:global_raw_confidence_probability}, and fix $t$, $(s,a,s')$ with $s'\in\calS_{h(s)+1}$, and $P'\in\calC_{k_t}$.
  If $k_t=1$, then $n_t(s,a)\leq2$ by~\cref{eq:po_frozen_counts}, and~\cref{eq:dann_input_radius} holds for $C\geq2$ since its left-hand side is at most $1\leq2\iota/n_t(s,a)$.
  If $k_t\geq2$, write $\beta=\beta_{k_t}(s,a,s')$, $\bar P=\bar P_{k_t}(s'\mid s,a)$, $p=P(s'\mid s,a)$, and $m=\max\{1,N_{k_t}(s,a)-1\}$.
  Since $P'(s'\mid s,a)$ and $p$ both lie within $\beta$ of $\bar P$,
  \begin{equation*}
    |P'(s'\mid s,a)-p|\leq2\beta,
    \qquad
    \bar P\leq p+\beta=p+2\sqrt{\frac{\bar P\iota}{m}}+\frac{14\iota}{3m}.
  \end{equation*}
  Solving the second inequality for $\sqrt{\bar P}$ and substituting the result into the definition of $\beta$ give
  \begin{equation*}
    \sqrt{\bar P}\leq\sqrt{\frac\iota m}+\sqrt{p+\frac{17\iota}{3m}}\leq\sqrt p+4\sqrt{\frac\iota m},
    \qquad
    \beta\leq2\sqrt{\frac{p\iota}{m}}+\frac{13\iota}{m}.
  \end{equation*}
  Since $m\geq n_t(s,a)/4$ by~\cref{eq:po_frozen_counts} and $|P'(s'\mid s,a)-p|\leq1$, condition~\cref{eq:dann_input_radius} holds for a universal constant $C$.

  By \citet[Lemma~B.2]{dann23best}, which is Lemma~D.3.7 of \citet{jin2021best}, for each layer $h$, with probability at least $1-\delta$,
  \begin{equation*}
    \sumT\sum_{s\in\calS_h,a\in\calA}\frac{q_t(s,a)}{n_t(s,a)}
    \leq C_2\prn[\big]{S_hA\log(2T)+\log(1/\delta)}
    \leq2C_2S_hA\iota
  \end{equation*}
  for a universal constant $C_2>0$, where the second inequality follows from $\log(2T)\leq\iota$ and $\log(1/\delta)\leq\iota$.
  The proof of the first inequality uses only $\E_t\brk{\I_t(s,a)}=q_t(s,a)$ and therefore applies even though the policies depend on the preference feedback.
  By a union bound over the $H$ layers, condition~\cref{eq:dann_input_counts} holds with probability at least $1-H\delta$.

  On the intersection of the two events, which has probability at least $1-(4+H)\delta\geq1-5H\delta$, \cref{lem:dann_occupancy_deviation} with $g_t\equiv1$ gives, for every choice of the kernels $P_t^s\in\calC_{k_t}$,
  \begin{align*}
    \sumT\sum_s|q^{P_t^s,\pi_t}(s)-q_t(s)|
    &\leq C'\prn*{\sqrt{HS^2A\iota^2\sumT\sum_s q_t(s)}+HS^4A\iota^2}\\
    &=C'\prn*{HS\sqrt{AT}\,\iota+HS^4A\iota^2},
  \end{align*}
  where the equality follows from $\sum_s q_t(s)=H$.
  This completes the proof.
\end{proof}

For global optimization, the kernels defining $\widetilde q_t$ and $u_t$ lie in $\calP_{k_t}$, which is contained in $\calC_{k_t}$ whenever $P\in\calC_k$ for every epoch $k$.
\Cref{lem:transition_occupancy_deviation} thus gives the following two bounds.
\begin{lemma}\label{lem:global_transition_confidence}
  Suppose that $T\geq A$.
  For a sufficiently large universal constant $C_{\rm tr}$, define
  \begin{align}
    \calE={}&\{P\in\calC_k\text{ for every epoch }k\}\nn
    &\cap\left\{
      \max\left\{
        \sum_{t,s,a}|\widetilde q_t(s,a)-q_t(s,a)|,
        \sum_{t,s}|u_t(s)-q_t(s)|
      \right\}
      \leq C_{\rm tr}\prn*{HS\sqrt{AT}\,\iota+HS^4A\iota^2}
    \right\}.
    \label{eq:global_transition_event}
\end{align}
  Then $\P(\calE)\geq1-C_0H\delta$ for a universal constant $C_0>0$.
  Moreover, $P\in\calP_k$ for every epoch $k$ on $\calE$.
\end{lemma}
\begin{proof}
  On the event of \cref{lem:transition_occupancy_deviation}, we have $P\in\calC_k$ for every epoch $k$, and hence $P\in\calP_k\subseteq\calC_k$ by~\cref{eq:global_nested_confidence_set}.
  Let $\widetilde P_t=P_{\widetilde x_t}$.
  For any choice of kernels $P_t^s\in\calP_{k_t}$,
  \begin{align*}
    \sum_{t,s,a}|q^{P_t^s,\pi_t}(s,a)-q_t(s,a)|
    &=\sum_{t,s}|q^{P_t^s,\pi_t}(s)-q_t(s)|\nn
    &\leq C_{\rm tr}\prn*{HS\sqrt{AT}\,\iota+HS^4A\iota^2},
\end{align*}
  where the equality follows from $q^{P',\pi_t}(s,a)=\pi_t(a\mid s)q^{P',\pi_t}(s)$ and $\sum_a\pi_t(a\mid s)=1$, and the inequality from~\cref{eq:po_imported_occupancy_deviation} in~\cref{lem:transition_occupancy_deviation}.
  Taking $P_t^s=\widetilde P_t$ for every $s$ gives the bound on the first sum in~\cref{eq:global_transition_event}.

  For the second sum, we choose
  \begin{equation*}
    P_t^s\in\argmax_{P'\in\calP_{k_t}}q^{P',\pi_t}(s).
\end{equation*}
  These maxima exist since $\calP_{k_t}$ is closed and bounded, and
  \begin{equation*}
    |u_t(s)-q_t(s)|=|q^{P_t^s,\pi_t}(s)-q_t(s)|.
\end{equation*}
  Thus, on the same event, the second sum in~\cref{eq:global_transition_event} is at most $C_{\rm tr}(HS\sqrt{AT}\,\iota+HS^4A\iota^2)$ by~\cref{eq:po_imported_occupancy_deviation} in~\cref{lem:transition_occupancy_deviation}.
  Since both bounds are obtained by applying \cref{lem:transition_occupancy_deviation} to the same executed policies and counts, the lemma gives $\P(\calE)\geq1-C_0H\delta$, which completes the proof.
\end{proof}

\subsubsection{Regret bound under unknown transitions (Proof of \cref{thm:reg_occ_unknown})}
\begin{proof}[Proof of \cref{thm:reg_occ_unknown}]
  If $T<A=K^2$, then $HSK\sqrt T\geq HT$, and the trivial bound $\Reg_T\leq HT$ proves the theorem.
  In the rest of the proof, we assume $T\geq A$.
  Define $q_\rho^*=q^{P,\pi_\rho^*}$ as in the known-transition proof.
  For each episode $t$, let
  \begin{equation*}
    \calE_t=\{P\in\calP_{k_t}\}.
\end{equation*}
  These events are determined by the observations before episode $t$ and satisfy $\calE_1\supseteq\calE_2\supseteq\cdots\supseteq\calE_T$, whereas the event $\calE$ in \cref{lem:global_transition_confidence} depends on the entire interaction.
  We have
  \begin{align}
    \E\brk*{\Reg_T}
    &\leq\rho H^2T+\E\brk*{\sumT\inpr{q_t-q_\rho^*,\ell_t}}\nn
    &\leq\rho H^2T
      +\E\brk*{\sumT\ind{\calE_t}\inpr{q_t-q_\rho^*,\ell_t}}
      +HT\P(\calE_T^c),
    \label{eq:global_unknown_stopping}
\end{align}
  where the first inequality follows from \cref{lem:global_comparator_smoothing} and the second inequality from $V^\pi(s_0;\ell_t)\in[0,H]$ and $\calE_t^c\subseteq\calE_T^c$.
  For each episode $t$, we decompose the second term in~\cref{eq:global_unknown_stopping} as
  \begin{align}
    \inpr{q_t-q_\rho^*,\ell_t}
    =\inpr{q_t-\widetilde q_t,\ell_t}
    +\inpr{\widetilde q_t,\ell_t-\hat\ell_t}
    +\inpr{\widetilde q_t-q_\rho^*,\hat\ell_t}
    +\inpr{q_\rho^*,\hat\ell_t-\ell_t}.
    \label{eq:global_unknown_regret_decomposition}
\end{align}
  We bound the four terms separately, each multiplied by $\ind{\calE_t}$ and summed over $t\in[T]$.

  The first term of~\cref{eq:global_unknown_regret_decomposition} satisfies
  \begin{equation}
    \inpr{q_t-\widetilde q_t,\ell_t}
    \leq\sum_{s,a}|q_t(s,a)-\widetilde q_t(s,a)|.
    \label{eq:global_unknown_model_error}
\end{equation}
  For the second term of~\cref{eq:global_unknown_regret_decomposition}, since $\widetilde q_t$ and $\ind{\calE_t}$ are determined before episode $t$ and states with $u_t(s)=0$ contribute zero, \cref{lem:global_unknown_estimator_moments} gives
  \begin{align}
    \E_t\brk*{\ind{\calE_t}\inpr{\widetilde q_t,\ell_t-\hat\ell_t}}
    &=\ind{\calE_t}\sum_{s:u_t(s)>0}\sum_a\widetilde q_t(s,a)
      \prn*{1-\frac{q_t(s)}{u_t(s)}}\ell_t(s,a)\nn
    &\leq\ind{\calE_t}\sum_{s:u_t(s)>0}
      \frac{\widetilde q_t(s)}{u_t(s)}(u_t(s)-q_t(s))\nn
    &\leq\ind{\calE_t}\sum_s(u_t(s)-q_t(s)),
    \label{eq:global_unknown_learner_bias}
\end{align}
  where the first inequality follows from $\ell_t\leq1$ and $q_t(s)\leq u_t(s)$ on $\calE_t$, and the second inequality from $\widetilde q_t(s)\leq u_t(s)$ on $\calE_t$.

  For the third term of~\cref{eq:global_unknown_regret_decomposition}, define $n=\sum_{t=1}^T\ind{\calE_t}$.
  Since the events $\calE_t$ are decreasing in $t$, $\ind{\calE_t}=\ind{t\leq n}$ for every $t\in[T]$ and every realization.
  For $t\leq n$, we have $P\in\calP_{k_t}$, and hence $q_\rho^*\in\Omega_\rho\subseteq\Omega_\rho(\calP_{k_t})$ as shown in \cref{app:global_extended_occupancy}.
  Since $\hat\ell_t\geq0$ as in the known-transition case, \cref{lem:global_nested_ftrl} applies to the first $n$ episodes with $\calQ_t=\Omega_\rho(\calP_{k_t})$ and gives
  \begin{align*}
    \sumT\ind{\calE_t}\inpr{\widetilde q_t-q_\rho^*,\hat\ell_t}
    &=\sum_{t=1}^n\inpr{\widetilde q_t-q_\rho^*,\hat\ell_t}\nn
    &\leq\frac{H\log(SA)}\eta
      +\frac\eta2\sum_{t=1}^n\sum_{s,a}\widetilde q_t(s,a)\hat\ell_t(s,a)^2\nn
    &=\frac{H\log(SA)}\eta
      +\frac\eta2\sumT\ind{\calE_t}\sum_{s,a}\widetilde q_t(s,a)\hat\ell_t(s,a)^2.
\end{align*}
  This inequality holds for every realization of $n$.
  For each fixed $t$, since $\ind{\calE_t}$ is determined before episode $t$, we have
  \begin{align*}
    \E\brk*{\ind{\calE_t}\sum_{s,a}\widetilde q_t(s,a)\hat\ell_t(s,a)^2}
    &=\E\brk*{\ind{\calE_t}\E_t\brk*{\sum_{s,a}\widetilde q_t(s,a)\hat\ell_t(s,a)^2}}\nn
    &\leq\E\brk*{\ind{\calE_t}\frac{SK}{\rho}}\leq\frac{SK}{\rho},
\end{align*}
  where the first inequality follows from~\cref{eq:global_unknown_total_second_moment} in~\cref{lem:global_unknown_estimator_moments}, which holds on $\calE_t$.
  Consequently,
  \begin{equation*}
    \E\brk*{\sumT\ind{\calE_t}\inpr{\widetilde q_t-q_\rho^*,\hat\ell_t}}
    \leq\frac{H\log(SA)}\eta+\frac{\eta SKT}{2\rho}.
\end{equation*}

  For the fourth term of~\cref{eq:global_unknown_regret_decomposition}, since the comparator $q_\rho^*$ is fixed independently of the interaction and $\E_t\brk{\hat\ell_t(s,a)}\leq\ell_t(s,a)$ on $\calE_t$ by~\cref{eq:global_unknown_loss_mean} in~\cref{lem:global_unknown_estimator_moments}, we have
  \begin{align*}
    \E\brk*{\ind{\calE_t}\inpr{q_\rho^*,\hat\ell_t-\ell_t}}
    &=\E\brk*{\ind{\calE_t}\sum_{s,a}q_\rho^*(s,a)
      \prn[\Big]{\E_t\brk[\big]{\hat\ell_t(s,a)}-\ell_t(s,a)}}
    \leq0.
\end{align*}

  We now take expectations of the first two terms.
  By \cref{lem:global_transition_confidence}, on the event $\calE$,
  \begin{align*}
    \sumT\ind{\calE_t}\sum_{s,a}|q_t(s,a)-\widetilde q_t(s,a)|
    &\leq C_{\rm tr}\prn*{HS\sqrt{AT}\,\iota+HS^4A\iota^2},\\
    \sumT\ind{\calE_t}\sum_s(u_t(s)-q_t(s))
    &\leq C_{\rm tr}\prn*{HS\sqrt{AT}\,\iota+HS^4A\iota^2},
  \end{align*}
  whereas the two left-hand sides are at most $2HT$ and $ST$, respectively, on $\calE^c$, since $\sum_{s,a}(q_t(s,a)+\widetilde q_t(s,a))=2H$ and $u_t(s)-q_t(s)\leq1$.
  Since $\P(\calE^c)\leq C_0H\delta$, taking expectations in~\cref{eq:global_unknown_model_error,eq:global_unknown_learner_bias} gives
  \begin{align*}
    \E\brk*{\sumT\ind{\calE_t}\inpr{q_t-\widetilde q_t,\ell_t}}
    &\leq C_{\rm tr}\prn*{HS\sqrt{AT}\,\iota+HS^4A\iota^2}+2C_0H^2T\delta,\\
    \E\brk*{\sumT\ind{\calE_t}\inpr{\widetilde q_t,\ell_t-\hat\ell_t}}
    &\leq C_{\rm tr}\prn*{HS\sqrt{AT}\,\iota+HS^4A\iota^2}+C_0HST\delta.
  \end{align*}
  Substituting these bounds and those on the third and fourth terms into~\cref{eq:global_unknown_stopping}, and using $\P(\calE_T^c)\leq4\delta$ by~\cref{eq:global_raw_confidence_probability}, we obtain
  \begin{align*}
    \E\brk*{\Reg_T}
    &\leq\frac{H\log(SA)}\eta+\frac{\eta SKT}{2\rho}+\rho H^2T\nn
    &\quad+2C_{\rm tr}\prn*{HS\sqrt{AT}\,\iota+HS^4A\iota^2}
      +C_1H(3H+S)T\delta,
\end{align*}
  where $C_1=\max\{C_0,4\}$.
  We choose $\delta=(HSA(T+1))^{-2}$ and take $\eta$ and $\rho$ as in the proof of \cref{thm:reg_occ_meas}.
  Then the last term is bounded by a universal constant.
  As shown in the proof of \cref{thm:reg_occ_meas}, if $\sqrt{\eta SK}/H\leq1/2$, then the sum of the first three terms is at most $\frac52(H^3SK\log(SA))^{1/3}T^{2/3}$, and otherwise $HT\leq2(H^3SK\log(SA))^{1/3}T^{2/3}$, so that the trivial bound $\Reg_T\leq HT$ is of the claimed order.

  Substituting $A=K^2$ gives $HS\sqrt{AT}\,\iota=HSK\sqrt T\,\iota$ and $HS^4A\iota^2=HS^4K^2\iota^2$, and combining these bounds with the trivial bound $\Reg_T\leq HT$ proves~\cref{eq:global_unknown_bound}.
  This completes the proof.
\end{proof}

\section{Deferred proofs for the policy optimization approach}\label{app:proof_policy_opt_23}
We first prove the additional moment bounds needed for policy optimization and then derive the regret bound under known transitions, using~\cref{lem:bhat_unbiased,lem:global_comparator_smoothing}.
The unknown-transition algorithm in~\cref{app:po_unknown} uses the same regret decomposition.

Throughout, $\hat\ell_t$ denotes~\cref{eq:def_l_hat}, not~\cref{eq:global_borda_estimator}.
Recall that we use $\E_t$ and $\P_t(E)=\E_t\brk{\ind{E}}$ as in \cref{app:proof_global_opt}, including the learner's past random choices in the conditioning.

\subsection{Loss estimates}
The sampling rule in \crefrange{line:po_coin}{line:po_observe} of \cref{alg:pf_mdp_policy_optimization} does not require knowledge of $P$, and hence the conditional moment bounds below apply in both settings.
Write $q_t=q^{P,\pi_t}$, $Q_t(s,a)=Q^{\pi_t}(s,a;\ell_t)$, and $V_t(s)=V^{\pi_t}(s;\ell_t)$.
The true occupancy measure $q_t$ is used only for analysis when $P$ is unknown.
For the analysis, define
\begin{equation}
    \hat g_t(s)=\sum_{a\in\calA}\pi_t(a\mid s)\hat\ell_t(s,a),
 \label{eq:po_average_loss_definition}
\end{equation}
so that $\hat L_{t,h}=\sum_{j=h}^{H-1}\hat g_t(s_{t,j})$ in~\cref{eq:def_Q_hat}.

\begin{lemma}\label{lem:po_local_moments}
For every state $s$ with positive visit probability and every $a\in\calA$,
\begin{align}
    \E_t\brk*{\hat\ell_t(s,a)\relmiddle|\I_t(s)=1}&=\ell_t(s,a),
 &\E_t\brk*{\hat\ell_t(s,a)^2\relmiddle|\I_t(s)=1}&\leq K/\gamma,
 \label{eq:ell_hat_state_unbiased}\\
 \E_t\brk*{\hat g_t(s)\relmiddle|\I_t(s)=1}&=\sum_a\pi_t(a\mid s)\ell_t(s,a),
 &\E_t\brk*{\hat g_t(s)^2\relmiddle|\I_t(s)=1}&\leq K/\gamma.
 \label{eq:g_hat_state_unbiased}
\end{align}
\end{lemma}
\begin{proof}
The sampling rule in~\cref{line:po_coin} gives $\P_t(Z_t(s)=1\mid\I_t(s)=1)=\gamma$.
Using the conditional moments of $\hat b_t$ from~\cref{lem:bhat_unbiased} in~\cref{eq:def_l_hat}, we obtain
\begin{align}
    \E_t\brk*{\hat\ell_t(s,a)\relmiddle|\I_t(s)=1}
 &=-\frac{\gamma}{2\gamma}
   \E_t\brk*{\hat b_t(s,a^{\rm L})+\hat b_t(s,a^{\rm R})
               \relmiddle|\I_t(s)=1,Z_t(s)=1}\nonumber\\
 &=-\frac12\prn[\big]{b_t(s,a^{\rm L})+b_t(s,a^{\rm R})}=\ell_t(s,a),
 \label{eq:po_pair_mean_proof}\\
 \E_t\brk*{\hat\ell_t(s,a)^2\relmiddle|\I_t(s)=1}
 &\leq\frac1{2\gamma}
   \E_t\brk*{\hat b_t(s,a^{\rm L})^2+\hat b_t(s,a^{\rm R})^2
                   \relmiddle|\I_t(s)=1,Z_t(s)=1}
 \leq\frac K\gamma,
 \label{eq:po_pair_moments_proof}
\end{align}
where the first inequality follows from $(x+y)^2\leq2x^2+2y^2$ and the last inequality from the second-moment bound in~\cref{lem:bhat_unbiased}.
For the average in~\cref{eq:po_average_loss_definition},
\begin{align*}
    \E_t\brk*{\hat g_t(s)\relmiddle|\I_t(s)=1}
 &=\sum_a\pi_t(a\mid s)\E_t\brk*{\hat\ell_t(s,a)\relmiddle|\I_t(s)=1}
 =\sum_a\pi_t(a\mid s)\ell_t(s,a),
 \\
 \E_t\brk*{\hat g_t(s)^2\relmiddle|\I_t(s)=1}
 &\leq\sum_a\pi_t(a\mid s)
           \E_t\brk*{\hat\ell_t(s,a)^2\relmiddle|\I_t(s)=1}
 \leq K/\gamma,
\end{align*}
where the last equality follows from~\cref{eq:po_pair_mean_proof}, the first inequality from Jensen's inequality, and the last inequality from~\cref{eq:po_pair_moments_proof}, which completes the proof.
\end{proof}

\subsection{Estimated losses at later steps}
The variable $\hat L_{t,h}$ in~\cref{eq:def_Q_hat} sums the estimated losses from step $h$ to the end of episode $t$.
The next lemma computes its conditional expectation and bounds its second moment, separating the diagonal and cross terms.
\begin{samepage}
\begin{lemma}\label{lem:po_suffix_moments}
For $s\in\calS_h$ with $0\leq h<H$ and $a\in\calA$, let $n=H-h-1$.
Whenever the conditioning events have positive probability,
\begin{align}
    \E_t\brk*{\hat L_{t,h}\relmiddle|\I_t(s)=1}&=V_t(s),
 \nonumber\\
 \E_t\brk*{\hat L_{t,h+1}\relmiddle|\I_t(s,a)=1}&=Q_t(s,a)-\ell_t(s,a),
 \label{eq:Lcirc_future_exp}\\
 \E_t\brk*{\hat L_{t,h+1}^2\relmiddle|\I_t(s,a)=1}&\leq nK/\gamma+n(n-1).
 \label{eq:Lcirc_second_future}
\end{align}
\end{lemma}
\end{samepage}
\begin{proof}
Let $\calF_{t,0}$ be the history before episode $t$, so that $\E_t\brk{\cdot}=\E\brk{\cdot\mid\calF_{t,0}}$.
For $0\leq j\leq H$, let $\calF_{t,j}$ additionally contain the variables $Z_t(s_{t,k})$, actions, and observations at steps $0,\ldots,j-1$ and the states $s_{t,0},\ldots,s_{t,j}$.
In particular, $\sigma(s_{t,j})\subseteq\calF_{t,j}$, whereas the variable $Z_t(s_{t,j})$, action, and feedback at step $j$ are not included.
We first prove $\E\brk{\hat L_{t,j}\mid\calF_{t,j}}=V_t(s_{t,j})$ by backward induction.
Both sides are zero at $j=H$.
If the identity holds at $j+1$, then
\begin{align}
    \E\brk*{\hat L_{t,j}\relmiddle|\calF_{t,j}}
 &=\E\brk*{\hat g_t(s_{t,j})\relmiddle|\calF_{t,j}}
    +\E\brk[\Big]{\E\brk[\big]{\hat L_{t,j+1}\bigm|\calF_{t,j+1}}\Bigm|\calF_{t,j}}\nonumber\\
 &=\sum_a\pi_t(a\mid s_{t,j})\ell_t(s_{t,j},a)
    +\sum_{a,s'}\pi_t(a\mid s_{t,j})P(s'\mid s_{t,j},a)V_t(s')\nonumber\\
 &=V_t(s_{t,j}),
 \label{eq:po_suffix_mean_induction}
\end{align}
where the first equality follows from~\cref{eq:def_Q_hat}, the second equality from the sampling rule in~\cref{line:po_coin}, the induction hypothesis, and $s_{t,j+1}\sim P(\cdot\mid s_{t,j},a_{t,j})$, and the last equality from the definition of $V_t$.
For a fixed $s\in\calS_h$, the event $\{\I_t(s)=1\}=\{s_{t,h}=s\}$ belongs to $\calF_{t,h}$.
Therefore,
\begin{align*}
    \E_t\brk*{\hat L_{t,h}\relmiddle|\I_t(s)=1}
 &=\E_t\brk[\Big]{\E\brk[\big]{\hat L_{t,h}\bigm|\calF_{t,h}}\Bigm|\I_t(s)=1}\\
 &=\E_t\brk*{V_t(s_{t,h})\relmiddle|\I_t(s)=1}
 =V_t(s),
\end{align*}
where the first equality follows from $\{\I_t(s)=1\}\in\calF_{t,h}$, the second equality from~\cref{eq:po_suffix_mean_induction}, and the last equality from $s_{t,h}=s$ on that event and the fact that $V_t(s)$ is determined before episode $t$.
For a fixed $(s,a)$, the visit event instead belongs to $\calF_{t,h+1}$, and
\begin{align*}
    \E_t\brk*{\hat L_{t,h+1}\relmiddle|\I_t(s,a)=1}
 &=\E_t\brk[\Big]{\E\brk[\big]{\hat L_{t,h+1}\bigm|\calF_{t,h+1}}\Bigm|\I_t(s,a)=1}\\
 &=\sum_{s'\in\calS_{h+1}}P(s'\mid s,a)V_t(s')
 =Q_t(s,a)-\ell_t(s,a),
\end{align*}
where the second equality follows from~\cref{eq:po_suffix_mean_induction} at step $h+1$ and $s_{t,h+1}\sim P(\cdot\mid s,a)$ given $\I_t(s,a)=1$, and the last equality from the definition of $Q_t$.

For the second moment, let $X_j=\hat g_t(s_{t,j})$.
The calculation of~\cref{eq:g_hat_state_unbiased} in~\cref{lem:po_local_moments}, conditional on the history $\calF_{t,j}$ before sampling at step $j$, gives
\begin{equation}
    0\leq\E\brk*{X_j\relmiddle|\calF_{t,j}}\leq1,
 \qquad \E\brk*{X_j^2\relmiddle|\calF_{t,j}}\leq K/\gamma.
 \label{eq:po_sequential_moments}
\end{equation}
Let $E=\{\I_t(s,a)=1\}$.
For $h<k<j<H$, both $E$ and $X_k$ are determined by $\calF_{t,j}$.
Since $X_k\geq0$, we have
\begin{align}
    \E_t\brk*{X_kX_j\relmiddle| E}
 &=\E_t\brk*{X_k\E\brk*{X_j\relmiddle|\calF_{t,j}}\relmiddle| E}
 \leq\E_t\brk*{X_k\relmiddle| E}\nonumber\\
 &=\E_t\brk*{\E\brk*{X_k\relmiddle|\calF_{t,k}}\relmiddle| E}\leq1,
 \qquad \E_t\brk*{X_k^2\relmiddle| E}\leq K/\gamma,
 \label{eq:po_cross_moment}
\end{align}
where the equalities follow from $E\in\calF_{t,k}\subseteq\calF_{t,j}$ and the $\calF_{t,j}$-measurability of $X_k$, and the inequalities from~\cref{eq:po_sequential_moments} and $X_k\geq0$.
Expanding the square gives
\begin{align*}
 \E_t\brk*{\hat L_{t,h+1}^2\relmiddle| E}
 =\sum_{k=h+1}^{H-1}\E_t\brk*{X_k^2\relmiddle| E}
   +2\sum_{h<k<j<H}\E_t\brk*{X_kX_j\relmiddle| E}
 \leq nK/\gamma+n(n-1),
\end{align*}
where the inequality follows from~\cref{eq:po_cross_moment} applied to the $n$ diagonal terms and the $n(n-1)/2$ pairs of distinct steps.
For $h=H-1$, $\hat L_{t,H}=0$ and the bound is zero as well, which completes the proof.
\end{proof}

\subsection{Q-function estimate under known transitions}
We now take into account the denominator $q_t(s,a)+\delta$ in~\cref{eq:def_Q_hat}.
The first term in the numerator of~\cref{eq:def_Q_hat} estimates the immediate loss, and the second term estimates the losses after taking $a$ at $s$.
\begin{lemma}\label{lem:Qhat_est_exp}
For every nonterminal $(s,a)$, the estimate in~\cref{eq:def_Q_hat} satisfies
\begin{align}
    \E_t\brk*{\hat Q_t(s,a)}
 &=\frac{q_t(s,a)}{q_t(s,a)+\delta}Q_t(s,a),
 \label{eq:po_known_Q_mean}\\
 \E_t\brk*{\hat Q_t(s,a)^2}
 &\leq\frac{2H(K/\gamma+H)}{q_t(s,a)+\delta}.
 \label{eq:po_known_Q_second}
\end{align}
\end{lemma}
\begin{proof}
If $q_t(s)=0$, then $\hat Q_t(s,a)=0$ almost surely, and the claims follow.
Otherwise, $q_t(s,a)>0$ since $\pi_t(a\mid s)\geq\gamma/A$.
Taking conditional expectations of the two terms in~\cref{eq:def_Q_hat} gives
\begin{align*}
    \E_t\brk*{\hat Q_t(s,a)}
 &=\frac{q_t(s)\pi_t(a\mid s)}{q_t(s,a)+\delta}
           \E_t\brk*{\hat\ell_t(s,a)\relmiddle|\I_t(s)=1}\\
 &\qquad+\frac{q_t(s,a)}{q_t(s,a)+\delta}
           \E_t\brk*{\hat L_{t,h(s)+1}\relmiddle|\I_t(s,a)=1}\\
 &=\frac{q_t(s,a)}{q_t(s,a)+\delta}
           \prn[\big]{\ell_t(s,a)+Q_t(s,a)-\ell_t(s,a)}\\
 &=\frac{q_t(s,a)}{q_t(s,a)+\delta}Q_t(s,a),
\end{align*}
where the second equality follows from~\cref{eq:ell_hat_state_unbiased} in~\cref{lem:po_local_moments}, \cref{eq:Lcirc_future_exp} in~\cref{lem:po_suffix_moments}, and $q_t(s)\pi_t(a\mid s)=q_t(s,a)$.

We next prove~\cref{eq:po_known_Q_second} by bounding the two terms in $\hat Q_t(s,a)$ separately.
We have
\begin{equation}
    \E_t\brk*{\hat Q_t(s,a)^2}
 \leq
 2\E_t\brk*{
 \prn*{\frac{\I_t(s)\pi_t(a\mid s)\hat\ell_t(s,a)}{q_t(s,a)+\delta}}^2}
 +
 2\E_t\brk*{
 \prn*{\frac{\I_t(s,a)\hat L_{t,h(s)+1}}{q_t(s,a)+\delta}}^2}
 ,
 \label{eq:po_known_Q_second_combination_pre}
\end{equation}
where the inequality follows from $(x+y)^2\leq2x^2+2y^2$.
We bound the first expectation on the right-hand side of~\cref{eq:po_known_Q_second_combination_pre} as
\begin{align}
    \E_t\brk*{
 \prn*{\frac{\I_t(s)\pi_t(a\mid s)\hat\ell_t(s,a)}{q_t(s,a)+\delta}}^2}
 &=\frac{q_t(s)\pi_t(a\mid s)^2}{(q_t(s,a)+\delta)^2}
          \E_t\brk*{\hat\ell_t(s,a)^2\relmiddle|\I_t(s)=1}\nonumber\\
 &\leq\frac{K}{\gamma(q_t(s,a)+\delta)},
 \label{eq:po_known_immediate_second}
\end{align}
where the inequality follows from~\cref{eq:ell_hat_state_unbiased} in~\cref{lem:po_local_moments} and $q_t(s)\pi_t(a\mid s)^2\leq q_t(s,a)\leq q_t(s,a)+\delta$.
For the second expectation on the right-hand side of~\cref{eq:po_known_Q_second_combination_pre}, let $n=H-h(s)-1$.
Using~\cref{eq:Lcirc_second_future} in~\cref{lem:po_suffix_moments}, we obtain
\begin{align}
    \E_t\brk*{
 \prn*{\frac{\I_t(s,a)\hat L_{t,h(s)+1}}{q_t(s,a)+\delta}}^2}
 &=\frac{q_t(s,a)}{(q_t(s,a)+\delta)^2}
           \E_t\brk*{\hat L_{t,h(s)+1}^2\relmiddle|\I_t(s,a)=1}\nonumber\\
 &\leq\frac{nK/\gamma+n(n-1)}{q_t(s,a)+\delta}.
 \label{eq:po_known_future_second}
\end{align}
Substituting~\cref{eq:po_known_immediate_second,eq:po_known_future_second} into the right-hand side of~\cref{eq:po_known_Q_second_combination_pre} gives
\begin{equation*}
    \E_t\brk*{\hat Q_t(s,a)^2}
 \leq\frac{2((n+1)K/\gamma+n(n-1))}{q_t(s,a)+\delta}
 \leq\frac{2H(K/\gamma+H)}{q_t(s,a)+\delta},
\end{equation*}
where the last inequality follows from $n+1\leq H$, which completes the proof.
\end{proof}

\subsection{Regret decomposition}
We first decompose the expected regret into a sum controlled by FTRL and a sum containing the bonuses.
Adding the same uniform component to the comparator and to $\pi_t$ allows us to bound the remaining terms by the expected bonus at the initial state.

For a best fixed policy $\pi^*$, define
\begin{equation}
    \pi_\gamma^*(a\mid s)=(1-\gamma)\pi^*(a\mid s)+\gamma/A,
 \qquad q_\gamma^*=q^{P,\pi_\gamma^*}.
 \label{eq:po_smoothed_comparator}
\end{equation}
These quantities are fixed independently of the interaction under the oblivious-adversary assumption, and \cref{lem:global_comparator_smoothing} with $\rho=\gamma$ gives
\begin{equation}
    V^{\pi_\gamma^*}(s_0;\ell_t)-V^{\pi^*}(s_0;\ell_t)\leq\gamma H^2.
 \label{eq:po_comparator_cost}
\end{equation}
Subtracting~\cref{eq:po_smoothed_comparator} from~\cref{eq:po_executed_policy} yields
\begin{equation}
    \pi_t(\cdot\mid s)-\pi_\gamma^*(\cdot\mid s)
 =(1-\gamma)\prn[\big]{\tilde\pi_t(\cdot\mid s)-\pi^*(\cdot\mid s)}.
 \label{eq:po_comparator_difference}
\end{equation}
The regret is still measured against $\pi^*$.
The following lemma combines~\cref{eq:po_comparator_difference} with the bonus argument of \citet{luo21policy}.
In its proof, the states are weighted by $q_\gamma^*(s)$, while the comparator in the FTRL bound at $s$ is $\pi^*(\cdot\mid s)$.

\begin{lemma}\label{lem:po_common_regret}
Fix $\eta>0$ and $\gamma\in(0,1/2]$.
Suppose that the policy update is~\cref{eq:po_ftrl_update} with $\tilde\pi_1$ uniform, that the executed policy is $\pi_t(a\mid s)=(1-\gamma)\tilde\pi_t(a\mid s)+\gamma/A$, that $\hat Q_t\geq0$, and that $m_t$ and $M_t$ are determined before episode $t$ and satisfy $m_t\geq0$ and $0\leq M_t\leq9H^2$.
Let $\tau\in\{1,\ldots,T+1\}$ be such that $\{t<\tau\}$ is determined before episode $t$.
On $\{t<\tau\}$, suppose that, for every nonterminal $(s,a)$,
\begin{align}
    \E_t\brk*{\hat Q_t(s,a)}&\leq Q_t(s,a),
 \label{eq:po_optimism_condition}\\
 \sum_a\tilde\pi_t(a\mid s)
  \prn[\Big]{Q_t(s,a)-\E_t\brk[\big]{\hat Q_t(s,a)}+
                    \eta\E_t\brk[\big]{\hat Q_t(s,a)^2}}&\leq m_t(s).
 \label{eq:po_bias_variance_condition}
\end{align}
Define $\bar M_t(s)=\sum_a\pi_t(a\mid s)M_t(s,a)$, with $\bar M_t(s_H)=0$, and suppose also that, on $\{t<\tau\}$,
\begin{equation}
    M_t(s,a)\geq m_t(s)+\sum_{s'}P(s'\mid s,a)\bar M_t(s').
 \label{eq:po_bonus_supersolution}
\end{equation}
If $\eta\leq1/(9H^2)$, then
\begin{equation}
    \E\brk*{\Reg_T}\leq\gamma H^2T+\frac{H\log A}{\eta}
       +81\eta H^5T
       +\E\brk*{\sum_{t<\tau}\bar M_t(s_0)}
       +HT\P(\tau\leq T).
 \label{eq:po_common_regret_bound}
\end{equation}
\end{lemma}
Condition~\cref{eq:po_bonus_supersolution} holds with equality for the bonus in~\cref{eq:def_bonus_M_po}, and it follows from $P\in\calC_{k_t}$ for the bonus in~\cref{eq:def_M_unk}.
In the applications below, $\tau=T+1$ for known transitions, whereas for unknown transitions, \cref{eq:po_first_confidence_failure} defines $\tau$ as the first episode in which the confidence set fails to contain $P$.

\begin{proof}
Splitting the regret at $\tau$ and inserting $\pi_\gamma^*$ gives
\begin{align*}
    \E\brk*{\Reg_T}
 &=\E\brk*{\sum_{t<\tau}
       \prn[\big]{V^{\pi_t}(s_0;\ell_t)-V^{\pi^*}(s_0;\ell_t)}}
   +\E\brk*{\sum_{\tau\leq t\leq T}
       \prn[\big]{V^{\pi_t}(s_0;\ell_t)-V^{\pi^*}(s_0;\ell_t)}}\\
 &\leq\gamma H^2T+HT\P(\tau\leq T)
   +\E\brk*{\sum_{t<\tau}
       \prn[\big]{V^{\pi_t}(s_0;\ell_t)-V^{\pi_\gamma^*}(s_0;\ell_t)}},
\end{align*}
where the inequality follows from~\cref{eq:po_comparator_cost} and $0\leq V^\pi(s_0;\ell_t)\leq H$.
We next decompose the expectation in the last bound, which gives
\begin{align}
    \E\brk*{\Reg_T}
 \leq{}&\gamma H^2T+HT\P(\tau\leq T)\nonumber\\
 &+(1-\gamma)\sum_s q_\gamma^*(s)\E\brk*{\sum_{t<\tau}
   \inpr{\tilde\pi_t(\cdot\mid s)-\pi^*(\cdot\mid s),Q_t(s,\cdot)-M_t(s,\cdot)}}\nonumber\\
 &+\E\brk*{\sum_{t<\tau}\sum_s q_\gamma^*(s)
   \inpr{\pi_t(\cdot\mid s)-\pi_\gamma^*(\cdot\mid s),M_t(s,\cdot)}},
 \label{eq:po_regret_comparison_identity}
\end{align}
where the inequality follows from the performance difference lemma and~\cref{eq:po_comparator_difference}.
We first bound the term containing $Q_t(s,\cdot)-M_t(s,\cdot)$ on the right-hand side of~\cref{eq:po_regret_comparison_identity}.
For each state $s$, the FTRL update in~\cref{eq:po_ftrl_update} has losses $\hat Q_t(s,\cdot)-M_t(s,\cdot)$ satisfying
$\eta(\hat Q_t(s,a)-M_t(s,a))\geq-9\eta H^2\geq-1$.
By the standard regret bound for FTRL with the negative Shannon entropy regularizer~\citep[Chapter~28]{lattimore2020book}, for every $n\in\{0,\ldots,T\}$,
\begin{align}
    \sum_{t=1}^n\inpr{\tilde\pi_t(\cdot\mid s)-\pi^*(\cdot\mid s),\hat Q_t(s,\cdot)-M_t(s,\cdot)}
    \leq\frac{\log A}{\eta}
       +\eta\sum_{t=1}^n\sum_a\tilde\pi_t(a\mid s)
                          (\hat Q_t(s,a)-M_t(s,a))^2.
 \label{eq:po_ftrl_initial_segment}
\end{align}
Subtracting and adding $\hat Q_t(s,\cdot)$ gives
\begin{align}
    &\inpr{\tilde\pi_t(\cdot\mid s)-\pi^*(\cdot\mid s),Q_t(s,\cdot)-M_t(s,\cdot)}\nonumber\\
 &=\inpr{\tilde\pi_t(\cdot\mid s)-\pi^*(\cdot\mid s),\hat Q_t(s,\cdot)-M_t(s,\cdot)}\nonumber\\
 &\qquad+\inpr{\tilde\pi_t(\cdot\mid s),Q_t(s,\cdot)-\hat Q_t(s,\cdot)}
    +\inpr{\pi^*(\cdot\mid s),\hat Q_t(s,\cdot)-Q_t(s,\cdot)}.
 \label{eq:po_local_bias_decomposition}
\end{align}
For the third term in~\cref{eq:po_local_bias_decomposition},
\begin{align}
    &\E\brk*{\ind{t<\tau}\sum_a\pi^*(a\mid s)(\hat Q_t(s,a)-Q_t(s,a))}\nonumber\\
 &=\E\brk*{\ind{t<\tau}\sum_a\pi^*(a\mid s)
                   (\E_t\brk*{\hat Q_t(s,a)}-Q_t(s,a))}\leq0,
 \label{eq:po_comparator_estimation_error}
\end{align}
where the equality follows from the fact that $\ind{t<\tau}$ and $\pi^*$ are determined before episode $t$, and the inequality from~\cref{eq:po_optimism_condition} in the statement of this lemma.
Taking $n=\tau-1$ in~\cref{eq:po_ftrl_initial_segment} and substituting~\cref{eq:po_local_bias_decomposition,eq:po_comparator_estimation_error} gives
\begin{align*}
    &\E\brk*{\sum_{t<\tau}
       \inpr{\tilde\pi_t(\cdot\mid s)-\pi^*(\cdot\mid s),Q_t(s,\cdot)-M_t(s,\cdot)}}\\
 &\leq\frac{\log A}{\eta}
  +\E\brk*{\sum_{t<\tau}\sum_a\tilde\pi_t(a\mid s)
       \prn[\Big]{Q_t(s,a)-\E_t\brk[\big]{\hat Q_t(s,a)}
                         +\eta\E_t\brk[\big]{\hat Q_t(s,a)^2}}}
  +81\eta H^4T\\
 &\leq\frac{\log A}{\eta}
       +\E\brk*{\sum_{t<\tau}m_t(s)}+81\eta H^4T,
\end{align*}
where the first inequality follows from $(\hat Q_t(s,a)-M_t(s,a))^2\leq\hat Q_t(s,a)^2+M_t(s,a)^2$ and $M_t(s,a)^2\leq81H^4$, and the second inequality from~\cref{eq:po_bias_variance_condition} in the statement.
Substituting this bound into~\cref{eq:po_regret_comparison_identity} yields
\begin{align}
    \E\brk*{\Reg_T}\leq{}&\gamma H^2T+HT\P(\tau\leq T)
       +\frac{H\log A}{\eta}+81\eta H^5T\nonumber\\
 &\quad+\E\brk*{\sum_{t<\tau}\sum_s q_\gamma^*(s)
   \left((1-\gamma)m_t(s)+\bar M_t(s)
                       -\sum_a\pi_\gamma^*(a\mid s)M_t(s,a)\right)},
 \label{eq:po_regret_before_bonus_sum}
\end{align}
where we used $(1-\gamma)\sum_s q_\gamma^*(s)\leq H$ and the definition of $\bar M_t$.

It remains to bound the last sum in~\cref{eq:po_regret_before_bonus_sum}.
For every $s'\in\calS\setminus\{s_0,s_H\}$, since $q_\gamma^*$ is the occupancy measure of $\pi_\gamma^*$ under $P$, we have
\begin{equation}
    \sum_{s,a}q_\gamma^*(s)\pi_\gamma^*(a\mid s)P(s'\mid s,a)
 =q_\gamma^*(s').
 \label{eq:po_comparator_flow}
\end{equation}
For every $t<\tau$, we have
\begin{align*}
    &\sum_s q_\gamma^*(s)
  \left((1-\gamma)m_t(s)+\bar M_t(s)
                     -\sum_a\pi_\gamma^*(a\mid s)M_t(s,a)\right)\\
 &\leq\sum_s q_\gamma^*(s)
  \left(\bar M_t(s)-\sum_{a,s'}\pi_\gamma^*(a\mid s)P(s'\mid s,a)\bar M_t(s')
                    -\gamma m_t(s)\right)\\
 &=\bar M_t(s_0)-\gamma\sum_s q_\gamma^*(s)m_t(s)
 \leq\bar M_t(s_0),
\end{align*}
where the first inequality follows from~\cref{eq:po_bonus_supersolution} in the statement, the equality from~\cref{eq:po_comparator_flow}, $q_\gamma^*(s_0)=1$, and $\bar M_t(s_H)=0$, and the last inequality from $m_t\geq0$.
Bounding the last sum in~\cref{eq:po_regret_before_bonus_sum} by $\bar M_t(s_0)$ gives~\cref{eq:po_common_regret_bound}, which completes the proof.
\end{proof}

\subsection[Regret bound under known transitions and parameter choices]{Regret bound under known transitions and parameter choices (Proof of \cref{thm:regret_po_23})}
We now prove~\cref{thm:regret_po_23} and specify the parameters used in both settings.
\begin{proof}[Proof of \cref{thm:regret_po_23}]
We first verify the conditions of~\cref{lem:po_common_regret}, with $\tau=T+1$.
We have $\hat Q_t\geq0$ by the definitions in~\cref{eq:def_l_hat,eq:def_Q_hat}, and~\cref{eq:po_known_Q_mean} in~\cref{lem:Qhat_est_exp} gives $\E_t\brk{\hat Q_t(s,a)}\leq Q_t(s,a)$, which verifies~\cref{eq:po_optimism_condition}.
For~\cref{eq:po_bias_variance_condition}, we have
\begin{align*}
    &\sum_a\tilde\pi_t(a\mid s)
      \prn[\Big]{Q_t(s,a)-\E_t\brk[\big]{\hat Q_t(s,a)}
                         +\eta\E_t\brk[\big]{\hat Q_t(s,a)^2}}\\
 &\leq H\prn[\big]{\delta+2\eta(K/\gamma+H)}
             \sum_a\frac{\tilde\pi_t(a\mid s)}{q_t(s,a)+\delta}
       =m_t(s),
\end{align*}
where the inequality follows from~\cref{eq:po_known_Q_mean,eq:po_known_Q_second} in~\cref{lem:Qhat_est_exp} and $Q_t(s,a)\leq H$, and the equality from $\delta=2\eta(K/\gamma+H)$.
For the remaining conditions in~\cref{lem:po_common_regret}, \cref{eq:def_bonus_M_po} gives $0\leq m_t(s)\leq2H$ and $0\leq M_t(s,a)\leq2H^2$, and~\cref{eq:po_bonus_supersolution} holds with equality.
The lemma therefore applies whenever $\eta\leq1/(9H^2)$.
We next bound $\bar M_t(s_0)=V^{\pi_t}(s_0;m_t)$ in~\cref{eq:po_common_regret_bound}, which satisfies
\begin{align*}
    \bar M_t(s_0)&=2\delta H\sum_{s,a}
                 \frac{q_t(s)\tilde\pi_t(a\mid s)}{q_t(s,a)+\delta}\\
 &\leq4\delta H\sum_{s,a}\frac{q_t(s,a)}{q_t(s,a)+\delta}
 \leq4\delta HSA,
\end{align*}
where the first inequality follows from $\tilde\pi_t(a\mid s)\leq\pi_t(a\mid s)/(1-\gamma)\leq2\pi_t(a\mid s)$ and $q_t(s)\pi_t(a\mid s)=q_t(s,a)$, and the second inequality from $q_t(s,a)/(q_t(s,a)+\delta)\leq1$.
Substituting this bound into~\cref{eq:po_common_regret_bound} of~\cref{lem:po_common_regret} gives
\begin{align}
    \E\brk*{\Reg_T}
 &\leq\gamma H^2T+\frac{H\log A}{\eta}
                   +81\eta H^5T+4\delta HSAT\nonumber\\
 &\leq\gamma H^2T+\frac{H\log A}{\eta}
       +\eta T\prn*{81H^5+24H^2SK^2+\frac{24HSK^3}{\gamma}},
 \label{eq:po_master_known}
\end{align}
where, in the second inequality, we substituted $\delta=2\eta(K/\gamma+H)$ and $A=K^2$ and increased the resulting coefficients from $8$ to $24$ so that the same parameter choice can be used in both settings.

For both algorithms, we choose
\begin{align}
    \gamma&=\min\left\{\frac12,
       \prn*{\frac{24SK^3\log A}{H^2T}}^{1/3}\right\},\nonumber\\
 \eta&=\min\left\{\frac1{9H^2},
       \sqrt{\frac{H\log A}
          {T(81H^5+24H^2SK^2+24HSK^3/\gamma)}}\right\},
 \qquad \delta=2\eta(K/\gamma+H).
 \label{eq:po_parameters}
\end{align}
To substitute these parameters into~\cref{eq:po_master_known}, we write
$D=H\log A$, $B_0=81H^5+24H^2SK^2$, and $B_1=24HSK^3$.
If $\gamma<1/2$ and $\eta<1/(9H^2)$ in~\cref{eq:po_parameters}, the right-hand side of~\cref{eq:po_master_known} is at most
\begin{align}
    \gamma H^2T+2\sqrt{DT(B_0+B_1/\gamma)}
 &\leq\gamma H^2T+2\sqrt{DB_1T/\gamma}+2\sqrt{DB_0T}\nonumber\\
 &=3(D B_1 H^2)^{1/3}T^{2/3}+2\sqrt{DB_0T},
 \label{eq:po_tuned_bound}
\end{align}
where the inequality follows from $\sqrt{x+y}\leq\sqrt{x}+\sqrt{y}$ for $x,y\geq0$, and the equality from substituting $\gamma=(24SK^3\log A/(H^2T))^{1/3}$.
Moreover,
\begin{equation}
    DB_1H^2=24H^4SK^3\log A,
 \qquad DB_0=(81H^6+24H^3SK^2)\log A.
 \label{eq:po_parameter_powers}
\end{equation}
If $\gamma=1/2$, then
$(D B_1 H^2)^{1/3}T^{2/3}\geq H^2T/2\geq HT/2$.
If $\eta=1/(9H^2)$, then
$D\geq T(B_0+B_1/\gamma)/(81H^4)\geq HT$ and hence
$\sqrt{DB_0T}\geq9H^3T\geq HT$.
In either case the trivial bound $\Reg_T\leq HT$ is within a universal factor of the claimed minimum.
Substituting~\cref{eq:po_parameter_powers} into~\cref{eq:po_tuned_bound} and using $\log A=2\log K$ gives the bound in \cref{thm:regret_po_23}, which completes the proof.
\end{proof}

\section{Policy optimization under unknown transitions}\label{app:po_unknown}
This section presents the policy optimization algorithm for unknown transitions (\cref{alg:pf_mdp_policy_optimization_unk}) and its regret bound.
We retain $\hat\ell_t$ in~\cref{eq:def_l_hat}, $\hat L_{t,h}$ in~\cref{eq:def_Q_hat}, and the policy update in~\cref{eq:po_ftrl_update}.
The only changes are the visit probabilities used in $\hat Q_t$ and the bonus computed before each episode.

\begin{theorem}\label{thm:regret_po_23_unk}
  Suppose that the transition kernel is fixed but unknown.
  With the parameter choices specified in the proof, \cref{alg:pf_mdp_policy_optimization_unk} satisfies, for every $T\geq1$,
  \begin{equation}
    \E\brk*{\Reg_T}=\tilde{O}\prn*{\min\set*{HT,\,
      (H^4SK^3)^{1/3}T^{2/3}+H^2SK\sqrt T+H^2S^4K^2}}.
    \label{eq:po_unknown_rate}
\end{equation}
\end{theorem}
The last term depends on $T$ only through the logarithmic factors hidden in $\tilde{O}$.
As in \Cref{thm:reg_occ_unknown}, it comes from the lower-order term of the bound in \cref{lem:transition_occupancy_deviation}, and we have not attempted to optimize it.

\subsection{Algorithm}
The confidence sets specify the bounds on the visit probabilities used in the estimates and the transition kernels used to compute the bonuses.
\paragraph{Transition confidence sets.}
We use the counts $N,M,N_k,M_k$, the epochs, the empirical transition kernels, and the confidence sets $\calC_k$ of \cref{app:global_unknown}, defined in~\cref{eq:global_empirical_transition,eq:global_transition_radius,eq:global_raw_confidence_set}.
The counts here are generated by the trajectories of \cref{alg:pf_mdp_policy_optimization_unk}.
In those definitions, we replace the confidence parameter denoted by $\delta$ with
\begin{equation}
    \delta'=\frac1{20H^3ST},\qquad \iota=\log\frac{TSA}{\delta'}.
 \label{eq:po_confidence_parameter}
\end{equation}
Here $\delta'$ controls the probability that a transition confidence set fails to contain $P$, while $\delta=2\eta(K/\gamma+H)$ remains the additive term in the denominators of the Q-function estimate.
We use $\calC_{k_t}$ directly, without the intersections in~\cref{eq:global_nested_confidence_set}, since the FTRL decision set here is always $\Delta(\calA)$.
These sets are nonempty, since for a visited $(s,a)$ the empirical distribution $\bar P_k(\cdot\mid s,a)$ satisfies the constraints on $P'(\cdot\mid s,a)$, and for an unvisited $(s,a)$ the radius $\beta_k$ exceeds one, so that every distribution satisfies them.
The initialization appears in~\cref{line:po_unknown_initialize} and the count updates in~\crefrange{line:po_unknown_counts}{line:po_unknown_confidence} of~\cref{alg:pf_mdp_policy_optimization_unk}.

\paragraph{Confidence bounds on the visit probabilities and Q-function estimates.}
Before episode $t$, the algorithm computes the following upper and lower confidence bounds on the visit probabilities (\cref{line:po_unknown_occupancy}), in the notation $q^{P',\pi}$ of~\cref{app:global_unknown}:
\begin{equation}
    \bar q_t(s)=\max_{P'\in\calC_{k_t}}q^{P',\pi_t}(s),\qquad
 \underline q_t(s)=\min_{P'\in\calC_{k_t}}q^{P',\pi_t}(s),\qquad
 \begin{cases}
 \bar q_t(s,a)=\pi_t(a\mid s)\bar q_t(s),\\
 \underline q_t(s,a)=\pi_t(a\mid s)\underline q_t(s).
 \end{cases}
 \label{eq:po_upper_lower_occupancy}
\end{equation}
The maximization is the same as in~\cref{eq:global_upper_state_probability}, with the policy fixed and the kernel ranging over $\calC_{k_t}$, and the minimization is analogous.
Both can be computed by dynamic programming as in \cref{app:global_unknown}.
The relation $q^{P',\pi_t}(s,a)=\pi_t(a\mid s)q^{P',\pi_t}(s)$ shows that a kernel attaining the maximum (or the minimum) at a state $s$ attains it at every $(s,a)$ as well.

After observing the episode, the algorithm computes $\hat\ell_t$ and $\hat L_{t,h}$ from~\cref{eq:def_l_hat,eq:def_Q_hat} and sets
\begin{equation}
    \hat Q_t(s,a)=
 \frac{\I_t(s)\pi_t(a\mid s)\hat\ell_t(s,a)
 +\I_t(s,a)\hat L_{t,h(s)+1}}{\bar q_t(s,a)+\delta}.
 \label{eq:def_Q_est_unk}
\end{equation}
The only difference from~\cref{eq:def_Q_hat} is that the unknown $q_t(s,a)$ in the denominator is replaced by its upper confidence bound $\bar q_t(s,a)$.
When $P\in\calC_{k_t}$, \cref{eq:po_unknown_Q_bias} in~\cref{lem:po_unknown_Q_moments} gives $\E_t\brk{\hat Q_t(s,a)}\leq Q_t(s,a)$ and bounds the difference $Q_t(s,a)-\E_t\brk{\hat Q_t(s,a)}$ in terms of $\delta$ and $\bar q_t(s,a)-\underline q_t(s,a)$.

\paragraph{Bonus term.}
The algorithm sets $M_t(s_H,\cdot)=0$ and computes the following quantities backward through the layers:
\begin{align}
    m_t(s)&=H\sum_a\tilde\pi_t(a\mid s)
  \frac{2\delta+\bar q_t(s,a)-\underline q_t(s,a)}{\bar q_t(s,a)+\delta},
 \label{eq:po_unknown_local_bonus}\\
 M_t(s,a)&=m_t(s)+\prn*{1+\frac1H}
  \max_{P'\in\calC_{k_t}}\sum_{s'}P'(s'\mid s,a)
                       \sum_{a'}\pi_t(a'\mid s')M_t(s',a').
 \label{eq:def_M_unk}
\end{align}
The difference $\bar q_t(s,a)-\underline q_t(s,a)$ in~\cref{eq:po_unknown_local_bonus} accounts for the additional bias from estimating $P$.
The definition in~\cref{eq:def_M_unk} follows \citet{luo21policy}, and the factor $1+1/H$ increases the weight of $m_t$ at each later step.
Whenever $P\in\calC_{k_t}$, the maximization includes $P$, and hence~\cref{eq:po_bonus_supersolution} in~\cref{lem:po_common_regret} holds.

\begin{algorithm2e}[t]
\KwIn{Learning rate $\eta>0$, exploration rate $\gamma\in(0,1/2]$, $\delta=2\eta(K/\gamma+H)$, and confidence parameter $\delta'$ in~\cref{eq:po_confidence_parameter}}
Set $\tilde\pi_1(a\mid s)=1/A$, $k=1$, all counts $N,M,N_1,M_1$ to zero, and $\calC_1$ equal to the set of all layered transition kernels\label{line:po_unknown_initialize}\;
\For{$t=1,\ldots,T$}{
 Set $k_t=k$ and $\pi_t(a\mid s)=(1-\gamma)\tilde\pi_t(a\mid s)+\gamma/A$\label{line:po_unknown_policy}\;
 Compute $\bar q_t,\underline q_t$ by~\cref{eq:po_upper_lower_occupancy}\label{line:po_unknown_occupancy}\;
 Compute $m_t,M_t$ by~\cref{eq:po_unknown_local_bonus,eq:def_M_unk}\label{line:po_unknown_bonus}\;
 Execute $\pi_t$ as in \crefrange{line:po_execute}{line:po_observe} of \cref{alg:pf_mdp_policy_optimization}\label{line:po_unknown_execute}\;
 Compute $\hat\ell_t$ and $\hat L_{t,h}$ as in \crefrange{line:po_loss_estimates}{line:po_local_estimates} of \cref{alg:pf_mdp_policy_optimization}, and compute $\hat Q_t$ by~\cref{eq:def_Q_est_unk}\label{line:po_unknown_estimates}\;
 Compute $\tilde\pi_{t+1}$ by~\cref{eq:po_ftrl_update}\label{line:po_unknown_ftrl}\;
 Increment $N(s,a)$ and $M(s,a,s')$ for every observed transition $(s,a,s')$\label{line:po_unknown_counts}\;
 \If{$t<T$ and $N(s,a)\geq\max\{1,2N_k(s,a)\}$ for some visited $(s,a)$}{
   Set $k\leftarrow k+1$, $N_k\leftarrow N$, and $M_k\leftarrow M$\label{line:po_unknown_epoch}\;
   Compute $\calC_k$ by~\cref{eq:global_empirical_transition,eq:global_transition_radius,eq:global_raw_confidence_set}, with confidence parameter $\delta'$\label{line:po_unknown_confidence}\;
 }
}
\caption{Policy optimization under unknown transitions}
\label{alg:pf_mdp_policy_optimization_unk}
\end{algorithm2e}
The algorithm computes the policy, the bounds $\bar q_t$ and $\underline q_t$, and the bonuses using only observations from preceding episodes (\crefrange{line:po_unknown_policy}{line:po_unknown_bonus}).
It then collects feedback and updates the policy (\crefrange{line:po_unknown_execute}{line:po_unknown_ftrl}).
The count updates in \crefrange{line:po_unknown_counts}{line:po_unknown_confidence} are the same as those in \cref{alg:pf_mdp_occ_unknown}, except that no intersection of confidence sets is taken.
The policy update is not restarted when a new epoch starts.

\subsection{Q-function estimate under unknown transitions}
The moments of $\hat L_{t,h}$ in~\cref{lem:po_suffix_moments} do not require knowledge of $P$.
We now account for the upper confidence bound $\bar q_t(s,a)$ in the denominator of~\cref{eq:def_Q_est_unk}.

\begin{lemma}\label{lem:po_unknown_Q_moments}
Suppose that $P\in\calC_{k_t}$.
Then, for every nonterminal $(s,a)$, the estimate in~\cref{eq:def_Q_est_unk} satisfies
\begin{align}
    \E_t\brk*{\hat Q_t(s,a)}
 &=\frac{q_t(s,a)}{\bar q_t(s,a)+\delta}Q_t(s,a),
 \label{eq:po_unknown_Q_mean}\\
 0\leq Q_t(s,a)-\E_t\brk*{\hat Q_t(s,a)}
 &\leq\frac{H(\delta+\bar q_t(s,a)-\underline q_t(s,a))}
                   {\bar q_t(s,a)+\delta},
 \label{eq:po_unknown_Q_bias}\\
 \E_t\brk*{\hat Q_t(s,a)^2}
 &\leq\frac{2H(K/\gamma+H)}{\bar q_t(s,a)+\delta}.
 \label{eq:po_unknown_Q_second}
\end{align}
\end{lemma}
\begin{proof}
If $q_t(s)=0$, then $\hat Q_t(s,a)=0$ almost surely, and~\cref{eq:po_unknown_Q_mean} holds since $q_t(s,a)=0$.
Otherwise, $q_t(s,a)>0$ since $\pi_t(a\mid s)\geq\gamma/A$, and taking conditional expectations of the two terms in~\cref{eq:def_Q_est_unk} gives
\begin{align*}
    \E_t\brk*{\hat Q_t(s,a)}
 &=\frac{q_t(s)\pi_t(a\mid s)}{\bar q_t(s,a)+\delta}
           \E_t\brk*{\hat\ell_t(s,a)\relmiddle|\I_t(s)=1}
   +\frac{q_t(s,a)}{\bar q_t(s,a)+\delta}
           \E_t\brk*{\hat L_{t,h(s)+1}\relmiddle|\I_t(s,a)=1}\\
 &=\frac{q_t(s,a)}{\bar q_t(s,a)+\delta}
           \prn[\big]{\ell_t(s,a)+Q_t(s,a)-\ell_t(s,a)}
  =\frac{q_t(s,a)}{\bar q_t(s,a)+\delta}Q_t(s,a),
\end{align*}
where the second equality follows from~\cref{eq:ell_hat_state_unbiased} in~\cref{lem:po_local_moments}, \cref{eq:Lcirc_future_exp} in~\cref{lem:po_suffix_moments}, and $q_t(s)\pi_t(a\mid s)=q_t(s,a)$.
This proves~\cref{eq:po_unknown_Q_mean}.

Under the assumption $P\in\calC_{k_t}$, we have $\underline q_t(s,a)\leq q_t(s,a)\leq\bar q_t(s,a)$, and hence the coefficient in~\cref{eq:po_unknown_Q_mean} lies in $[0,1]$.
Since $Q_t(s,a)\in[0,H]$, \cref{eq:po_unknown_Q_mean} gives
\begin{equation*}
    0\leq Q_t(s,a)-\E_t\brk*{\hat Q_t(s,a)}
 =\frac{\bar q_t(s,a)+\delta-q_t(s,a)}{\bar q_t(s,a)+\delta}Q_t(s,a)
 \leq\frac{H(\delta+\bar q_t(s,a)-\underline q_t(s,a))}{\bar q_t(s,a)+\delta},
\end{equation*}
where the last inequality follows from $q_t(s,a)\geq\underline q_t(s,a)$.
This proves~\cref{eq:po_unknown_Q_bias}, including states with $q_t(s)=0$.

We next prove~\cref{eq:po_unknown_Q_second}.
It suffices to consider $q_t(s)>0$.
Let $n=H-h(s)-1$.
Bounding the squares of the two terms in~\cref{eq:def_Q_est_unk} separately gives
\begin{align*}
    \E_t\brk*{\hat Q_t(s,a)^2}
 &\leq\frac{2q_t(s)\pi_t(a\mid s)^2}{(\bar q_t(s,a)+\delta)^2}
             \E_t\brk*{\hat\ell_t(s,a)^2\relmiddle|\I_t(s)=1}\\
 &\qquad+\frac{2q_t(s,a)}{(\bar q_t(s,a)+\delta)^2}
             \E_t\brk*{\hat L_{t,h(s)+1}^2\relmiddle|\I_t(s,a)=1}\\
 &\leq\frac{2K/\gamma+2(nK/\gamma+n(n-1))}{\bar q_t(s,a)+\delta}
 \leq\frac{2H(K/\gamma+H)}{\bar q_t(s,a)+\delta},
\end{align*}
where the first inequality follows from $(x+y)^2\leq2x^2+2y^2$, the second inequality from~\cref{eq:ell_hat_state_unbiased} in~\cref{lem:po_local_moments}, \cref{eq:Lcirc_second_future} in~\cref{lem:po_suffix_moments}, and $q_t(s)\pi_t(a\mid s)^2\leq q_t(s,a)\leq\bar q_t(s,a)+\delta$, and the last inequality from $n+1\leq H$, which completes the proof.
\end{proof}

\subsection{Bounds on the visit probabilities from transition confidence sets}
We use \cref{lem:transition_occupancy_deviation} with the executed policies of \cref{alg:pf_mdp_policy_optimization_unk} and confidence parameter $\delta'$ to bound the expected sum of the differences between the maximum and minimum in~\cref{eq:po_upper_lower_occupancy}.
\begin{lemma}\label{lem:po_occupancy_width}
For the confidence sets $\calC_{k_t}$ of \cref{alg:pf_mdp_policy_optimization_unk}, there is a universal constant $C$ such that
\begin{align}
    \P(P\notin\calC_{k_t}\text{ for some }t\leq T)&\leq CH\delta',
 \label{eq:po_confidence_containment}\\
 \E\brk*{\sumT\sum_{s,a}(\bar q_t(s,a)-\underline q_t(s,a))}
 &\leq C\prn*{HS\sqrt{AT}\,\iota+HS^4A\iota^2+HST\delta'}.
 \label{eq:po_expected_occupancy_width}
\end{align}
\end{lemma}
\begin{proof}
By~\cref{eq:global_raw_confidence_probability} with confidence parameter $\delta'$, the probability in~\cref{eq:po_confidence_containment} is at most $4\delta'\leq CH\delta'$.
If $T<A$, then
$\sum_{t,s,a}(\bar q_t(s,a)-\underline q_t(s,a))\leq ST\leq HS\sqrt{AT}$.
It therefore remains to prove~\cref{eq:po_expected_occupancy_width} for $T\geq A$.
Let $\calE$ be the event on which $P\in\calC_k$ for every epoch $k$ and~\cref{eq:po_imported_occupancy_deviation} holds for every choice of the kernels $P_t^s$, with confidence parameter $\delta'$.
By \cref{lem:transition_occupancy_deviation}, $\P(\calE^c)\leq C_0H\delta'$.

We choose kernels $P_t^{s,+},P_t^{s,-}\in\calC_{k_t}$ such that
\begin{equation}
    q^{P_t^{s,+},\pi_t}(s)=\bar q_t(s),\qquad
 q^{P_t^{s,-},\pi_t}(s)=\underline q_t(s).
 \label{eq:po_extremizing_kernels}
\end{equation}
Such kernels exist since $\calC_{k_t}$ is closed and bounded.
For each episode,
\begin{align*}
    \sum_{s,a}(\bar q_t(s,a)-\underline q_t(s,a))
 &=\sum_s\sum_a\pi_t(a\mid s)(\bar q_t(s)-\underline q_t(s))\\
 &=\sum_s(\bar q_t(s)-\underline q_t(s))\\
 &\leq\sum_s\left|q^{P_t^{s,+},\pi_t}(s)-q_t(s)\right|
       +\sum_s\left|q^{P_t^{s,-},\pi_t}(s)-q_t(s)\right|,
\end{align*}
where the first equality follows from~\cref{eq:po_upper_lower_occupancy}, the second equality from $\sum_a\pi_t(a\mid s)=1$, and the inequality from~\cref{eq:po_extremizing_kernels} and the triangle inequality after adding and subtracting $q_t(s)$.
On $\calE$, summing this inequality over $t$ gives
\begin{equation*}
    \sumT\sum_{s,a}(\bar q_t(s,a)-\underline q_t(s,a))
 \leq2C\prn*{HS\sqrt{AT}\,\iota+HS^4A\iota^2},
\end{equation*}
where we used~\cref{eq:po_imported_occupancy_deviation} in~\cref{lem:transition_occupancy_deviation} once for each of the two choices $P_t^{s,+}$ and $P_t^{s,-}$.
On $\calE^c$, the left-hand side is at most $ST$, since both $\bar q_t(s)$ and $\underline q_t(s)$ lie in $[0,1]$.
Consequently,
\begin{align*}
    &\E\brk*{\sumT\sum_{s,a}(\bar q_t(s,a)-\underline q_t(s,a))}\\
 &\leq2C\prn*{HS\sqrt{AT}\,\iota+HS^4A\iota^2}
       +ST\P(\calE^c)\\
 &\leq2C\prn*{HS\sqrt{AT}\,\iota+HS^4A\iota^2}+C_0HST\delta',
\end{align*}
where the last inequality follows from $\P(\calE^c)\leq C_0H\delta'$ by~\cref{lem:transition_occupancy_deviation}.
Increasing $C$ gives~\cref{eq:po_expected_occupancy_width}, which completes the proof.
\end{proof}

\subsection[Regret bound under unknown transitions]{Regret bound under unknown transitions (Proof of \cref{thm:regret_po_23_unk})}
The proof combines \cref{lem:po_common_regret,lem:po_occupancy_width}.
\begin{proof}[Proof of \cref{thm:regret_po_23_unk}]
We first verify the conditions of~\cref{lem:po_common_regret} up to the first episode in which the confidence set fails to contain $P$.
Define
\begin{equation}
    \tau=\min\prn[\big]{\{t\in[T]:P\notin\calC_{k_t}\}\cup\{T+1\}}.
 \label{eq:po_first_confidence_failure}
\end{equation}
For the fixed true kernel $P$,
$\{t<\tau\}=\bigcap_{u=1}^t\{P\in\calC_{k_u}\}$ is determined before episode $t$.

For $t<\tau$, we have $\hat Q_t\geq0$ by the definitions in~\cref{eq:def_l_hat,eq:def_Q_est_unk}, and~\cref{eq:po_unknown_Q_bias} in~\cref{lem:po_unknown_Q_moments} gives $\E_t\brk{\hat Q_t(s,a)}\leq Q_t(s,a)$, which verifies~\cref{eq:po_optimism_condition}.
For~\cref{eq:po_bias_variance_condition}, we have
\begin{align*}
    &\sum_a\tilde\pi_t(a\mid s)
  \prn[\Big]{Q_t(s,a)-\E_t\brk[\big]{\hat Q_t(s,a)}+\eta\E_t\brk[\big]{\hat Q_t(s,a)^2}}\\
 &\leq H\sum_a\tilde\pi_t(a\mid s)
  \frac{\delta+\bar q_t(s,a)-\underline q_t(s,a)+2\eta(K/\gamma+H)}
       {\bar q_t(s,a)+\delta}
 =m_t(s),
\end{align*}
where the inequality follows from~\cref{eq:po_unknown_Q_bias,eq:po_unknown_Q_second} in~\cref{lem:po_unknown_Q_moments}, and the equality from $\delta=2\eta(K/\gamma+H)$ and~\cref{eq:po_unknown_local_bonus}.
We next verify the remaining conditions on $m_t$ and $M_t$ in~\cref{lem:po_common_regret}.
Since $0\leq2\delta+\bar q_t(s,a)-\underline q_t(s,a)\leq3(\bar q_t(s,a)+\delta)$ and $\sum_a\tilde\pi_t(a\mid s)=1$, we have $0\leq m_t(s)\leq3H$ in~\cref{eq:po_unknown_local_bonus}, and expanding~\cref{eq:def_M_unk} gives
\begin{equation*}
    0\leq M_t(s,a)\leq3H\sum_{j=0}^{H-h(s)-1}(1+1/H)^j
 \leq3\e H^2\leq9H^2,
\end{equation*}
where we used $(1+1/H)^H\leq\e$.
For $t<\tau$, we have $P\in\calC_{k_t}$, and hence
\begin{equation*}
    M_t(s,a)\geq m_t(s)+(1+1/H)\sum_{s'}P(s'\mid s,a)\bar M_t(s')
 \geq m_t(s)+\sum_{s'}P(s'\mid s,a)\bar M_t(s'),
\end{equation*}
where the first inequality follows from taking $P'=P$ in the maximization in~\cref{eq:def_M_unk}, and the second inequality from $\bar M_t\geq0$.
We have verified the conditions of~\cref{lem:po_common_regret}, and $\eta\leq1/(9H^2)$ holds by~\cref{eq:po_parameters}, so that \cref{eq:po_common_regret_bound} gives
\begin{equation}
    \E\brk*{\Reg_T}\leq\gamma H^2T+\frac{H\log A}{\eta}+81\eta H^5T
      +\E\brk*{\sum_{t<\tau}\bar M_t(s_0)}+HT\P(\tau\leq T).
 \label{eq:po_unknown_before_bonus_cost}
\end{equation}
To bound the fourth term on the right-hand side of~\cref{eq:po_unknown_before_bonus_cost}, we choose $\tilde P_t\in\calC_{k_t}$ satisfying
\begin{equation*}
    \sum_{s'}\tilde P_t(s'\mid s,a)\bar M_t(s')
  =\max_{P'\in\calC_{k_t}}\sum_{s'}P'(s'\mid s,a)\bar M_t(s')
 \quad\text{for every }(s,a),
\end{equation*}
where such a choice exists since the constraints in~\cref{eq:global_raw_confidence_set} are imposed separately on each $P'(\cdot\mid s,a)$.
Expanding~\cref{eq:def_M_unk} with $\tilde P_t$ gives
\begin{align}
    \bar M_t(s_0)
 &=\sum_{h=0}^{H-1}(1+1/H)^h
       \sum_{s\in\calS_h}q^{\tilde P_t,\pi_t}(s)m_t(s)\nonumber\\
 &\leq\e\,V^{\tilde P_t,\pi_t}(s_0;m_t)
 \leq3\sum_s q^{\tilde P_t,\pi_t}(s)m_t(s),
 \label{eq:po_dilated_root_bound}
\end{align}
where the first inequality follows from $(1+1/H)^h\leq\e$, and the second inequality from the definition of the value function.
Substituting $m_t$ from~\cref{eq:po_unknown_local_bonus} yields
\begin{align}
    \sum_s q^{\tilde P_t,\pi_t}(s)m_t(s)
 &=H\sum_{s,a}
  \frac{q^{\tilde P_t,\pi_t}(s)\tilde\pi_t(a\mid s)
   (2\delta+\bar q_t(s,a)-\underline q_t(s,a))}
       {\bar q_t(s,a)+\delta}\nonumber\\
 &\leq2H\sum_{s,a}
  \frac{q^{\tilde P_t,\pi_t}(s,a)
   (2\delta+\bar q_t(s,a)-\underline q_t(s,a))}
       {\bar q_t(s,a)+\delta}\nonumber\\
 &\leq4\delta HSA+2H\sum_{s,a}(\bar q_t(s,a)-\underline q_t(s,a)),
 \label{eq:po_unknown_bonus_cost}
\end{align}
where the first inequality follows from $\tilde\pi_t(a\mid s)\leq2\pi_t(a\mid s)$ and $q^{\tilde P_t,\pi_t}(s)\pi_t(a\mid s)=q^{\tilde P_t,\pi_t}(s,a)$, and the second inequality from $q^{\tilde P_t,\pi_t}(s,a)/(\bar q_t(s,a)+\delta)\leq1$, which holds by~\cref{eq:po_upper_lower_occupancy}.

Substituting~\cref{eq:po_dilated_root_bound,eq:po_unknown_bonus_cost} into~\cref{eq:po_unknown_before_bonus_cost} and extending the sums over $t<\tau$ to $t\leq T$ gives
\begin{align}
    \E\brk*{\Reg_T}\leq{}&\gamma H^2T+\frac{H\log A}{\eta}
       +81\eta H^5T+12\delta HSAT\nonumber\\
 &+6H\E\brk*{\sumT\sum_{s,a}(\bar q_t(s,a)-\underline q_t(s,a))}
       +HT\P(\tau\leq T).
 \label{eq:po_unknown_master}
\end{align}
The first four terms in~\cref{eq:po_unknown_master} equal the right-hand side of~\cref{eq:po_master_known} after substituting $\delta=2\eta(K/\gamma+H)$ and $A=K^2$.
For the remaining two terms, \cref{eq:po_confidence_containment,eq:po_expected_occupancy_width} in~\cref{lem:po_occupancy_width} give
\begin{align}
    &6H\E\brk*{\sumT\sum_{s,a}(\bar q_t(s,a)-\underline q_t(s,a))}
       +HT\P(\tau\leq T)\nonumber\\
 &\leq C\prn*{H^2S\sqrt{AT}\,\iota+H^2S^4A\iota^2+H^2ST\delta'}\nonumber\\
 &\leq C\prn*{H^2S\sqrt{AT}\,\iota+H^2S^4A\iota^2+1},
 \label{eq:po_unknown_transition_cost}
\end{align}
where the last inequality follows from $\delta'=1/(20H^3ST)$ in~\cref{eq:po_confidence_parameter}.
When $\gamma<1/2$ and $\eta<1/(9H^2)$ in~\cref{eq:po_parameters}, substituting~\cref{eq:po_unknown_transition_cost} into~\cref{eq:po_unknown_master} and using~\cref{eq:po_tuned_bound,eq:po_parameter_powers} in the proof of~\cref{thm:regret_po_23} yields
\begin{align*}
    \E\brk*{\Reg_T}\leq O\Bigl(& (H^4SK^3\log K)^{1/3}T^{2/3}
       +(H^3+H^{3/2}K\sqrt S)\sqrt{T\log K}\\
       &\qquad+H^2SK\sqrt T\,\iota+H^2S^4K^2\iota^2+1\Bigr).
\end{align*}
Since $S\geq H+1$ and $K\geq2$, we have
$H^3+H^{3/2}K\sqrt S\leq2H^2SK$,
and the square-root terms are thus absorbed into $H^2SK\sqrt T\,\iota$.
If $\gamma=1/2$ or $\eta=1/(9H^2)$, the same comparison with $HT$ as in the proof of~\cref{thm:regret_po_23} applies.
Combining these bounds with $\Reg_T\leq HT$ gives~\cref{eq:po_unknown_rate}, which completes the proof.
\end{proof}

The second term inside the parentheses in~\cref{eq:po_unknown_transition_cost} comes from the same bound on the occupancy measures as for global optimization (\cref{lem:transition_occupancy_deviation}), and the additional factor $H$ comes from the coefficient $6H$ in~\cref{eq:po_unknown_master}.
This term is dominated by $H^2S\sqrt{AT}\,\iota$ when $T\geq S^6A\iota^2$.

\section{Proof of the regret lower bound (\Cref{thm:lower_bound_adv_pref_borda})}\label{app:proof_lower_bounds}
This section provides the proof of \Cref{thm:lower_bound_adv_pref_borda}.
Throughout, $\mathrm D(\P,\P')$ denotes the Kullback--Leibler~(KL) divergence and $\mathrm{kl}(p,q)$ the KL divergence between Bernoulli distributions with means $p$ and $q$.

\begin{proof}[Proof of \Cref{thm:lower_bound_adv_pref_borda}]
Fix a learning algorithm.
We construct a layered MDP in which comparisons that provide information about a Borda winner force the learner into states with larger losses for the rest of the episode.
The transition kernel is known, and each state's preference matrix is fixed over all episodes.

The instance is shown in \Cref{fig:strong_pref_instance}, and its construction is given in the first three steps below.
We then lower bound the regret averaged over a random choice of the preference matrices in two ways, through the perturbation of the matrices, which the learner cannot detect without many comparisons that reveal the Borda winner, and through the losses at the penalty states, which every such comparison incurs for the rest of the episode.
Adding the two bounds and choosing the perturbation size gives the theorem.

The instance has $M$ chains, and the $H-1$ intermediate layers of each chain are split into a first part of $L$ layers and a final part of $R$ layers (\Cref{fig:strong_pref_instance}), where
\begin{equation*}
    M
  =
  \floor*{\frac{S-2}{2(H-1)}}
  \com\qquad
  L
  =
  \floor*{\frac{H-1}{2}}
  \com\qquad
  R
  =
  H-1-L
  \per
\end{equation*}
The assumption $S\geq2H$ ensures that $M\geq1$.
The preference matrices in the first part are unknown to the learner, and the final part penalizes the comparisons that reveal them.

\proofparagraph{States and transitions.}
For every $c\in[M]$ and $h\in[H-1]$, we introduce two states $x_{c,h}$ and $z_{c,h}$ in layer $h$ (\Cref{fig:strong_pref_instance}(b)).
We call the states $z_{c,h}$ penalty states.
Together with $s_0$ and $s_H$, the construction uses $2+2M(H-1)\leq S$ states.
If the inequality is strict, we add unreachable states to the intermediate layers, with arbitrary transitions to the next layer.
Let
\begin{equation*}
    \calX
  =
  \set*{x_{c,h}:c\in[M],\ h\in[L]}
\end{equation*}
be the set of the states $x_{c,h}$ in the first $L$ intermediate layers.

Recall that $\calA=[K]\times[K]$.
Let $\Ktil=\floor*{K/2}$, $\calK_g=[\Ktil]$, and $\calK_b=[K]\setminus\calK_g$.
Define
\begin{equation*}
    \calA_b
  =
  \calA\setminus\prn*{\calK_g\times\calK_g}
\end{equation*}
so that $a\in\calA_b$ exactly when at least one coordinate of $a$ belongs to $\calK_b$.

For every action $a$, the transition from $s_0$ is uniform over $x_{1,1},\ldots,x_{M,1}$ (\Cref{fig:strong_pref_instance}(a)).
For $c\in[M]$, $a\in\calA$, and $h\in[H-2]$, the remaining nonzero entries of the transition kernel $P$ are
\begin{equation*}
    \begin{aligned}
  P(x_{c,h+1}\mid x_{c,h},a)
  &=1
  &&\text{if }h\leq L\text{ and }a\notin\calA_b\com\\
  P(z_{c,h+1}\mid x_{c,h},a)
  &=1
  &&\text{if }h\leq L\text{ and }a\in\calA_b\com\\
  P(x_{c,h+1}\mid x_{c,h},a)
  &=1
  &&\text{if }h>L\com\\
  P(z_{c,h+1}\mid z_{c,h},a)
  &=1
  &&\text{for every }h\in[H-2]
  \per
\end{aligned}
\end{equation*}
All states in layer $H-1$ lead to $s_H$.
Hence a pair in $\calA_b$ selected at a state in $\calX$ moves the trajectory to the penalty states, where it stays for the rest of the episode, and in particular for the final $R$ layers.
The transition kernel is the same for every choice of the preference matrices below, so that knowing it gives the learner no information about the matrices.

\proofparagraph{Preference matrices.}
We now define the matrices placed at the states in $\calX$.
Let $K_{\rm b}=K-\Ktil$.
For $\epsilon\in(0,1/40]$, define
\begin{equation}
    \begin{aligned}
  \Pmat
  &=
  \left(
  \renewcommand{\arraystretch}{1.4}
  \begin{array}{@{\hspace{4pt}}c|c@{\hspace{4pt}}}
    0.5\,\ones_{\Ktil\times\Ktil}
    &
    0.9\,\ones_{\Ktil\times K_{\rm b}}
    \\[8pt]\hline\rule{0pt}{12pt}
    0.1\,\ones_{K_{\rm b}\times\Ktil}
    &
    0.5\,\ones_{K_{\rm b}\times K_{\rm b}}
  \end{array}
  \right)
  \com\\[4pt]
  \Pmat^{(m)}
  &=
  \left(
  \renewcommand{\arraystretch}{1.4}
  \begin{array}{@{\hspace{4pt}}c|c@{\hspace{4pt}}}
    0.5\,\ones_{\Ktil\times\Ktil}
    &
    0.9\,\ones_{\Ktil\times K_{\rm b}}
      +2\epsilon e_m\ones_{K_{\rm b}}^{\top}
    \\[8pt]\hline\rule{0pt}{12pt}
    0.1\,\ones_{K_{\rm b}\times\Ktil}
      -2\epsilon\ones_{K_{\rm b}}e_m^{\top}
    &
    0.5\,\ones_{K_{\rm b}\times K_{\rm b}}
  \end{array}
  \right)
  \com\qquad m\in\calK_g
  \per
\end{aligned}
  \label{eq:strong_pref_matrices}
\end{equation}
Here $e_m$ is the $m$-th standard basis vector in $\R^{\Ktil}$.
We also define the matrix at the penalty states,
\begin{equation}
    \mathsf U
  =
  0.5\,\ones_{K\times K}
  \per
  \label{eq:strong_pref_U}
\end{equation}
For a preference matrix $\mathsf Q$, we write the shifted Borda score and the corresponding loss as
\begin{equation}
    b_{\mathsf Q}(i)=\frac1K\sum_{r=1}^K(\mathsf Q[i,r]-1),
  \qquad
  \ell_{\mathsf Q}(i,k)=-\frac12\prn[\big]{b_{\mathsf Q}(i)+b_{\mathsf Q}(k)}.
  \label{eq:strong_pref_borda_definition}
\end{equation}

Under $\Pmat$, all arms in $\calK_g$ have the same Borda score $\beta_g$ and all arms in $\calK_b$ have the same Borda score $\beta_b$, where
\begin{equation*}
    \beta_g
  =
  \frac{0.5\Ktil+0.9K_{\rm b}}{K}-1
  \com\qquad
  \beta_b
  =
  \frac{0.1\Ktil+0.5K_{\rm b}}{K}-1
  =
  \beta_g-\frac25
  \per
\end{equation*}
Hence
\begin{equation}
    b_{\Pmat}(i)-b_{\Pmat}(k)
  =
  \frac25
  \qquad
  \text{for every }i\in\calK_g\text{ and }k\in\calK_b
  \per
  \label{eq:strong_pref_baseline_score_gap}
\end{equation}
Under $\Pmat^{(m)}$, the Borda scores are
\begin{equation*}
    b_{\Pmat^{(m)}}(i)
  =
  \begin{cases}
    \beta_g+\Delta & \mbox{if }i=m\com\\
    \beta_g & \mbox{if }i\in\calK_g\setminus\set{m}\com\\
    \beta_b-2\epsilon/K & \mbox{if }i\in\calK_b\com
  \end{cases}
  \qquad
  \Delta
  =
  \frac{2\epsilon K_{\rm b}}{K}
  \geq
  \epsilon
  \per
\end{equation*}
Consequently, $(m,m)$ is the unique optimal pair under $\Pmat^{(m)}$, and for every other pair $a=(i,k)$,
\begin{align}
    \ell_{\Pmat^{(m)}}(a)
  -
  \ell_{\Pmat^{(m)}}(m,m)
  &=
  \frac12\prn*{
    b_{\Pmat^{(m)}}(m)-b_{\Pmat^{(m)}}(i)
    +
    b_{\Pmat^{(m)}}(m)-b_{\Pmat^{(m)}}(k)
  }
  \geq
  \frac{\epsilon}{2}
  \per
  \label{eq:strong_pref_normal_gap}
\end{align}
At the penalty states, the matrix is $\mathsf U$, so that $b_{\mathsf U}(i)=-1/2$ and $\ell_{\mathsf U}(a)=1/2$ for every $i$ and $a$.
Since
\begin{equation}
    \frac12-\ell_{\Pmat}(1,1)
  =
  \frac12+\beta_g
  =
  \frac{2K_{\rm b}}{5K}
  \geq
  \frac15
  \label{eq:strong_pref_penalty_gap}
\end{equation}
and $\ell_{\Pmat^{(m)}}(m,m)\leq\ell_{\Pmat}(1,1)$, the loss $1/2$ at a penalty state exceeds the minimal loss under $\Pmat$ and under $\Pmat^{(m)}$ by at least $1/5$.

The preference matrices are those of \citet[Lemma~14]{saha21adversarial}, with perturbation size $2\epsilon$ and group sizes $\Ktil$ and $K_{\rm b}$.
In our construction, every pair to which $\Pmat$ and $\Pmat^{(m)}$ assign different probabilities belongs to $\calA_b$ and therefore changes the subsequent state trajectory.

\proofparagraph{Random choice of the preference matrices.}
We choose the matrices at random before the interaction begins.
Set $\Pmat^{(0)}=\Pmat$, let $T'=T/M$ be the expected number of episodes in which a given chain is visited, and let
\begin{equation}
    \epsilon
  =
  c_0\min\set*{1,\prn*{\frac{KR}{T'}}^{1/3}}
  \per
  \label{eq:strong_pref_epsilon}
\end{equation}
The universal constant $c_0\in(0,1/40]$ will be fixed at the end of the proof.

For every $x\in\calX$, let $J_x\in\set{0}\cup\calK_g$ specify the matrix at $x$.
We draw the indices independently according to
\begin{equation}
    \Pr[J_x=0]
  =
  \frac12
  \com\qquad
  \Pr[J_x=m]
  =
  \frac{1}{2\Ktil}
  \quad
  \text{for every }m\in\calK_g
  \per
  \label{eq:strong_pref_index_distribution}
\end{equation}
Let $J=(J_x)_{x\in\calX}$ and let $\calD_{\rm pref}$ denote its product distribution.
The indices are not revealed to the learner.

For a realization $j=(j_x)_{x\in\calX}$, define, for every episode $t$, the preference function
\begin{equation*}
    \Pmat_t^{(j)}(s,(i,k))
  =
  \begin{cases}
    \Pmat^{(j_x)}[i,k] & \mbox{if }s=x\in\calX\com\\
    \Pmat[i,k] & \mbox{if }s=s_0\text{ or }s=x_{c,h}\text{ with }h>L\com\\
    \mathsf U[i,k] & \mbox{otherwise}\per
  \end{cases}
\end{equation*}
Thus the matrices at all penalty states and additional unreachable states are $\mathsf U$ as shown in \Cref{fig:strong_pref_instance}.
Let $\calM_{\mathsf{pref}}(j)$ be the resulting instance and $\ell^{(j)}$ its Borda-induced loss.
For fixed $j$, the preference matrices and $\ell^{(j)}$ are fixed over all episodes.
The regret is measured with $\ell^{(j)}$, while the learner observes only the Bernoulli outcomes of its comparisons.
We write $\E^{(j)}$ for the expectation over the learner's randomization, transitions, and preference outcomes under this fixed instance.

\proofparagraph{Optimal policy.}
The comparator in the regret is the best fixed policy for the instance $\calM_{\mathsf{pref}}(j)$, which we now identify.
Define $a_0^\star=(1,1)$ and $a_m^\star=(m,m)$ for $m\in\calK_g$, and let $\pi_j^\star$ be the policy that selects $a_{j_x}^\star$ at each $x\in\calX$, $(1,1)$ at $s_0$ and at the states $x_{c,h}$ with $h>L$, and an arbitrary pair elsewhere.

At $x\in\calX$, the pair $a_{j_x}^\star$ minimizes the immediate loss and leads to $x_{c,h+1}$.
An action in $\calA_b$ has no smaller immediate loss and leads instead to the penalty states, whose losses exceed the corresponding losses under $\pi_j^\star$.
At the remaining states $x_{c,h}$, the transition does not depend on the action, and minimizing the immediate loss is therefore sufficient.
Backward induction gives $\pi_j^\star\in\argmin_{\pi\in\Pi}V^\pi(s_0;\ell^{(j)})$.
Since $\ell^{(j)}$ is fixed over episodes, this policy attains the minimum in the regret definition for each fixed $j$.

\proofparagraph{Loss difference at one state.}
We first express the loss difference at one state using the learner's visit probabilities.
Let $q_t=q^{\pi_t}$ be the occupancy measure of the policy selected in episode $t$ under the common transition kernel, and define
\begin{equation}
    N_T(y,a)
  =
  \sumT q_t(y,a)
  \qquad
  \text{for }y\in\calS\setminus\set{s_H}\text{ and }a\in\calA
  \per
  \label{eq:strong_pref_N}
\end{equation}
Since the policies depend on past observations, $N_T(y,a)$ may be random.
Fix $x=x_{c,h}\in\calX$ and write $z=z_{c,h}$.
Every policy enters chain $c$ with probability $1/M$ and then reaches exactly one of $x$ and $z$ at layer $h$ as shown in \Cref{fig:strong_pref_instance}, and thus $q_t(x)+q_t(z)=1/M$ and
\begin{equation}
    \sum_{a\in\calA}
  \prn*{N_T(x,a)+N_T(z,a)}
  =
  \frac{T}{M}
  =
  T'
  \com\qquad
  0\leq N_T(x,a)\leq T'
  \per
  \label{eq:strong_pref_pair_occupancy}
\end{equation}
Recall that $\I_t(y,a)$ indicates whether $(y,a)$ is visited in episode $t$.
Given the observations before episode $t$, which determine $\pi_t$, the state-action pair $(y,a)$ is visited with probability $q_t(y,a)$, that is, $\E^{(j)}\brk*{\I_t(y,a)\relmiddle|\pi_t}=q_t(y,a)$, and hence
\begin{equation}
    \E^{(j)}\brk*{\sumT \I_t(y,a)}
  =
  \E^{(j)}\brk*{\sumT q_t(y,a)}
  =
  \E^{(j)}\brk*{N_T(y,a)}
  \per
  \label{eq:strong_pref_realized_to_expected}
\end{equation}

Fix a realization $j_{-x}=(j_{x'})_{x'\in\calX\setminus\set{x}}$ of the indices at the other states.
For $u\in\set{0}\cup\calK_g$, let $\calM_{\mathsf{pref}}(u,j_{-x})$ be the instance in which the matrix at $x$ is $\Pmat^{(u)}$ and the other indices are $j_{-x}$.
Thus the instances considered below differ only in the matrix at $x$.
Let $\P_{u,j_{-x}}$ be the joint distribution of the learner's randomization, the selected policies, the state-action trajectories, and the observed preference outcomes, and let $\E^{(u,j_{-x})}$ be the corresponding expectation.
Since $x$ and $j_{-x}$ remain fixed in this step and the next, we drop them from the notation and write
\begin{equation}
    \P_0=\P_{0,j_{-x}}
  \com\quad
  \P_m=\P_{m,j_{-x}}
  \com\quad
  \E_0=\E^{(0,j_{-x})}
  \com\quad
  \E_m=\E^{(m,j_{-x})}
  \per
  \label{eq:strong_pref_abbreviations}
\end{equation}
Both lower bounds below are expressed through the expected number of comparisons in $\calA_b$ at $x$ under the baseline matrix $\Pmat$,
\begin{equation}
    B_x(j_{-x})
  =
  \sum_{a\in\calA_b}
  \E_0\brk*{N_T(x,a)}
  \per
  \label{eq:strong_pref_Bx}
\end{equation}

Let $R_{x,u}(j_{-x})$ be the expected loss difference between the learner and the comparator $\pi_j^\star$ with $j_x=u$ at layer $h$ of chain $c$, that is, at the states $x$ and $z$, under the instance $\calM_{\mathsf{pref}}(u,j_{-x})$.
The comparator visits $x$ with probability $1/M$ and never visits $z$, so that its expected loss at this layer of the chain is $T'\ell_{\Pmat^{(u)}}(a_u^\star)$, and hence
\begin{align*}
    R_{x,u}(j_{-x})
  &=
  \E^{(u,j_{-x})}\brk*{
    \sum_{a\in\calA}N_T(x,a)\ell_{\Pmat^{(u)}}(a)
    +\frac12\sum_{a\in\calA}N_T(z,a)
  }
  -T'\ell_{\Pmat^{(u)}}(a_u^\star)
  \nn
  &=
  \E^{(u,j_{-x})}\brk*{
    \sum_{a\in\calA}N_T(x,a)
    \prn*{\ell_{\Pmat^{(u)}}(a)-\ell_{\Pmat^{(u)}}(a_u^\star)}
  }
  \nn
  &\qquad+
  \E^{(u,j_{-x})}\brk*{
    \sum_{a\in\calA}N_T(z,a)
    \prn*{\frac12-\ell_{\Pmat^{(u)}}(a_u^\star)}
  }
  \com
\end{align*}
where the second equality follows from \cref{eq:strong_pref_pair_occupancy}.
Both gaps in the last expression are nonnegative.
When the matrix at $x$ is $\Pmat^{(m)}$, we have
\begin{align}
    R_{x,m}(j_{-x})
  &\geq
  \frac{\epsilon}{2}
  \E_m\brk*{
    \sum_{a\neq(m,m)}N_T(x,a)
    +
    \sum_{a\in\calA}N_T(z,a)
  }
  \nn
  &=
  \frac{\epsilon}{2}
  \prn*{
    T'-\E_m\brk*{N_T(x,(m,m))}
  }
  \com
  \label{eq:strong_pref_alternative_regret}
\end{align}
where the inequality follows from the gap of at least $\epsilon/2$ for every pair $a\neq(m,m)$ in \cref{eq:strong_pref_normal_gap} and from the gap of at least $1/5\geq\epsilon/2$ at the penalty states, shown after \cref{eq:strong_pref_penalty_gap}, and the equality from \cref{eq:strong_pref_pair_occupancy}.

\proofparagraph{Change of measure at one state.}
This step shows that, on average over $m$, the learner cannot select $(m,m)$ much more often under $\Pmat^{(m)}$ than under $\Pmat$ unless $B_x(j_{-x})$ is large.
We compare the two expectations of $N_T(x,(m,m))$.
Since this variable takes values in $[0,T']$, the variational characterization of the total variation distance and Pinsker's inequality give
\begin{align}
    \frac1{T'}\E_m\brk*{N_T(x,(m,m))}
  &\leq
  \frac1{T'}\E_0\brk*{N_T(x,(m,m))}
  +
  \mathrm{TV}(\P_0,\P_m)
  \nn
  &\leq
  \frac1{T'}\E_0\brk*{N_T(x,(m,m))}
  +
  \sqrt{\frac12\mathrm D(\P_0,\P_m)}
  \com
  \label{eq:strong_pref_pinsker}
\end{align}
where $\mathrm{TV}(\cdot, \cdot)$ denotes the total variation distance.

Under $\P_0$ and $\P_m$, the learner's decision rules and the transition kernel are the same, and the conditional distribution of a preference outcome differs only at $x$ for the pairs $(m,i)$ and $(i,m)$ with $i\in\calK_b$.
By the chain rule for the KL divergence and \cref{eq:strong_pref_realized_to_expected},
\begin{align}
    \mathrm D(\P_0,\P_m)
  &=
  \sum_{i\in\calK_b}
  \E_0\brk*{N_T(x,(m,i))}
  \mathrm{kl}(0.9,0.9+2\epsilon)
  \nn
  &\qquad+
  \sum_{i\in\calK_b}
  \E_0\brk*{N_T(x,(i,m))}
  \mathrm{kl}(0.1,0.1-2\epsilon)
  \nn
  &\leq
  C_{\rm kl}\epsilon^2
  \sum_{i\in\calK_b}
  \E_0\brk*{
    N_T(x,(m,i))+N_T(x,(i,m))
  }
  \com
  \label{eq:strong_pref_KL}
\end{align}
where $C_{\rm kl}=85$, and the inequality follows from $\mathrm{kl}(p,q)\leq(p-q)^2/(q(1-q))$ and $q(1-q)\geq0.05\cdot0.95$ for both divergences, which holds since $\epsilon\leq1/40$.

Since $\Ktil=\floor*{K/2}\geq K/3$ for $K\geq5$, \cref{eq:strong_pref_pair_occupancy} implies
\begin{equation}
    \frac1{\Ktil}
  \sum_{m\in\calK_g}
  \E_0\brk*{N_T(x,(m,m))}
  \leq
  \frac{T'}{\Ktil}
  \leq
  \frac{3T'}K
  \per
  \label{eq:strong_pref_baseline_diagonal}
\end{equation}
Moreover, summing \cref{eq:strong_pref_KL} over $m$ and using \cref{eq:strong_pref_Bx} yields
\begin{equation*}
    \sum_{m\in\calK_g}
  \mathrm D(\P_0,\P_m)
  \leq
  C_{\rm kl}\epsilon^2B_x(j_{-x})
  \per
\end{equation*}
Jensen's inequality then gives
\begin{align}
    \frac1{\Ktil}
  \sum_{m\in\calK_g}
  \sqrt{\frac12\mathrm D(\P_0,\P_m)}
  &\leq
  \sqrt{
    \frac1{2\Ktil}
    \sum_{m\in\calK_g}
    \mathrm D(\P_0,\P_m)
  }
  \nn
  &\leq
  \epsilon
  \sqrt{\frac{3C_{\rm kl}B_x(j_{-x})}{2K}}
  \per
  \label{eq:strong_pref_average_KL}
\end{align}
Averaging \cref{eq:strong_pref_pinsker} over $m$ and using \cref{eq:strong_pref_baseline_diagonal,eq:strong_pref_average_KL}, we obtain
\begin{align}
    \frac1{\Ktil}
  \sum_{m\in\calK_g}
  \E_m\brk*{N_T(x,(m,m))}
  &\leq
  \frac{3T'}K
  +
  T'\epsilon
  \sqrt{\frac{3C_{\rm kl}B_x(j_{-x})}{2K}}
  \per
  \label{eq:strong_pref_average_optimal_count}
\end{align}
Substituting this inequality into \cref{eq:strong_pref_alternative_regret}, and using $1-3/K\geq2/5$ for $K\geq5$, gives
\begin{align}
    \frac1{2\Ktil}
  \sum_{m\in\calK_g}
  R_{x,m}(j_{-x})
  &\geq
  \frac{\epsilon}{4}\prn*{
    T'
    -
    \frac1{\Ktil}
    \sum_{m\in\calK_g}
    \E_m\brk*{N_T(x,(m,m))}
  }
  \nn
  &\geq
  \frac{\epsilon T'}{10}
  -
  C_1\epsilon^2T'
  \sqrt{\frac{B_x(j_{-x})}{K}}
  \com
  \label{eq:strong_pref_local_alternative_bound}
\end{align}
where $C_1=\sqrt{3C_{\rm kl}/2}/4$.
Thus the average regret at $x$ under the perturbed matrices is of order $\epsilon T'$ unless $B_x(j_{-x})$ is of order $K/\epsilon^2$.

\proofparagraph{Two lower bounds on the average regret.}
We bound the regret averaged over the independent choices of the preference matrices.
Let
\begin{equation*}
    \mathfrak R_T
  =
  \E_{J\sim\calD_{\rm pref}}
  \brk*{\E^{(J)}\brk*{\Reg_T}}
  \per
\end{equation*}
We first use the losses incurred in the first $L$ intermediate layers.
For a fixed realization $j$, the expected regret equals the total expected loss difference from $\pi_j^\star$.
The contribution at layer $h\leq L$ in chain $c$ is $R_{x_{c,h},j_{x_{c,h}}}(j_{-x_{c,h}})$, while the contributions at $s_0$ and in the last $R$ intermediate layers are nonnegative, so that
\begin{equation}
    \E^{(j)}\brk*{\Reg_T}
  \geq
  \sum_{x\in\calX}
  R_{x,j_x}(j_{-x})
  \per
  \label{eq:strong_pref_active_decomposition}
\end{equation}
For each fixed $j_{-x}$, the independence of the indices gives
\begin{equation*}
    \E_J\brk*{R_{x,J_x}(J_{-x})\relmiddle| J_{-x}=j_{-x}}
  =\frac12R_{x,0}(j_{-x})
   +\frac1{2\Ktil}\sum_{m\in\calK_g}R_{x,m}(j_{-x}).
\end{equation*}
Taking expectations in \cref{eq:strong_pref_active_decomposition}, dropping the nonnegative terms $\frac12R_{x,0}(J_{-x})$, and using \cref{eq:strong_pref_local_alternative_bound} yields
\begin{align*}
    \mathfrak R_T
  &\geq
  \sum_{x\in\calX}
  \E_{J_{-x}}\brk*{
    \frac1{2\Ktil}
    \sum_{m\in\calK_g}
    R_{x,m}(J_{-x})
  }
  \nn
  &\geq
  \frac{\epsilon T'LM}{10}
  -
  C_1\epsilon^2T'
  \sum_{x\in\calX}
  \E_{J_{-x}}\brk*{
    \sqrt{\frac{B_x(J_{-x})}{K}}
  }
  \per
\end{align*}

We next use the losses in the last $R$ intermediate layers.
Fix $j$.
If the learner selects a pair in $\calA_b$ at some $x\in\calX$ in an episode, which happens at most once per episode since the trajectory then stays at the penalty states, the learner incurs the loss $1/2$ at each of the last $R$ layers, whereas the comparator incurs $\ell_{\Pmat}(1,1)$, and the difference is at least $R/5$ by \cref{eq:strong_pref_penalty_gap}.
All other contributions to $V^{\pi_t}(s_0;\ell^{(j)})-V^{\pi_j^\star}(s_0;\ell^{(j)})$ are nonnegative, since the comparator minimizes the loss at every state $x_{c,h}$ and the losses at the penalty states are larger.
Hence, counting only the comparisons at the states with $j_x=0$ so that the counts match $B_x(j_{-x})$ after averaging over $J$,
\begin{equation*}
    V^{\pi_t}(s_0;\ell^{(j)})-V^{\pi_j^\star}(s_0;\ell^{(j)})
  \geq
  \frac{R}{5}
  \sum_{x\in\calX}
  \ind{j_x=0}
  \sum_{a\in\calA_b}
  q_t(x,a)
  \per
\end{equation*}
Summing over $t$ and taking expectations gives
\begin{align*}
    \E^{(j)}\brk*{\Reg_T}
  &\geq
  \frac{R}{5}
  \sum_{x\in\calX}
  \ind{j_x=0}
  \sum_{a\in\calA_b}
  \E^{(j)}\brk*{N_T(x,a)}
  \per
\end{align*}
Averaging this inequality over $J$ and using \cref{eq:strong_pref_Bx,eq:strong_pref_index_distribution} gives
\begin{align*}
    \mathfrak R_T
  &\geq
  \frac{R}{5}
  \sum_{x\in\calX}
  \E_J\brk*{
    \ind{J_x=0}
    \sum_{a\in\calA_b}
    \E^{(J)}\brk*{N_T(x,a)}
  }
  \nn
  &=
  \frac{R}{10}
  \sum_{x\in\calX}
  \E_{J_{-x}}\brk*{B_x(J_{-x})}
  \per
\end{align*}

\proofparagraph{Combining the two bounds.}
The first bound is useful when the counts $B_x$ are small, and the second when they are large.
Adding the two inequalities gives
\begin{align}
    2\mathfrak R_T
  &\geq
  \frac{\epsilon T'LM}{10}
  +
  \sum_{x\in\calX}
  \E_{J_{-x}}\brk*{
    \frac{R}{10}B_x(J_{-x})
    -
    C_1\epsilon^2T'
    \sqrt{\frac{B_x(J_{-x})}{K}}
  }
  \per
  \label{eq:strong_pref_combined_before_square}
\end{align}
For every $B\geq0$,
\begin{align}
    \frac{R}{10}B
  -
  C_1\epsilon^2T'\sqrt{\frac BK}
  &=
  \frac{R}{10}
  \prn*{
    \sqrt B
    -
    \frac{5C_1\epsilon^2T'}{R\sqrt K}
  }^2
  -
  \frac{5C_1^2\epsilon^4(T')^2}{2KR}
  \nn
  &\geq
  -
  \frac{5C_1^2\epsilon^4(T')^2}{2KR}
  \per
  \label{eq:strong_pref_square_completion}
\end{align}
Since $|\calX|=LM$, \cref{eq:strong_pref_combined_before_square,eq:strong_pref_square_completion} imply
\begin{align*}
    2\mathfrak R_T
  &\geq
  \epsilon T'LM
  \prn*{
    \frac1{10}
    -
    \frac{5C_1^2}{2}
    \frac{\epsilon^3T'}{KR}
  }
  \per
\end{align*}
The choice in \cref{eq:strong_pref_epsilon} gives $\epsilon^3T'/(KR)\leq c_0^3$.
We choose $c_0\leq1/40$ sufficiently small that $(5C_1^2/2)c_0^3\leq1/20$.
Then
\begin{equation}
    \mathfrak R_T
  \geq
  \frac{\epsilon T'LM}{40}
  \per
  \label{eq:strong_pref_bayes_final_MLR}
\end{equation}
By \cref{eq:strong_pref_epsilon} and $MT'=T$,
\begin{align}
    \epsilon T'LM
  &=
  c_0LM
  \min\set*{
    T',
    (KR)^{1/3}(T')^{2/3}
  }
  \nn
  &=
  c_0
  \min\set*{
    LT,
    L(KRM)^{1/3}T^{2/3}
  }
  \per
  \label{eq:strong_pref_scale_MLR}
\end{align}
For $H\geq3$, $L\geq(H-1)/3\geq2H/9$ and $R\geq(H-1)/2\geq H/3$.
Also, since $S\geq2H$ and $\floor{v}\geq v/2$ for $v\geq1$, we have
\begin{equation*}
    M\geq\frac{S-2}{4(H-1)}\geq\frac{S}{8H}.
\end{equation*}
It follows that $LT\geq(2/9)HT$ and
\begin{align}
    L(KRM)^{1/3}
  &\geq
  \frac{2H}{9}
  \prn*{
    K\cdot\frac H3\cdot\frac{S}{8H}
  }^{1/3}
  \geq
  c_2H(SK)^{1/3}
  \com
  \label{eq:strong_pref_final_scaling}
\end{align}
for a universal constant $c_2>0$.
Combining \cref{eq:strong_pref_bayes_final_MLR,eq:strong_pref_scale_MLR,eq:strong_pref_final_scaling} gives
\begin{equation*}
    \mathfrak R_T
  \geq
  c_3
  \min\set*{
    HT,
    \prn{H^3SK}^{1/3}T^{2/3}
  }
  \com
\end{equation*}
where $c_3>0$ is a universal constant.
Finally, $\mathfrak R_T$ is the average of $\E^{(j)}\brk{\Reg_T}$ over fixed realizations $j$.
Hence some realization $j^\star$ satisfies the same bound, and $\calM_{\mathsf{pref}}(j^\star)$ is the required state-wise PbMDP, which completes the proof.
\end{proof}

\section{Pair-independent transitions}\label{app:pair_independent}
To isolate the effect of transitions that depend on the selected pair, we consider here state-wise PbMDPs whose transition kernel does not depend on the selected pair, that is, $P(s'\mid s,a)=P(s'\mid s)$ for all $(s,a,s')$.
The state-visit probabilities are then the same for every policy, and the learner faces a separate dueling bandit problem at each state.
We show that the minimax regret under this assumption is $\tilde{\Theta}(\min\{HT,(H^2SK)^{1/3}T^{2/3}\})$ for $H\geq2$ and $K\geq5$ (\cref{prop:pair_independent_upper,thm:pair_independent_lower}).
Under the assumptions of \cref{thm:lower_bound_adv_pref_borda}, the $T^{2/3}$ term is smaller by a factor of $H^{1/3}$ than that of the minimax regret $\tilde{\Theta}(\min\{HT,(H^3SK)^{1/3}T^{2/3}\})$ of general state-wise PbMDPs (\cref{thm:reg_occ_meas,thm:lower_bound_adv_pref_borda}), and this factor is thus the cost of transitions that depend on the selected pair.

\subsection{Upper bound}\label{app:pair_independent_upper}
The algorithm for the upper bound runs a separate copy of Dueling-EXP3~\citep[Theorem~2]{saha21adversarial} at each nonterminal state.
The copy at $s$ uses only the comparisons observed at $s$ and is updated only when $s$ is visited.
Since a state is visited at most once per episode, the distributions over pairs at all states are determined before each episode, and the algorithm thus selects a policy $\pi_t\in\Pi$ at the start of episode $t$.
The algorithm does not use the transition kernel.
Since the parameters of Dueling-EXP3 depend on the number of rounds, each copy applies the doubling trick to its own number of visits.
In a phase of fewer than $2K\log K$ visits, the copy makes uniform comparisons, and in a longer phase it runs Dueling-EXP3 with the parameters for the phase length.
After $n$ visits, the expected regret of the copy against any fixed arm is $O((K\log K)^{1/3}n^{2/3})$ by \citet[Theorem~2]{saha21adversarial}.
\begin{proposition}\label{prop:pair_independent_upper}
  Suppose that $P(s'\mid s,a)=P(s'\mid s)$ for all $(s,a,s')$.
  The algorithm above satisfies, for every $T\geq1$,
  \begin{equation*}
    \E\brk*{\Reg_T}
    =O\prn*{\min\set*{HT,\prn{H^2SK\log K}^{1/3}T^{2/3}}}.
  \end{equation*}
\end{proposition}
\begin{proof}

Let $n_s=\sumT\I_t(s)$ be the number of visits to $s$, and let $q(s)=q^\pi(s)$ be the visit probability of $s$, which does not depend on $\pi$.
Since the visit probabilities do not depend on the policy and the loss of a pair $(i,k)$ is the average of the losses of $(i,i)$ and $(k,k)$, the comparator can be taken to play $(i_s^*,i_s^*)$ at each state $s$, where $i_s^*\in\argmin_{i\in[K]}\sumT\ell_t(s,(i,i))$.
Under the oblivious-adversary assumption, the episodes in which $s$ is visited do not depend on the learner, and conditionally on these episodes the copy at $s$ faces a fixed sequence of preference matrices.
Since the shift in \cref{eq:def_borda_ell} cancels in loss differences, the regret bound of the copy applies to the losses $\ell_t$.
Therefore,
\begin{align*}
  \E\brk*{\Reg_T}
  &=\sum_s\E\brk*{\sumT\I_t(s)
    \bigl(\ell_t(s,a_{t,h(s)})-\ell_t(s,(i_s^*,i_s^*))\bigr)}\nn
  &\leq C(K\log K)^{1/3}\sum_s\E\brk*{n_s^{2/3}}\nn
  &\leq C(K\log K)^{1/3}T^{2/3}\sum_s q(s)^{2/3}\nn
  &\leq C(H^2SK\log K)^{1/3}T^{2/3},
\end{align*}
where $C$ is a universal constant, the equality follows from $\E_t\brk{\I_t(s)\ell_t(s,a_{t,h(s)})}=q(s)\sum_a\pi_t(a\mid s)\ell_t(s,a)$, the first inequality from the regret bound of the copy at $s$ conditionally on its visits, the second inequality from Jensen's inequality and $\E\brk{n_s}=Tq(s)$, and the third inequality from $\sum_s q(s)^{2/3}\leq S^{1/3}(\sum_s q(s))^{2/3}=(SH^2)^{1/3}$.
Combining this bound with the trivial bound $\Reg_T\leq HT$ completes the proof.
\end{proof}

Note that under known transitions, \Cref{alg:pf_mdp_occ} with adjusted parameters attains a bound of the same order.
Since the visit probabilities do not depend on the policy, the policy $\pi_\rho$ of \cref{lem:global_comparator_smoothing} satisfies
\begin{equation*}
  V^{\pi_\rho}(s_0;\ell)-V^\pi(s_0;\ell)\leq\rho H.
\end{equation*}
Using this bound in place of \cref{lem:global_comparator_smoothing} in the proof of \cref{thm:reg_occ_meas} gives
\begin{equation*}
  \E\brk{\Reg_T}\leq H\log(SA)/\eta+\eta SKT/(2\rho)+\rho HT.
\end{equation*}
Choosing $\eta=(H\log(SA)/(T\sqrt{HSK}))^{2/3}$ and $\rho=\min\{1/2,\sqrt{\eta SK/H}\}$ gives $\E\brk{\Reg_T}=O(\min\{HT,(H^2SK\log(SK))^{1/3}T^{2/3}\})$ for every $T\geq1$.

\subsection{Lower bound}\label{app:pair_independent_lower}
\begin{theorem}\label{thm:pair_independent_lower}
  Suppose that $H\geq2$ and $K\geq5$.
  For any learning algorithm and any $T\geq1$, there exists an episodic state-wise PbMDP with Borda scores and a known transition kernel satisfying $P(s'\mid s,a)=P(s'\mid s)$ for all $(s,a,s')$, such that
  \[
    \E\brk*{\Reg_T}
    =\Omega\prn*{\min\set*{HT,(H^2SK)^{1/3}T^{2/3}}}.
  \]
\end{theorem}

\begin{figure}[t]
  \centering
  \begin{tikzpicture}[
    x=1cm,y=1cm,>=stealth,font=\footnotesize,
    vertex/.style={circle,draw=black!80,fill=white,
      minimum size=3.8mm,inner sep=0pt,line width=0.55pt},
    endpoint/.style={vertex,minimum size=5.8mm},
    edge/.style={->,draw=black!65,line width=0.5pt}
  ]
    \node[endpoint] (start) at (0,-0.05) {$s_0$};
    \node[endpoint] (end) at (7.2,-0.05) {$s_H$};
    \foreach \name/\pos/\lab in {a/1.45/1,b/3.35/2,c/5.80/{H-1}} {
      \node at (\pos,1.22) {layer $\lab$};
      \node[vertex] (\name1) at (\pos,0.72) {};
      \node[vertex] (\name2) at (\pos,0.14) {};
      \node at (\pos,-0.37) {$\vdots$};
      \node[vertex] (\name3) at (\pos,-0.87) {};
    }
    \foreach \i in {1,2,3} {
      \draw[edge] (start) -- (a\i);
      \draw[edge] (c\i) -- (end);
      \foreach \j in {1,2,3} {
        \draw[edge] (a\i) -- (b\j);
      }
    }
    \node at (0.60,0.94) {$1/S'$};
    \node at (2.40,0.94) {$1/S'$};
    \node at (6.63,0.94) {$1$};
    \node at (4.57,-0.05) {$\cdots$};
  \end{tikzpicture}
  \caption{Pair-independent instance with $S'=\floor{(S-2)/(H-1)}$
  reachable states per intermediate layer.
  Transitions to the next intermediate layer are uniform for every pair,
  and all states in layer $H-1$ lead to $s_H$.}
  \label{fig:pair_independent_instance}
\end{figure}

The instance used in the proof is the layered MDP in \Cref{fig:pair_independent_instance}, in which each intermediate layer has $S'=\floor{(S-2)/(H-1)}$ reachable states and every policy visits each of them with probability $1/S'$ in every episode.
At each of these states, the preference matrix is $\Pmat$ or $\Pmat^{(m)}$ of \cref{eq:strong_pref_matrices}, drawn at random independently across states.
The proof is a simpler version of the proof of \cref{thm:lower_bound_adv_pref_borda}.
Since the learner's choices at a state may depend on the observations at the other states, the change of measure is again applied one state at a time, with the matrices at the other states fixed.
The difference is that a pair in $\calA_b$ now has a loss gap of at least $1/5$ at the state itself under $\Pmat$, instead of $R/5$ at the penalty states, so that the two bounds obtained under $\Pmat$ and $\Pmat^{(m)}$ concern the regret at the same state and are combined state by state.

\begin{proof}
Fix a learning algorithm.
We use the notation of the proof of \cref{thm:lower_bound_adv_pref_borda}, namely $\Ktil=\floor*{K/2}$, $\calK_g=[\Ktil]$, $\calK_b=[K]\setminus\calK_g$, $\calA_b=\calA\setminus(\calK_g\times\calK_g)$, the matrices $\Pmat$, $\Pmat^{(m)}$, and $\mathsf U$ of \cref{eq:strong_pref_matrices,eq:strong_pref_U}, and the Borda scores and losses $b_{\mathsf Q}$ and $\ell_{\mathsf Q}$ of \cref{eq:strong_pref_borda_definition}.

\proofparagraph{Instance.}
Let $S'=\floor*{(S-2)/(H-1)}$.
Since every intermediate layer is nonempty, $S\geq H+1$ and hence $S'\geq1$.
Each intermediate layer $h\in[H-1]$ contains the states $x_{h,1},\ldots,x_{h,S'}$ (\Cref{fig:pair_independent_instance}), and if $2+(H-1)S'<S$, we add unreachable states to the intermediate layers.
For every pair $a$, the transition from $s_0$ is uniform over $x_{1,1},\ldots,x_{1,S'}$, the transition from $x_{h,i}$ with $h\leq H-2$ is uniform over $x_{h+1,1},\ldots,x_{h+1,S'}$, the transition from an unreachable state is uniform over the next layer, and every state in layer $H-1$ leads to $s_H$.
Thus the transition kernel does not depend on the selected pair, and every policy visits each of the states in
\begin{equation*}
    \calX
  =
  \set*{x_{h,i}:h\in[H-1],\ i\in[S']}
\end{equation*}
with probability $1/S'$ in every episode.

Let $T'=T/S'$ be the expected number of episodes in which a given state in $\calX$ is visited, and set
\begin{equation}
    \epsilon
  =
  \frac1{40}\min\set*{1,\prn*{\frac{K}{T'}}^{1/3}}
  \per
  \label{eq:pi_epsilon}
\end{equation}
For every $x\in\calX$, we draw an index $J_x\in\set{0}\cup\calK_g$ according to \cref{eq:strong_pref_index_distribution}, independently across states and before the first episode, and let $J=(J_x)_{x\in\calX}$, with $\calD_{\rm pref}$ its product distribution.
Set $\Pmat^{(0)}=\Pmat$.
For a realization $j=(j_x)_{x\in\calX}$, the preference function is fixed over episodes and given by $\Pmat^{(j_x)}$ at $x\in\calX$, by $\Pmat$ at $s_0$, and by $\mathsf U$ at the unreachable states.
Let $\calM_{\mathsf{pref}}(j)$ be the resulting instance, $\ell^{(j)}$ its Borda-induced loss, and $\E^{(j)}$ the expectation over the learner's randomization, transitions, and preference outcomes under this instance.

The comparator is the following policy.
With $a_0^\star=(1,1)$ and $a_m^\star=(m,m)$ for $m\in\calK_g$ as in the proof of \cref{thm:lower_bound_adv_pref_borda}, let $\pi_j^\star$ be the policy that selects $a_{j_x}^\star$ at each $x\in\calX$, $(1,1)$ at $s_0$, and an arbitrary pair elsewhere.

Since the state-visit probabilities are the same for every policy, $V^\pi(s_0;\ell^{(j)})=\sum_{s}q^\pi(s)\sum_a\pi(a\mid s)\ell^{(j)}(s,a)$ is minimized by minimizing the loss at every state separately.
Under $\Pmat$, every pair in $\calK_g\times\calK_g$ minimizes the loss by \cref{eq:strong_pref_baseline_score_gap}, and under $\Pmat^{(m)}$, the pair $(m,m)$ is the unique minimizer by \cref{eq:strong_pref_normal_gap}.
Hence $\pi_j^\star\in\argmin_{\pi\in\Pi}V^\pi(s_0;\ell^{(j)})$, and since $\ell^{(j)}$ is fixed over episodes, $\pi_j^\star$ attains the minimum in the regret definition for each fixed $j$.

\proofparagraph{Regret at one state.}
Let $q_t=q^{\pi_t}$ be the occupancy measure of the policy selected in episode $t$, so that $q_t(x,a)=\pi_t(a\mid x)/S'$ for $x\in\calX$, and define $N_T(x,a)=\sumT q_t(x,a)$ as in \cref{eq:strong_pref_N}.
Then
\begin{equation}
    \sum_{a\in\calA}N_T(x,a)=T'
  \com\qquad
  0\leq N_T(x,a)\leq T'
  \com
  \label{eq:pi_occupancy_identity}
\end{equation}
and \cref{eq:strong_pref_realized_to_expected} holds since $\E^{(j)}\brk{\I_t(x,a)\mid\pi_t}=q_t(x,a)$.
Since $\pi_t$ and $\pi_j^\star$ visit every state with the same probability, the loss difference in episode $t$ is
\begin{equation*}
  V^{\pi_t}(s_0;\ell^{(j)})-V^{\pi_j^\star}(s_0;\ell^{(j)})=\sum_s q_t(s)\sum_a(\pi_t(a\mid s)-\pi_j^\star(a\mid s))\ell^{(j)}(s,a)\per
\end{equation*}
The term for $s_0$ is nonnegative, since the pair $(1,1)$ minimizes $\ell_{\Pmat}$, and the terms for the unreachable states vanish.
Summing over episodes and taking expectations gives
\begin{equation}
    \E^{(j)}\brk*{\Reg_T}
  \geq
  \sum_{x\in\calX}R_x(j)
  \com\quad
  R_x(j)
  =
  \E^{(j)}\brk*{
    \sum_{a\in\calA}N_T(x,a)
    \prn*{\ell_{\Pmat^{(j_x)}}(a)-\ell_{\Pmat^{(j_x)}}(a_{j_x}^\star)}
  }
  \per
  \label{eq:pi_state_decomposition}
\end{equation}

Fix $x\in\calX$ and a realization $j_{-x}=(j_{x'})_{x'\in\calX\setminus\set{x}}$ of the indices at the other states.
For $u\in\set{0}\cup\calK_g$, let $\calM_{\mathsf{pref}}(u,j_{-x})$ be the instance in which the matrix at $x$ is $\Pmat^{(u)}$ and the other indices are $j_{-x}$, and let $\P_{u,j_{-x}}$ and $\E^{(u,j_{-x})}$ be the distribution of the interaction and the expectation under this instance, defined as in the proof of \cref{thm:lower_bound_adv_pref_borda}.
As in \cref{eq:strong_pref_abbreviations}, we drop $x$ and $j_{-x}$ from the notation and write $\P_0$, $\P_m$, $\E_0$, and $\E_m$.
The instances compared below differ only in the matrix at $x$.

Define $B_x(j_{-x})=\sum_{a\in\calA_b}\E_0\brk{N_T(x,a)}$ as in \cref{eq:strong_pref_Bx} and
\begin{equation*}
    R_{x,u}(j_{-x})
  =
  \E^{(u,j_{-x})}\brk*{
    \sum_{a\in\calA}N_T(x,a)
    \prn*{\ell_{\Pmat^{(u)}}(a)-\ell_{\Pmat^{(u)}}(a_u^\star)}
  }
  \com
\end{equation*}
so that $R_x(j)=R_{x,j_x}(j_{-x})$, and $R_{x,u}(j_{-x})\geq0$ for every $u$.

Under $\Pmat$, \cref{eq:strong_pref_baseline_score_gap} gives $\ell_{\Pmat}(a)-\ell_{\Pmat}(1,1)=\frac15\prn{\ind{a^{\rm L}\in\calK_b}+\ind{a^{\rm R}\in\calK_b}}$, which is at least $1/5$ for $a\in\calA_b$.
Hence
\begin{equation}
    R_{x,0}(j_{-x})
  \geq
  \frac15B_x(j_{-x})
  \per
  \label{eq:pi_baseline_regret}
\end{equation}
Under $\Pmat^{(m)}$, we have
\begin{equation}
    R_{x,m}(j_{-x})
  \geq
  \frac{\epsilon}{2}
  \E_m\brk*{\sum_{a\neq(m,m)}N_T(x,a)}
  =
  \frac{\epsilon}{2}
  \prn*{T'-\E_m\brk*{N_T(x,(m,m))}}
  \com
  \label{eq:pi_alternative_regret}
\end{equation}
where the inequality follows from the gap of at least $\epsilon/2$ for every pair $a\neq(m,m)$ in \cref{eq:strong_pref_normal_gap}, and the equality from \cref{eq:pi_occupancy_identity}.

Applying the change-of-measure argument of \crefrange{eq:strong_pref_pinsker}{eq:strong_pref_average_optimal_count} in the proof of \cref{thm:lower_bound_adv_pref_borda} with $T'=T/S'$, and substituting the resulting bound into \cref{eq:pi_alternative_regret}, we obtain, for $K\geq5$,
\begin{equation}
    \frac1{2\Ktil}\sum_{m\in\calK_g}R_{x,m}(j_{-x})
  \geq
  \frac{\epsilon T'}{10}
  -
  C_1\epsilon^2T'\sqrt{\frac{B_x(j_{-x})}{K}}
  \com
  \label{eq:pi_local_alternative_bound}
\end{equation}
where $C_1=\sqrt{3C_{\rm kl}/2}/4$ and $C_{\rm kl}$ is the universal constant in \cref{eq:strong_pref_KL}.

\proofparagraph{Lower bound at one state.}
Let
\begin{equation*}
    \overline R_x(j_{-x})
  =
  \frac12R_{x,0}(j_{-x})
  +\frac1{2\Ktil}\sum_{m\in\calK_g}R_{x,m}(j_{-x})
\end{equation*}
be the regret at $x$ averaged over $J_x$ according to \cref{eq:strong_pref_index_distribution}.
If $B_x(j_{-x})\geq K/(40C_{\rm kl}\epsilon^2)$, then \cref{eq:pi_baseline_regret} gives
\begin{equation*}
    \overline R_x(j_{-x})
  \geq
  \frac12R_{x,0}(j_{-x})
  \geq
  \frac1{10}B_x(j_{-x})
  \geq
  \frac{K}{400C_{\rm kl}\epsilon^2}
  \per
\end{equation*}
Otherwise, $C_1\epsilon^2T'\sqrt{B_x(j_{-x})/K}<C_1\epsilon T'/\sqrt{40C_{\rm kl}}\leq\epsilon T'/20$, since $C_1/\sqrt{40C_{\rm kl}}=\sqrt{3/80}/4<1/20$, and \cref{eq:pi_local_alternative_bound} gives
\begin{equation*}
    \overline R_x(j_{-x})
  \geq
  \frac1{2\Ktil}\sum_{m\in\calK_g}R_{x,m}(j_{-x})
  \geq
  \frac{\epsilon T'}{20}
  \per
\end{equation*}
In both cases, $\overline R_x(j_{-x})\geq c\min\{K/\epsilon^2,\epsilon T'\}$ with $c=\min\{1/(400C_{\rm kl}),1/20\}$.
By \cref{eq:pi_epsilon}, if $T'\leq K$, then $\epsilon=1/40$ and the minimum equals $T'/40$.
If $T'>K$, then $\epsilon=(K/T')^{1/3}/40$ and the minimum equals $K^{1/3}(T')^{2/3}/40$.
Hence, for every $x\in\calX$ and every $j_{-x}$,
\begin{equation}
    \overline R_x(j_{-x})
  \geq
  \frac{c}{40}\min\set*{T',K^{1/3}(T')^{2/3}}
  \per
  \label{eq:pi_per_state_bound}
\end{equation}

\proofparagraph{Summing over the states.}
Since the indices are independent, conditioning on $J_{-x}=j_{-x}$ leaves $J_x$ distributed as in \cref{eq:strong_pref_index_distribution}, and hence $\E_J\brk{R_x(J)\mid J_{-x}=j_{-x}}=\overline R_x(j_{-x})$.
Taking expectations in \cref{eq:pi_state_decomposition} and using \cref{eq:pi_per_state_bound}, $|\calX|=(H-1)S'$, and $S'T'=T$, we obtain
\begin{align}
    \E_{J\sim\calD_{\rm pref}}\brk*{\E^{(J)}\brk*{\Reg_T}}
  &\geq
  \sum_{x\in\calX}\E_{J_{-x}}\brk*{\overline R_x(J_{-x})}
  \nn
  &\geq
  \frac{c}{40}(H-1)S'\min\set*{T',K^{1/3}(T')^{2/3}}
  \nn
  &=
  \frac{c}{40}(H-1)\min\set*{T,(S'K)^{1/3}T^{2/3}}
  \per
  \label{eq:pi_global_bound}
\end{align}
Finally, $H-1\geq H/2$ for $H\geq2$, and since $\floor{v}\geq v/2$ for $v\geq1$, we have
\begin{equation*}
  S'\geq(S-2)/(2(H-1))\geq S/(4H),
\end{equation*}
where the last inequality is equivalent to $S(H+1)\geq4H$ and follows from $S\geq H+1$.
It follows that
\begin{equation*}
  (H-1)(S'K)^{1/3}\geq\frac H2(SK/(4H))^{1/3}\geq\frac14(H^2SK)^{1/3},
\end{equation*}
and the right-hand side of \cref{eq:pi_global_bound} is at least $c'\min\{HT,(H^2SK)^{1/3}T^{2/3}\}$ for a universal constant $c'>0$.

The left-hand side of \cref{eq:pi_global_bound} is an average of $\E^{(j)}\brk{\Reg_T}$ over fixed realizations $j$.
Hence some realization $j^\star$ satisfies the same bound, and $\calM_{\mathsf{pref}}(j^\star)$ is the required state-wise PbMDP, which completes the proof.
\end{proof}

\section{Experiments}\label{app:experiments}
We evaluate global optimization and policy optimization (\cref{alg:pf_mdp_occ,alg:pf_mdp_policy_optimization}) with known transitions on small instances of the lower-bound construction in \cref{app:proof_lower_bounds}.
The experiments address three questions: (i) whether the regret grows with $T$ at the rate $T^{2/3}$ of \cref{thm:reg_occ_meas,thm:regret_po_23}, (ii) how the regret at a fixed $T$ depends on $S$, $H$, and $K$, and (iii) how the running times of the two algorithms compare.

\subsection{Setup}\label{app:experiments_setup}

\paragraph{Instances.}
An instance is specified by $H$, $M$, $K$, and $T$.
It has $A=K^2$ pairs and $S=2+2M(H-1)$ states, including the $M$ unreachable states $z_{c,1}$.
As in \cref{app:proof_lower_bounds}, let $L=\floor{(H-1)/2}$ and $R=H-1-L$.
The preference matrices are drawn according to \cref{eq:strong_pref_index_distribution} before each run and fixed during the run.
For comparison with \cref{thm:lower_bound_adv_pref_borda}, we set $\epsilon$ in \cref{eq:strong_pref_matrices} to
\begin{equation*}
  \epsilon_T=\min\set*{\frac1{40},\prn*{\frac{8KRM}{6375T}}^{1/3}},
\end{equation*}
which maximizes the explicit lower bound
\begin{equation*}
  B_{\epsilon}(T)=\frac{\epsilon LT}{20}
    -\frac{1275}{128}\frac{\epsilon^4LT^2}{KRM}
\end{equation*}
obtained from \cref{eq:strong_pref_combined_before_square,eq:strong_pref_square_completion} in the proof of \cref{thm:lower_bound_adv_pref_borda}.
For each run, we compute the regret $\Reg_T=\sum_{t=1}^T(V^{\pi_t}(s_0;\ell)-V^{\pi^*}(s_0;\ell))$ of the executed policies from the true losses.

\paragraph{Parameters.}
We use the parameters of \cref{thm:reg_occ_meas} and \cref{eq:po_parameters} in the proof of \cref{thm:regret_po_23} with the exploration rate and the learning rate multiplied by constants $c_e$ and $c_\eta$, respectively.
For global optimization, we set
\begin{equation*}
  \eta_0=\prn*{\frac{\log(SA)}{T\sqrt{SK}}}^{2/3},
  \qquad \eta=c_\eta\eta_0,
  \qquad \rho=\min\set*{\frac12,c_e\frac{\sqrt{\eta_0SK}}H},
\end{equation*}
and for policy optimization,
\begin{align*}
  \gamma&=\min\set*{\frac12,c_e\prn*{\frac{24SK^3\log A}{H^2T}}^{1/3}},\\
  \eta&=\min\set*{\frac1{9H^2},c_\eta\min\set*{\frac1{9H^2},
    \sqrt{\frac{H\log A}{T(81H^5+24H^2SK^2+24HSK^3/\gamma)}}}},\\
  \delta&=2\eta(K/\gamma+H).
\end{align*}
The estimators and the bonus are as in \cref{alg:pf_mdp_occ,alg:pf_mdp_policy_optimization}.
The constants $c_e$ and $c_\eta$ were chosen in preliminary runs with other seeds so as to minimize the regret per step $\Reg_T/(HT)$.
We use $(c_e,c_\eta)=(1,64)$ for global optimization in all experiments.
For policy optimization, we use $(1,4)$ in the experiment varying $T$ and $(0.14,4)$ in the other experiments.

The exploration rates $\rho$ and $\gamma$ decrease with $T$ and grow with $S$ and $K$.
With $c_e=1$, $\gamma$ would equal $1/2$ in nine of the ten configurations in \cref{app:experiments_size}.
If $\gamma=1/2$, policy optimization selects a uniformly random pair with probability $1/2$ at every visited state, and the bound in \cref{thm:regret_po_23} is of the same order as the trivial bound $HT$.
The experiment varying $T$ therefore uses a small instance and large values of $T$, for which $\gamma<1/2$ with $c_e=1$, and in the other experiments $c_e$ is chosen among values that keep $\gamma$ below $1/2$.
In all reported runs, $\rho\leq0.14$ and $\gamma<1/2$.

\subsection{Dependence on the number of episodes}\label{app:experiments_budget}
We first examine whether the regret grows with $T$ at the rate $T^{2/3}$ of \cref{thm:reg_occ_meas,thm:regret_po_23}.
We set $H=3$, $M=1$, and $K=6$, so that $S=6$, and run both algorithms with $T\in\{3\times10^5,10^6,3\times10^6\}$.
We also run policy optimization with $T\in\{10^7,3\times10^7\}$.
The runs for different values of $T$ are separate, each with its own $\epsilon_T$ and parameters, so that \cref{fig:experiments_regret} compares a family of instances indexed by $T$.
These instances have the smallest $H$ and $S$ allowed in \cref{thm:lower_bound_adv_pref_borda}.
Only the preference matrix at $x_{1,1}$ is random, and its index $J$ takes values in $\{0,1,2,3\}$.
Instead of drawing $J$, we fix its value in each run as follows.
For each $T$, we run each algorithm four times with $J=0$ and twice with each of $J=1,2,3$, and take the average of the mean regrets for $J=0,1,2,3$ weighted by their probabilities in \cref{eq:strong_pref_index_distribution}.
The gray curve in \cref{fig:experiments_regret} is the lower bound $B_{\epsilon_T}(T)$, which applies to the regret averaged over $J$ and the learner's randomness.

\begin{figure}[t]
  \centering
  \includegraphics[width=\textwidth]{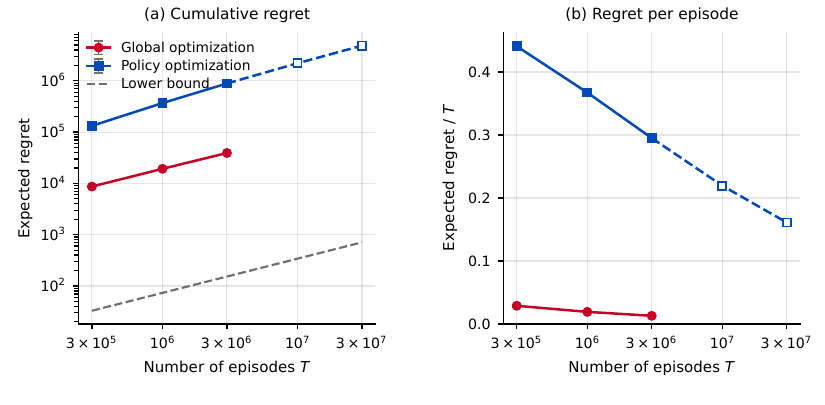}
  \caption{Cumulative regret and regret per episode versus the number of episodes $T$ ($H=3$, $S=6$, and $K=6$).
  Error bars show one standard error of the weighted mean over ten runs and are smaller than the markers.
  Open squares and dashed segments show the additional runs of policy optimization with $T\in\{10^7,3\times10^7\}$.
  The gray curve is the lower bound $B_{\epsilon_T}(T)$.}
  \label{fig:experiments_regret}
\end{figure}

\Cref{fig:experiments_regret}(b) shows that the regret per episode $\Reg_T/T$ decreases with $T$ for both algorithms.
In \cref{fig:experiments_regret}(a), a least-squares fit of the logarithm of the cumulative regret against $\log T$ has slope $0.65$ for global optimization (three values of $T$), which is consistent with the exponent $2/3$ in \cref{thm:reg_occ_meas}, and slope $0.78$ for policy optimization (five values of $T$).
For policy optimization, the slope between consecutive values of $T$ decreases from $0.85$ to $0.72$ as $T$ increases and thus approaches the exponent $2/3$ in \cref{thm:regret_po_23}.

\subsection{Dependence on the problem size}\label{app:experiments_size}
We next examine how the regret at a fixed $T$ depends on $S$, $H$, and $K$.
We fix $T=20{,}000$ and vary one of $S$, $H$, and $K$ at a time.
\begin{itemize}
  \item For $S$, we set $H=5$ and $K=6$ and vary $M\in\{1,2,4,8\}$, so that $S\in\{10,18,34,66\}$.
  \item For $H$, we set $S=34$ and $K=6$ and vary $(H,M)\in\{(3,8),(5,4),(9,2),(17,1)\}$, where $M$ changes with $H$ so that $S$ is fixed, and hence the number $S-M$ of reachable states changes as well.
  \item For $K$, we set $H=5$ and $M=4$ and vary $K\in\{6,8,12,16\}$.
\end{itemize}
Each configuration uses ten independent draws of the instance and of the learner's random numbers.

\begin{figure}[t]
  \centering
  \includegraphics[width=\textwidth]{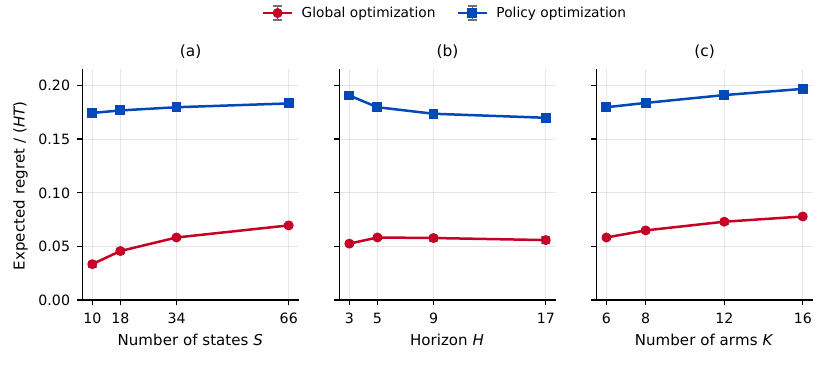}
  \caption{Regret per step for $T=20{,}000$ as $S$, $H$, and $K$ vary.
  Error bars show one standard error over ten runs and are smaller than the markers.}
  \label{fig:experiments_dimensions}
\end{figure}

\Cref{fig:experiments_dimensions} shows that the regret per step increases with $S$ and with $K$ for both algorithms.
As $H$ increases, the regret per step changes little for global optimization and decreases slightly for policy optimization.

\subsection{Running time}\label{app:experiments_runtime}
We finally compare the running times of the two algorithms, since avoiding the convex program of global optimization is a motivation for policy optimization.
We measure the time for $20{,}000$ episodes in the five configurations $(H,S,K)$ shown in \cref{fig:experiments_runtime}, with five repetitions.
In each repetition, both algorithms use the same instance and the same random numbers, and the order in which they run alternates between repetitions.
We use a single thread of an ARM64 macOS machine with Python~3.8.12, NumPy~1.24.4, SciPy~1.10.1, and Numba~0.58.1.

\begin{figure}[t]
  \centering
  \includegraphics[width=\textwidth]{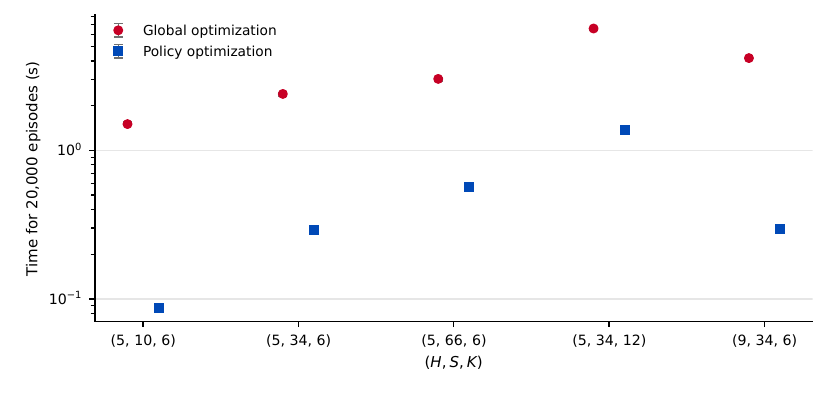}
  \caption{Running time for $20{,}000$ episodes, excluding the evaluation of the regret, for five configurations $(H,S,K)$.
  Points show medians.
  Error bars show interquartile ranges over five paired repetitions and are smaller than the markers.}
  \label{fig:experiments_runtime}
\end{figure}

\Cref{fig:experiments_runtime} shows that, with our implementations, policy optimization is faster in every configuration.
The median ratio of the running time of global optimization to that of policy optimization over paired repetitions ranges from $4.8$ to $17.1$.
\clearpage

\end{document}